\documentclass{article}
\PassOptionsToPackage{numbers, compress}{natbib}
\usepackage[main,final]{neurips_2026}

\usepackage[utf8]{inputenc}
\usepackage[T1]{fontenc}
\usepackage[pagebackref, breaklinks=true, colorlinks, citecolor=blue, bookmarks=false, urlcolor=blue]{hyperref}
\usepackage{wrapfig}
\usepackage{booktabs}
\usepackage{amsfonts}
\usepackage{amsmath,amssymb,amsthm,mathtools}
\usepackage{nicefrac}
\usepackage{microtype}
\usepackage[dvipsnames]{xcolor}
\usepackage{graphicx}
\usepackage{multirow}
\usepackage{makecell}
\usepackage{xcolor}
\usepackage{colortbl}
\usepackage{float}
\usepackage{algorithm}
\usepackage{algpseudocode}
\usepackage{enumitem}
\usepackage{xspace}
\usepackage{pifont}
\usepackage{subcaption}
\usepackage{cleveref}
\usepackage[most]{tcolorbox}
\crefname{figure}{Fig.}{Figs.}
\Crefname{figure}{Fig.}{Figs.}
\usepackage[most]{tcolorbox}
\usepackage{xcolor}

\crefname{table}{Table}{Tables}
\Crefname{table}{Table}{Tables}

\crefname{section}{Sec.}{Secs.}
\Crefname{section}{Sec.}{Secs.}

\usepackage{wrapfig}
\usepackage{tabularx}
\usepackage{colortbl}
\usepackage{array}
\usepackage{amsthm}

\graphicspath{ {./} }
\definecolor{bluecite}{HTML}{0071BC}
\definecolor{tablerowcolor}{RGB}{220,235,245}
\usepackage{url}
\definecolor{darkblue}{rgb}{0, 0, 0.5}
\definecolor{lightgray}{gray}{0.93}
\definecolor{bestrow}{RGB}{232, 245, 233}
\definecolor{pathvisual}{RGB}{132, 26, 202}
\definecolor{pathoverride}{RGB}{196, 15, 16}
\definecolor{pathparam}{RGB}{83, 152, 85}
\hypersetup{colorlinks=true, citecolor=darkblue, linkcolor=darkblue, urlcolor=darkblue}

\newtheorem{theorem}{Theorem}[section]
\newtheorem{corollary}[theorem]{Corollary}
\newtheorem{proposition}[theorem]{Proposition}

\theoremstyle{remark}

\newcommand{\method}{\textsc{Craft}\xspace}
\newcommand{\xmark}{\ding{55}}
\newcommand{\cmark}{\ding{51}}

\newcommand{\doOp}{\mathrm{do}}

\newcommand{\calU}{\mathcal{U}}
\newcommand{\calV}{\mathcal{V}}

\newcommand{\calP}{\mathcal{P}}
\newcommand{\calC}{\mathcal{C}}
\newcommand{\calW}{\mathcal{W}}

\newcommand{\calL}{\mathcal{L}}
\newcommand{\bh}{\mathbf{h}}
\newcommand{\ba}{\mathbf{a}}
\newcommand{\bW}{\mathbf{W}}

\newcommand{\CFR}{\mathrm{CFR}}
\newcommand{\UR}{\mathrm{UR}}
\newcommand{\CER}{\mathrm{CER}}
\newcommand{\BCP}{\mathrm{BCP}}
\newcommand{\HR}{\mathrm{HR}}

\title{\raisebox{-1mm}{\includegraphics[width=0.05\textwidth]{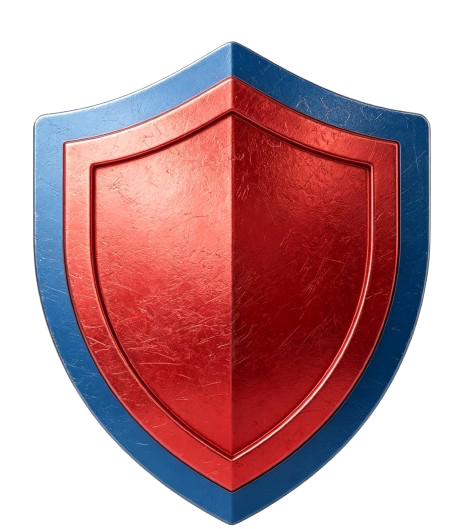}} CRAFT: Causal Responsibility and Failure Tracing\\in Medical Vision Language Models}

\author{
  \textbf{Chunzheng Zhu\textsuperscript{1,*}, Jiaqi Zeng\textsuperscript{1,*}, Hongbo Zhao\textsuperscript{1},}\\
  \textbf{Yihang Chen\textsuperscript{2}, Yijun Wang\textsuperscript{1,$\dagger$}, Jianxin Lin\textsuperscript{1,$\dagger$}}\\
  \textsuperscript{1}Hunan University\quad \textsuperscript{2}University of Hong Kong\\
  {\small \textsuperscript{*}Equal contribution.\quad \textsuperscript{$\dagger$}Corresponding authors.}
}

\begin{document}

\maketitle
\begin{abstract}
As vision language models are increasingly deployed in clinical diagnosis, understanding how they internally resolve competing visual and textual signals becomes a safety imperative. Existing mechanistic analyses remain confined to unimodal text and offer no explanation for why a single misleading sentence can override a correct image based diagnosis, or why a model commits to a confident answer despite insufficient visual evidence. We find that these two safety risks, arbitration failure where textual context overrides visual grounding and brake failure where the model commits without adequate evidence, are mediated by spatially disjoint attention head populations: arbitration heads form a mid-to-deep wideband reflecting cross-layer evidence
competition, while brake heads concentrate in a narrow middle-to-late layer band that regulates evidence sufficiency and abstention behavior. To ground these observations in causal circuitry, we introduce \method, which localizes each failure mode to a minimal causal head set via dual criteria and verifies necessity and sufficiency through temporal probes and Tuned Lens trajectory analysis. Excising arbitration heads sharply reduces conflict following with negligible degradation on clean inputs, while excising brake heads restores appropriate abstention under degraded visual evidence. The two interventions target spatially disjoint head sets and produce distinct corrective effects, underscoring the mechanistic separability of the failure modes. Experiments across multiple medical VQA benchmarks and VLM architectures validate both the localization and interventions, demonstrating that the identified heads causally drive each failure mode and that targeted modulation generalises without retraining.
The code is available at \url{https://github.com/zhcz328/CRAFT}.

\end{abstract}

\section{Introduction}\label{sec:intro}

Medical vision language models are increasingly deployed for clinical question answering and diagnostic support~\cite{bai2025qwen25vl,chen2024internvl,xia2024cares}, yet their behavior can change sharply when visual evidence, textual context, and parametric medical knowledge disagree~\cite{liu2025insight,wu2025conflictmedqa}. For safety critical deployment, observing these failures at the output level is insufficient: we need tools for identifying which internal components are sensitive to each failure mode and whether targeted interventions can correct the behavior without damaging clean diagnostic reasoning.

Mechanistic interpretability provides a natural starting point~\cite{elhage2021mathematical,conmy2023acdc}. Prior work has mapped specialized attention heads in language models, including memory heads for parametric recall~\cite{jin2024cutting}, context heads for in-context integration~\cite{li2025taming,hendel2023context}, and induction heads for token copying~\cite{wang2023interpretability}, while activation patching studies have localized factual recall to specific MLP modules~\cite{meng2022locating,vig2020investigating}. Recent extensions trace intra memory knowledge conflicts to MLP attention interactions~\cite{yu2025where} and reveal that parametric versus contextual signals can superpose within the same head~\cite{li2025taming}. However, these analyses focus exclusively on text only LLMs and binary knowledge conflicts. Medical VLMs introduce an additional axis of competition: the model must arbitrate among image evidence, textual context, and learned medical priors, and it should sometimes abstain when visual evidence is insufficient.

Cross modal competition therefore transcends the binary ``follow context vs.\ retrieve memory'' choice studied in prior work~\cite{longpre2021entity,xie2024adaptive,shi2024ircan}, giving rise to a \emph{tripartite} interaction among three causal pathways: a visual anchoring path~$\calP_v$, a parametric knowledge path~$\calP_\Theta$, and a textual override path~$\calP_t$ (\cref{fig:motivation}). This three way competition creates two mechanistically distinct safety risks:
\textit{(i)}~\textbf{Arbitration failure}: misleading text contradicts a correct visual diagnosis and the model erroneously privileges text, e.g., a chest radiograph showing clear pneumonia infiltrates accompanied by ``no cough reported'' yields the output \textit{non-pneumonia}.
\textit{(ii)}~\textbf{Brake failure}: visual evidence degrades below clinical sufficiency yet the model commits to a definitive answer rather than abstaining, \textit{e.g.}, a dermoscopy image with the lesion region masked still prompts a specific diagnosis.

\begin{figure}[t]
    \centering
    \includegraphics[width=1\textwidth]{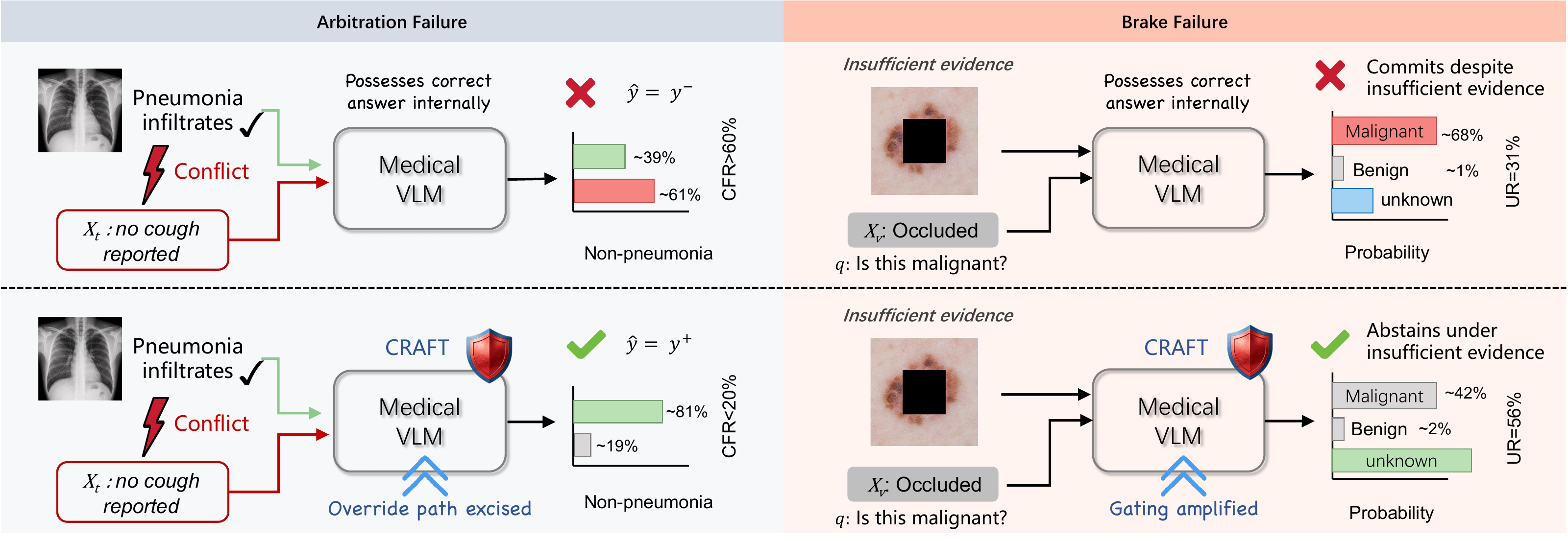}
    \vspace{-6pt}
\caption{Two failure modes in medical VLMs and \method interventions. \text{Left (Arbitration Failure):} Contradictory text hijacks a correct visual diagnosis, causing the model to follow the misleading answer. \text{Right (Brake Failure):} The model commits despite occluded evidence.
}
\label{fig:motivation}
    \vspace{-8pt}
\end{figure}

Both failures produce unsafe clinical outputs but reflect fundamentally different computational regimes: arbitration failure indicates a model that \textit{possesses the correct answer but is hijacked by a competing textual signal}; brake failure indicates one that \textit{lacks sufficient evidence yet commits regardless}. Prior work documents these phenomena at the output level~\cite{longpre2021entity,liu2025insight} and proposes training free mitigations via contrastive decoding~\cite{zhang2025ccd} or image grounded guidance~\cite{zhou2025marine}, but has not addressed \textit{where in the computation graph correct reasoning is overridden}. Purely observational head rankings conflate causal influence with correlation: a head active during conflicted outputs may be predictive without being a useful intervention target~\cite{neo2025interpreting,golovanevsky2025notice}. Three open challenges remain\label{para:challenges}:
\textbf{(C1)}~whether the two failure modes share circuitry or exhibit separable intervention sensitive patterns;
\textbf{(C2)}~how to identify heads whose manipulation changes failure behavior rather than merely predicts it;
\textbf{(C3)}~whether failure specific circuits can be intervened upon without degrading clean reasoning.

We introduce \method, an activation patching framework for tracing and intervening on these two failure modes in medical VLMs. In the first stage, \method identifies a critical layer band through layer wise conflict scoring, then isolates within that band the minimal head sets whose causal effect on the targeted failure is large while their perturbation to clean input reasoning remains small. In the second stage, linear probes and Tuned Lens trajectories~\cite{belrose2023eliciting} provide complementary temporal evidence about when conflict signals become linearly decodable and when answer preferences shift across layers. Our principal contributions are:

\begin{itemize}[leftmargin=*, noitemsep]
  \renewcommand\labelitemi{$\diamond$}
    \item We formalize two medical VLM failure modes: arbitration failure under textual conflict and brake failure under visual degradation, separating textual override from insufficient evidence commitment within a unified structural causal framework.

\item We propose \method, which narrows the causal search to a critical layer band via conflict scoring, then isolates a minimal head set whose excision maximally reduces the targeted failure while preserving clean capacity. The resulting targets are failure specific: arbitration head excision suppresses textual override; brake head excision restores abstention under degraded evidence.

    \item On multiple VLMs, our experiments reveal that arbitration-sensitive heads form a mid-to-deep band while brake-sensitive heads concentrate in a disjoint late-layer band. Targeted interventions reduce Conflict Flip Rate ($\CFR$) by up to 41 absolute points and raise Uncertainty Rate ($\UR$) under insufficient evidence by 75\%, with minimal collateral damage on clean diagnostic accuracy.

\end{itemize}

\section{Related Work}\label{sec:related}

\paragraph{Mechanistic Interpretability and Knowledge Conflicts.}
Circuit discovery via activation patching~\cite{meng2022locating,conmy2023acdc}, path patching~\cite{goldowskydill2023localizing}, and attribution based head analysis~\cite{achtibat2024attnlrp} have identified specialized attention heads for factual recall and in context retrieval~\cite{olsson2022context,wu2025retrieval}.
These tools have been extended to knowledge conflict settings: locating ``memory heads'' whose ablation restores context faithfulness~\cite{jin2024cutting}, revealing superposition of parametric and contextual signals within the same head~\cite{li2025taming}, and tracing intra memory conflicts to specific MLP attention interactions~\cite{yu2025where}.
Inference time intervention methods such as ITI~\cite{li2024iti}, RepE~\cite{zou2023repe}, and conceptor based steering~\cite{liu2024conceptors} demonstrate that activation modulation can redirect model behavior without weight updates, but operate on coarse trait directions rather than failure specific circuits.
Critically, all above studies address text only LLMs; initial mechanistic analyses of VLMs do not study multimodal evidence conflict or clinical failure modes.

\paragraph{Medical VLM Safety and Selective Prediction.}
Biomedical AI supports radiology QA~\cite{lau2018dataset,liu2021slake,zhu2026medeyes}. SP-Mind automates spatial proteomics analysis through an autonomous reasoning agent~\cite{yuan2026sp}. PDU-Net combines path aggregation and dual attention for lung nodule segmentation~\cite{chang2025research}. MRCF gates cross modal flow with sample specific modality reliability~\cite{meng2026rethinking}. Medical VLMs remain susceptible to hallucination under ambiguous inputs~\cite{umapathi2023med,jiang2025hulu}.
Training free mitigations include expert guided contrastive decoding~\cite{zhang2025ccd}, detector grounded guidance~\cite{zhou2025marine}, and internal representation probes for hallucination risk before generation~\cite{chen2025halp}.
Orthogonally, selective prediction~\cite{geifman2019selectivenet} and LLM abstention~\cite{wen2025abstention} advocate withholding low-confidence answers but rely on calibration or prompt heuristics. Our training-free mechanism instead modulates over-commitment circuits to enable abstention without retraining, though clinical thresholds require deployment-specific human validation.\looseness=-1

\begin{figure}[t]
    \centering
    \includegraphics[width=0.92\textwidth]{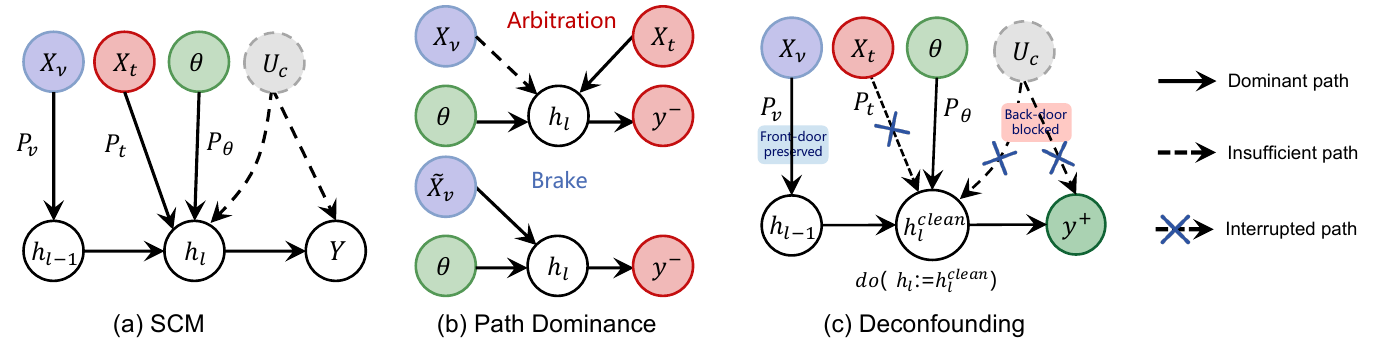}
    \vspace{-6pt}
\caption{\textbf{Interventional tracing setup.} (a)~The evidence conflict problem can be modeled by a structural causal model with three competing paths: \textcolor{pathvisual}{visual anchoring $\calP_v$}, \textcolor{pathparam}{parametric knowledge $\calP_\Theta$}, and \textcolor{pathoverride}{textual override $\calP_t$}, plus latent confounders $U_c$ that open back-door paths. (b)~Path dominance under each failure mode. (c)~Clean-state interchange blocks back-door paths through $U_c$ while preserving the front-door visual causal path, isolating genuine causal effects for head selection.}
\label{fig:scm}
    \vspace{-8pt}
\end{figure}

\section{Background and Preliminaries}\label{sec:background}

\subsection{Transformer Residual Structure}\label{sec:residual}

Given a medical image $v$ and a textual query $q$, a VLM processes the concatenated visual and textual token sequence through an $L$-layer decoder.
Let $\bh_{l,t}$ denote the residual stream at layer $l$ and token position $t$, and let $t_\mathrm{ans}$ denote the answer-token position used for next-token prediction.
At layer $l$, the residual stream accumulates contributions from $n_h$ parallel heads and a feed-forward network:
\begin{equation}
    \bh_{l,t} \;=\; \bh_{l-1,t} \;+\; \sum_{i=1}^{n_h}\bW_O^{l,i}\,\ba_{l,i,t} \;+\; \mathrm{FFN}_l(\bh_{l-1,t}),
    \label{eq:residual}
\end{equation}
where $\ba_{l,i,t}=\mathrm{Attn}^{l,i}(\bh_{l-1})_t$ is the value aggregated by head $i$ at position $t$.
The additive structure permits zeroing any individual head's output $\bW_O^{l,i}\,\ba_{l,i,t} \mapsto \mathbf{0}$ without compensatory redistribution among remaining heads~\cite{elhage2021mathematical,vig2020investigating}, forming the basis for per-head activation interventions that modify a single additive term while leaving the rest of the forward pass unchanged.

\subsection{Two Failure Modes and Problem Formulation}\label{sec:failures}

\paragraph{Structural view.}
We posit exogenous variables $\calU = \{X_v, X_t, \Theta, U_c\}$, where $X_v$ denotes the visual input (with $\tilde{X}_v$ denoting its degraded counterpart in the brake setting), $X_t$ the textual input, $\Theta$ the model parameters, and $U_c$ unobserved exogenous noise capturing latent confounders (\textit{e.g.}, lexical overlap, visual ambiguity, cross-modal coupling). The endogenous variables are $\calV = \{\bh_1,\ldots,\bh_L, Y\}$, where $\bh_l$ is the hidden representation at layer $l$ and $Y$ is the model output. These variables are governed by the structural equations $\bh_l = f_l(\bh_{l-1}, X_v, X_t, \Theta, U_c)$ and $Y = g(\bh_{L,t_\mathrm{ans}})$, where $t_\mathrm{ans}$ is the answer token position and $f_l, g$ are deterministic functions. The induced directed acyclic graph (\cref{fig:scm}a) admits three competing directed paths from sources to outcome: visual anchoring $\calP_v$ (information flow from $X_v$), parametric knowledge $\calP_\Theta$ (from $\Theta$), and textual override $\calP_t$ (from $X_t$).

\paragraph{Preference margin.}
We formalize both failure modes through $\delta = \log p(y^-) - \log p(y^+)$, where $y^+$ denotes the correct answer, $y^-$ the incorrect answer, and $p(y)$ is the model-assigned probability. A positive margin $\delta > 0$ indicates the model favours the wrong answer.
For multi-token answers, $\log p(y)$ denotes the length-normalised log-likelihood under teacher forcing; for multiple-choice datasets, $y$ corresponds to the candidate option token.

Two mechanistically distinct failures arise from evidence conflict:

\textbf{Arbitration failure} ($\calP_t$ dominates $\calP_v$ and $\calP_\Theta$): misleading textual evidence $X_t$ contradicts visual and parametric sources. The model possesses the correct answer under clean conditions ($\delta_{\text{clean}} < 0$, where $\delta_{\text{clean}}$ is the margin under non-conflicting inputs) yet is hijacked under conflict ($\delta_{\text{conflict}}$, the margin under conflicting inputs) (\cref{fig:scm}b, left).

\textbf{Brake failure} ($\calP_v$ insufficient): degraded visual evidence $X_v = x_v^{\text{degrade}}$ cannot sustain a diagnosis, yet the model commits to a concrete answer $y^{\mathrm{com}} \in \mathcal{Y}_{\mathrm{con}}$ (the set of concrete answers) rather than abstaining with $y^{\mathrm{unk}}$ (an “unknown” or abstention option) (\cref{fig:scm}b, right).
Formally:
\begin{align}
\text{(Arbitration)}\quad & \delta_{\text{conflict}} - \delta_{\text{clean}} > 0, \qquad\text{\textit{i.e.}, conflict shifts preference toward } y^-, \label{eq:arb_condition}\\[3pt]
\text{(Brake)}\quad & \log p(y^{\mathrm{com}}) > \log p(y^{\mathrm{unk}}), \quad y^{\mathrm{com}}\in\mathcal{Y}_{\mathrm{con}},; X_v = x_v^{\text{degrade}}. \label{eq:brake_condition}
\end{align}

\subsection{Interventional Tracing Setup}\label{sec:scm}

\paragraph{The confounding problem.}
Na\"ive observational attribution (\textit{e.g.}, ranking heads by activation magnitude) conflates causation with correlation. In the SCM of \cref{fig:scm}(a), the latent variable $U_c$ opens back-door paths $h_l \leftarrow U_c \rightarrow Y$: token-level shortcuts inflate apparent head importance, visual ambiguity causes widespread activation variance correlated with $Y$ through $U_c$ rather than through $\calP_v$, and cross-modal coupling via shared normalisation creates statistical dependence that does not reflect causal influence. The observational quantity $\mathbb{E}[Y \mid \ba_{l,i}]$ therefore conflates the direct causal effect $h \to Y$ with the spurious association $h \leftarrow U_c \to Y$.

\paragraph{Deconfounding via clean-state interchange.}
To block these back-door paths, we exploit the modularity of structural causal models: intervening on an endogenous variable severs all incoming arrows while preserving downstream equations~\cite{pearl2009causality}. For each conflict input we construct a paired conflict-free counterfactual sharing the same $(X_v, q)$ and cache its clean activation $\bh_l^{\mathrm{clean}}$. The interchange intervention $\doOp(h_l := h_l^{\mathrm{clean}})$ fixes the head output independently of the current $U_c$ realisation, thereby d-separating $h_l$ from $U_c$ and blocking all back-door paths (\cref{fig:scm}c). Any residual change in $Y$ is then attributable solely to the conflict mechanism mediated by the intervened head rather than to confounding (Appendix~\ref{app:confounding} for the d-separation proof).\looseness=-1

\paragraph{Per-head activation scaling.}
We index heads by $h=(l,i)$ and write $\calC\subseteq [L]\times[n_h]$ for a head set with $\calC_l=\{i:(l,i)\in\calC\}$. For a set $\calC$ we define \emph{additive head-level scaling} at all token positions:
\begin{equation}
    Y_{\calC}^{(\alpha)} \;=\; g\Bigl(\bh_{L,t_\mathrm{ans}} \;\Big|\; \bh_{l,t} \leftarrow \bh_{l,t} + \sum_{i \in \calC_l}(\alpha_{l,i}-1)\,\bW_O^{l,i}\,\ba_{l,i,t}, \;\forall l,t\Bigr),
    \label{eq:do_intervention}
\end{equation}
where $\alpha_h\!=\!\epsilon$ with $\epsilon\!\to\!0^{+}$ realises near-complete \emph{suppression} (denoted $\mathrm{excise}(\calC)$, used for both arbitration and brake heads). Because the intervention fixes head outputs independently of the current realisation of $U_c$, the measured effect reflects the deconfounded causal contribution of each head to the targeted failure mode. In this work, ``intervention'' consistently refers to this activation-level do-operation rather than to a claim about the input-level data-generating process.

\paragraph{localization objective.}
Equipped with these two deconfounded operations, we identify sparse intervention targets whose interventions improve the targeted failure mode while limiting collateral damage. Rather than treating the selection boundaries or the number of selected targets as pre-specified hyperparameters, we obtain them empirically from the observed intervention profiles after scanning candidate heads or layers:
\begin{align}
    \calC_{\text{arb}}^*
    &= \mathrm{ParetoKnee}_{\calC}
    \left\{
    \CFR(\mathrm{excise}(\calC))\downarrow,\;
    \mathrm{C2W}(\mathrm{excise}(\calC))\downarrow,\;
    |\calC|\downarrow
    \right\},
    \label{eq:objective_arb}\\[3pt]
    \calL_{\text{brk}}^*
    &= \mathrm{ParetoKnee}_{\calL}
    \left\{
    \UR(\mathrm{attenuate}(\calL;\beta))\uparrow,\;
    \mathrm{U2O}(\mathrm{attenuate}(\calL;\beta))\downarrow,\;
    |\calL|\downarrow
    \right\}.
    \label{eq:objective_brake}
\end{align}
Here, $\calC_{\text{arb}}^*$ denotes the selected arbitration heads, while $\calL_{\text{brk}}^*$ denotes the selected brake-related layers under layer-wise attenuation with $\beta<1$. $\CFR$ is the proportion of conflict samples following the misleading text, $\UR$ is the fraction of degraded samples where the model correctly abstains, and $\mathrm{C2W}$/$\mathrm{U2O}$ capture collateral damage on clean inputs. The effective selection boundaries and the resulting number of selected targets are therefore empirical outcomes of the localization procedure, rather than manually fixed hyperparameters (Appendix~\ref{app:metrics} for definitions).

\begin{figure}[t]
    \centering
    \includegraphics[width=\textwidth]{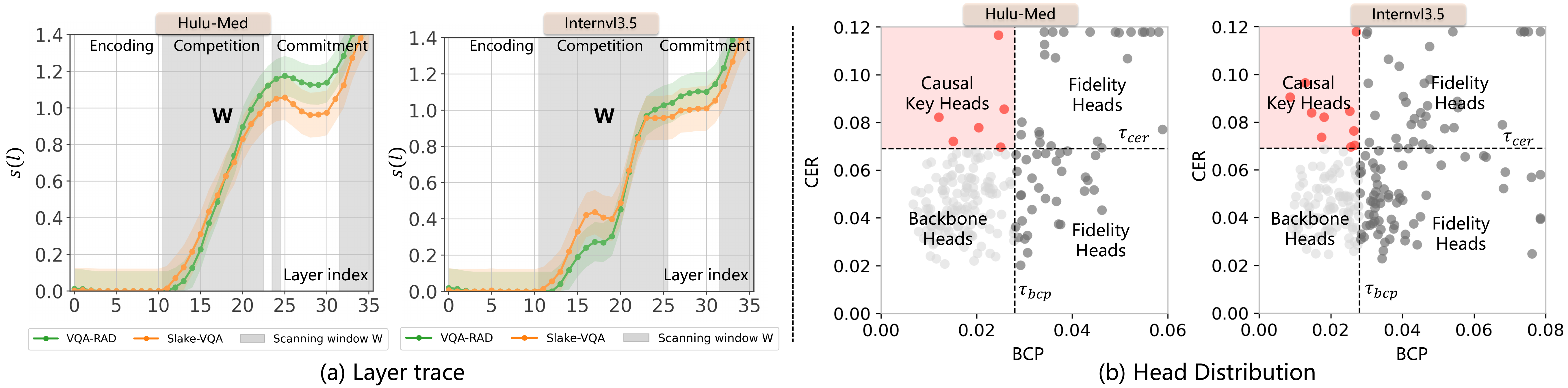}
    \vspace{-6pt}
    \caption{Arbitration failure localization. \textbf{(a)}~Layer-wise $S(l)$: shaded window marks layers where override concentrates. \textbf{(b)}~CER vs.\ BCP: red points (upper left) are selected causal key heads.}
    \label{fig:textual_analysis}
    \vspace{-10pt}
\end{figure}

\section{Causal Tracing of Failure Circuitry}\label{sec:localization}

We trace each failure mode to its causal origin through a coarse to fine procedure: scoring layers to identify a critical band, then resolving heads within that band driving override circuitry (arbitration failure) and gating circuitry (brake failure). Both stages apply the deconfounded intervention of \cref{sec:scm} to ensure selected components are causal drivers rather than observational correlates.
\subsection{Arbitration Failure}\label{sec:loc_text}

\paragraph{Method.}
Given a visual input $x_v$, a question $q$, and conflicting text $x_t^{\mathrm{conf}}$ asserting a misleading answer $y^-$, the preference margin $\delta = \log p_\theta(y^- \mid x_v, x_t^{\mathrm{conf}}, q) - \log p_\theta(y^+ \mid x_v, x_t^{\mathrm{conf}}, q)$ quantifies textual hijacking, where $\delta > 0$ means the model follows the misleading text.

We first narrow the search space via a layer wise conflict score. Let $\delta_{\mathrm{base}}$ denote the preference margin under clean input and $\delta_{\mathrm{conflict}}$ that under adversarial text. For each layer $l$, we cache the clean run hidden state and replace the layer $l$ output in the conflict forward pass with this cached representation, yielding a patched margin $\delta_{\mathrm{patched}}(l)$. The conflict score is then:
\begin{equation}
S(l) = |\delta_{\mathrm{conflict}} - \delta_{\mathrm{base}}| - |\delta_{\mathrm{patched}}(l) - \delta_{\mathrm{base}}|,
\end{equation}
measuring how much patching layer $l$ reduces the conflict effect toward the clean baseline. The critical band $\calW$ is identified from the $S(l)$ profile by selecting layers where the score rises sharply and remains elevated, corresponding to layers that actively sustain the textual override.

Within $\calW$, we isolate individual heads using two causal metrics: $\CER_{\mathrm{exc}}(h) = \delta^{\mathrm{conflict}} - \delta^{\mathrm{excise}(h)}$, measuring how much excising head $h$ reduces the override, and $\BCP_{\mathrm{exc}}(h) = \mathrm{Acc}^{\mathrm{clean}} - \mathrm{Acc}^{\mathrm{excise}(h)}$, measuring collateral damage on clean inputs. The arbitration critical head set is:
\begin{equation}
\calC^*_{\mathrm{arb}} = \{h : \CER_{\mathrm{exc}}(h) > \tau_s \;\wedge\; \BCP_{\mathrm{exc}}(h) < \tau_d\},
\label{eq:cstar}
\end{equation}
retaining only heads causally responsible for the override yet dispensable for normal reasoning.

\paragraph{Experiments and Analysis.}

\begin{figure*}[t]
    \centering
    \includegraphics[width=1\textwidth]{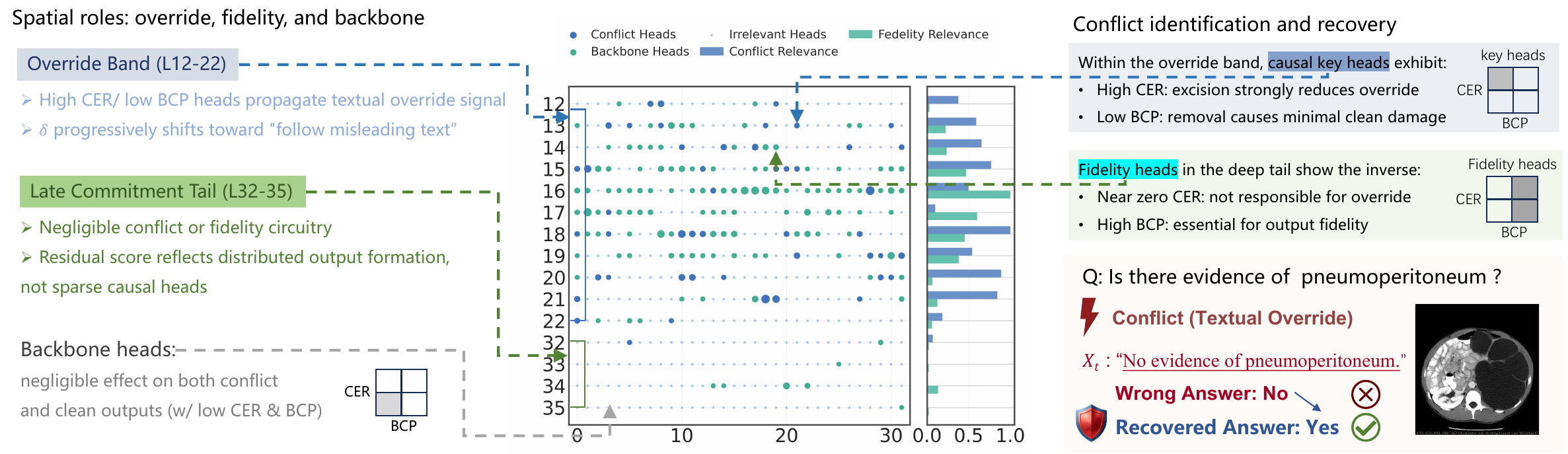}
\caption{\textbf{Arbitration head localization (Hulu-Med 4B).} Left: conflict heads (blue) cluster in the Override Band with high CER and low BCP; the Late Commitment Tail shows negligible conflict or fidelity circuitry, with residual score reflecting distributed output formation. Centre: layer$\times$head grid with per layer CER and BCP bars. Right: selection criterion and recovery example.}

\label{fig:arb_head_grid}
\end{figure*}

We evaluate on VQA-RAD~\cite{lau2018dataset} and SLAKE~\cite{liu2021slake} using Hulu-Med 4B~\cite{jiang2025hulu} (QwenVL~\cite{bai2025qwen3} backbone) and InternVL3.5 4B~\cite{wang2025internvl3}; dataset construction and conflict injection details are provided in \cref{app:exp_setup}. \cref{fig:textual_analysis}(a) shows that $S(l)$ peaks broadly across mid to deep layers (layers 12--17 in Hulu-Med; 14--20 in InternVL), confirming that textual override distributes across a wideband rather than localizing to a single layer. This contrasts with unimodal factual recall, which concentrates in narrow MLP bands~\cite{meng2022locating}.

\cref{fig:textual_analysis}(b) reveals a sparse cluster in the upper left quadrant: heads with strong causal effect on the conflict yet minimal baseline disruption, constituting fewer than 1\% of total heads. \cref{fig:arb_head_grid} further resolves the spatial structure: conflict heads form a continuous band spanning layers 12--22, while the deepest layers (32--35) exhibit per head CER near zero with negligible conflict or fidelity circuitry, their elevated $S(l)$ reflecting distributed output formation rather than sparse causal heads. This dichotomy persists across architectures, suggesting a shared motif: mid layer override wideband functionally distinct from late layer output formation.

\subsection{Brake Failure localization}\label{sec:loc_image}

\begin{wrapfigure}{r}{0.63\textwidth}
    \centering
    \vspace{-10pt}
    \includegraphics[width=\linewidth]{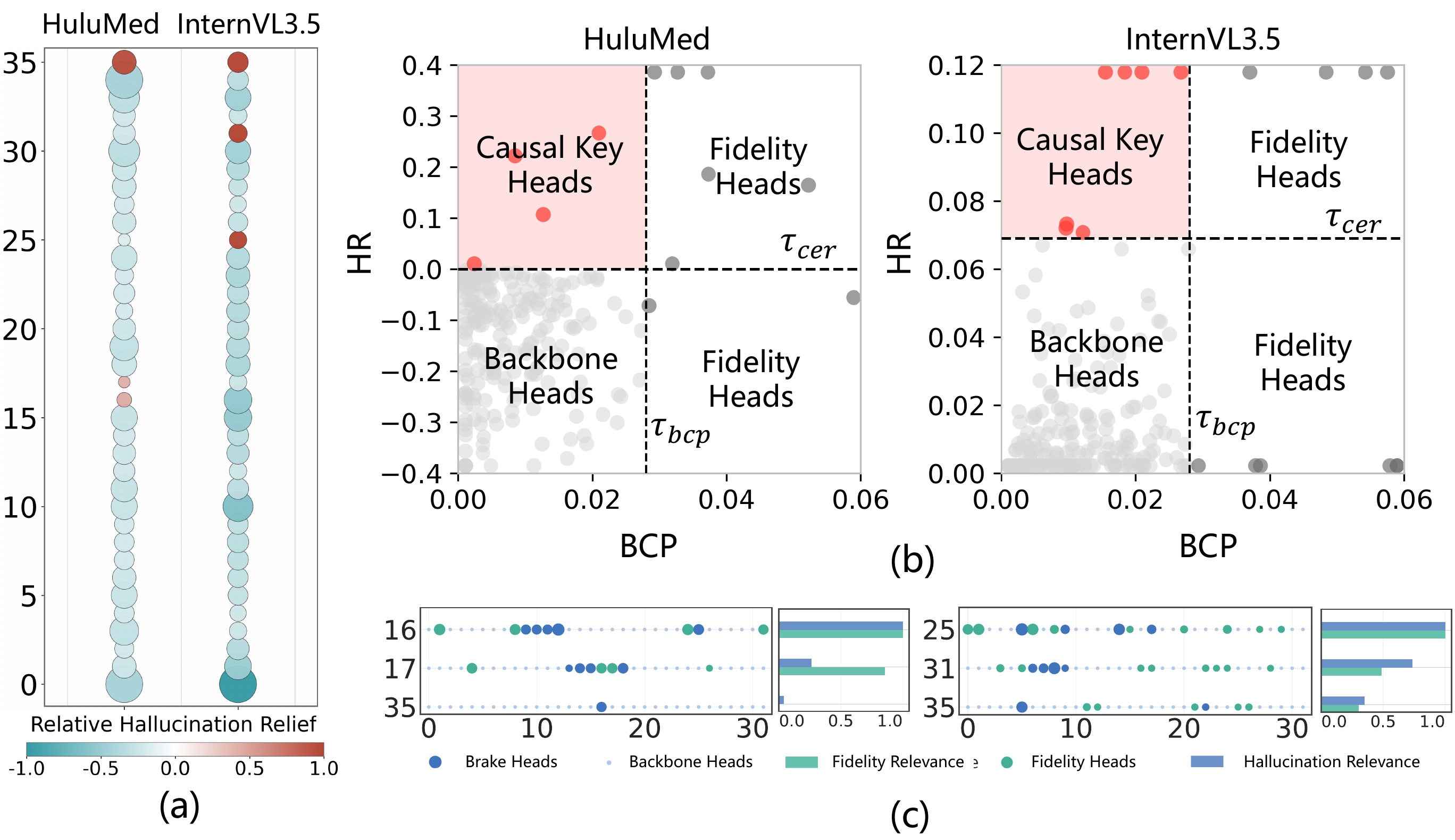}
    \vspace{-10pt}
\caption{Brake head localization (Hulu-Med 4B). (a)~HR($l$) map with positive peaks at the deepest layers. (b,c)~BCP versus HR landscape separating different heads from backbone.}    \label{fig:brake_localization}
    \vspace{-4pt}
\end{wrapfigure}

\paragraph{Method.}
For brake failure, visual evidence is degraded such that the correct response is abstention $y^{\mathrm{unk}}$, yet the model commits to a concrete answer $y^{\mathrm{com}}$. We apply the same  causally grounded deconfounded intervention logic: we first identify layers whose attenuation shifts probability mass toward abstention under degraded input, then localize individual heads within those layers for excision.

Following the same deconfounded intervention protocol described above, the layer-wise hallucination relief score quantifies how much scaling down all head outputs in a given layer $l$ by a factor $\beta < 1$ shifts the model's output probability mass toward the abstention token under degraded visual input:
\begin{equation}
\mathrm{HR}(l) = p_\theta(y^{\mathrm{unk}} \mid \mathrm{attenuate}(l)) - p_\theta(y^{\mathrm{unk}} \mid \mathrm{original}).
\label{eq:hr}
\end{equation}
A positive $\mathrm{HR}(l)$ indicates that weakening layer $l$ helps the model recognise insufficient visual evidence and suppress premature commitment. Within high HR layers, we isolate individual brake heads by jointly requiring that per head hallucination relief $\HR_{\mathrm{exc}}(h)$ exceeds a suppression bound while the collateral metric $\BCP_{\mathrm{exc}}(h) = \mathrm{Acc}^{\mathrm{clean}} - \mathrm{Acc}^{\mathrm{exc}(h)}$ remains below a preservation bound, yielding the candidate set $\calC^*_{\mathrm{brk}}$. The interpretation mirrors arbitration: we seek heads that are \textit{\textbf{causally responsible for suppressing abstention}} yet dispensable for normal diagnostic reasoning.

\paragraph{Experiments and Analysis.}

We construct degraded inputs by masking diagnostically critical regions and set abstention as the target (protocol in \cref{app:exp_setup}).
\Cref{fig:brake_localization}(a) reveals a spatial signature distinct from arbitration: while the arbitration override distributes across a broad mid to deep wideband, the hallucination relief signal is \textit{sharply focal}, concentrating in layers 31--35 of Hulu-Med and layers 25--35 of InternVL3.5. This narrow band accounts for over 78\% of total HR improvement despite spanning fewer than 15\% of model depth, indicating that the abstention mechanism is governed by a \textit{\textbf{compact, late-layer circuit}}. At head level (\Cref{fig:brake_localization}(b,\,c)), only 4--7 heads per model carry statistically significant brake signal ($p < 0.01$, permutation test). Excising these heads yields HR gains of 34.2 (Hulu-Med) and 29.7 (InternVL3.5) with BCP below 2.1, restoring appropriate uncertainty without degrading correct answers. Notably, the selected brake layers are entirely disjoint from the arbitration critical band, corroborating that conflict resolution and confidence calibration are handled by separate sub-circuits. Tuned Lens trajectory analysis further shows that brake heads act as late redirectors: the output distribution commits to a hallucinated token by layer 24, and brake heads enforce premature commitment only in the deepest layers, explaining why shallow probe methods fail to detect.

\section{Functional Roles and Decoupling Analysis}\label{sec:functional}

Having localized intervention-sensitive heads, we evaluate whether the two failure modes respond to distinct interventions and examine the functional specialisation of the identified head sets.

\subsection{Decoupling of Arbitration and Brake Interventions}\label{sec:decoupling}

\textbf{Setup}\hspace{1.5ex}
We compare \method against random ablation using the same head count as $|\calC^*|$.
For VQA experiments, evaluation spans four medical VLMs (Hulu-Med 4B, Hulu-Med 7B, InternVL3.5 4B, and Qwen3-VL-8B~\cite{bai2025qwen3}) on VQA-RAD and SLAKE for textual conflict, and on SLAKE and HealMed-VQA~\cite{nguyen2025localizing} for image conflict.
More backbone and task results are reported in Appendix~\ref{app:text_results}.

\begin{table}[t]
\centering
\caption{Arbitration failure on multimodal benchmarks under textual conflict. $\CFR$: conflict following rate ($\downarrow$); Resist: resist rate ($\uparrow$); $\mathrm{C2W}$: correct to wrong ($\downarrow$); W2C: wrong to correct ($\uparrow$). Bold: best comparison result.  All CRAFT gains are significant ($p < 0.01$).}
\label{tab:main_vqa}
\resizebox{\textwidth}{!}{
\begin{tabular}{@{}l@{\hspace{10pt}}l@{\hspace{6pt}}lcccccccccccc@{}}
\toprule
& & & \multicolumn{4}{c}{\textbf{Prefix}} & \multicolumn{4}{c}{\textbf{Before Q}} & \multicolumn{4}{c}{\textbf{Before A}} \\
\cmidrule(lr){4-7} \cmidrule(lr){8-11} \cmidrule(lr){12-15}
\textbf{Benchmark} & \textbf{Model} & \textbf{Method}
& \textbf{CFR$\downarrow$} & \textbf{Resist$\uparrow$} & \textbf{C2W$\downarrow$} & \textbf{W2C$\uparrow$}
& \textbf{CFR$\downarrow$} & \textbf{Resist$\uparrow$} & \textbf{C2W$\downarrow$} & \textbf{W2C$\uparrow$}
& \textbf{CFR$\downarrow$} & \textbf{Resist$\uparrow$} & \textbf{C2W$\downarrow$} & \textbf{W2C$\uparrow$} \\
\midrule

\multirow{15}{*}{\textbf{VQA-RAD}}
& \multirow{3}{*}{Hulu-Med 4B}
& Baseline
  & 33.33 & 66.67 & \textit{N/A} & \textit{N/A}
  & 44.62 & 55.38 & \textit{N/A} & \textit{N/A}
  & 98.21 & 1.79 & \textit{N/A} & \textit{N/A} \\
& & Random
  & 33.59 & 66.41 & 1.79 & 1.54
  & 40.77 & 59.23 & 1.03 & 4.87
  & 98.72 & 1.28 & 0.51 & 0.00 \\
& & \cellcolor{bestrow}\method
  & \cellcolor{bestrow}\textbf{15.38} & \cellcolor{bestrow}\textbf{84.62} & \cellcolor{bestrow}\textbf{0.26} & \cellcolor{bestrow}\textbf{18.21}
  & \cellcolor{bestrow}\textbf{14.36} & \cellcolor{bestrow}\textbf{85.64} & \cellcolor{bestrow}\textbf{0.51} & \cellcolor{bestrow}\textbf{30.77}
  & \cellcolor{bestrow}\textbf{70.26} & \cellcolor{bestrow}\textbf{29.74} & \cellcolor{bestrow}\textbf{0.00} & \cellcolor{bestrow}\textbf{27.95} \\
\cmidrule(lr){2-15}

& \multirow{3}{*}{Hulu-Med 7B}
& Baseline
  & 25.80 & 74.20 & \textit{N/A} & \textit{N/A}
  & 36.50 & 63.50 & \textit{N/A} & \textit{N/A}
  & 88.00 & 12.00 & \textit{N/A} & \textit{N/A} \\
& & Random
  & 25.40 & 74.60 & 0.90 & 1.30
  & 35.20 & 64.80 & 0.80 & 2.10
  & 88.30 & 11.70 & 0.35 & 0.05 \\
& & \cellcolor{bestrow}\method
  & \cellcolor{bestrow}\textbf{14.20} & \cellcolor{bestrow}\textbf{85.80} & \cellcolor{bestrow}\textbf{0.10} & \cellcolor{bestrow}\textbf{11.70}
  & \cellcolor{bestrow}\textbf{14.70} & \cellcolor{bestrow}\textbf{85.30} & \cellcolor{bestrow}\textbf{0.20} & \cellcolor{bestrow}\textbf{22.00}
  & \cellcolor{bestrow}\textbf{67.50} & \cellcolor{bestrow}\textbf{32.50} & \cellcolor{bestrow}\textbf{0.00} & \cellcolor{bestrow}\textbf{20.50} \\
\cmidrule(lr){2-15}

& \multirow{3}{*}{Hulu-Med 14B}
& Baseline
  & 13.80 & 86.20 & \textit{N/A} & \textit{N/A}
  & 10.60 & 89.40 & \textit{N/A} & \textit{N/A}
  & 24.20 & 75.80 & \textit{N/A} & \textit{N/A} \\
& & Random
  & 14.10 & 85.90 & 0.50 & 0.20
  & 10.40 & 89.60 & 0.30 & 0.50
  & 24.90 & 75.10 & 0.80 & 0.10 \\
& & \cellcolor{bestrow}\method
  & \cellcolor{bestrow}\textbf{8.60} & \cellcolor{bestrow}\textbf{91.40} & \cellcolor{bestrow}\textbf{0.40} & \cellcolor{bestrow}\textbf{5.60}
  & \cellcolor{bestrow}\textbf{6.90} & \cellcolor{bestrow}\textbf{93.10} & \cellcolor{bestrow}\textbf{0.25} & \cellcolor{bestrow}\textbf{3.95}
  & \cellcolor{bestrow}\textbf{16.80} & \cellcolor{bestrow}\textbf{83.20} & \cellcolor{bestrow}\textbf{0.20} & \cellcolor{bestrow}\textbf{7.60} \\
\cmidrule(lr){2-15}

& \multirow{3}{*}{InternVL3.5 4B}
& Baseline
  & 30.29 & 69.71 & \textit{N/A} & \textit{N/A}
  & 21.71 & 78.29 & \textit{N/A} & \textit{N/A}
  & 60.57 & 39.43 & \textit{N/A} & \textit{N/A} \\
& & Random
  & 34.29 & 65.71 & 4.57 & 0.57
  & 24.00 & 76.00 & 4.00 & 1.71
  & 77.71 & 22.29 & 17.71 & 0.57 \\
& & \cellcolor{bestrow}\method
  & \cellcolor{bestrow}\textbf{9.14} & \cellcolor{bestrow}\textbf{90.86} & \cellcolor{bestrow}\textbf{0.57} & \cellcolor{bestrow}\textbf{21.71}
  & \cellcolor{bestrow}\textbf{7.43} & \cellcolor{bestrow}\textbf{92.57} & \cellcolor{bestrow}\textbf{2.29} & \cellcolor{bestrow}\textbf{16.57}
  & \cellcolor{bestrow}\textbf{19.43} & \cellcolor{bestrow}\textbf{80.57} & \cellcolor{bestrow}\textbf{0.00} & \cellcolor{bestrow}\textbf{41.14} \\
\cmidrule(lr){2-15}

& \multirow{3}{*}{Qwen3-VL-8B}
& Baseline
  & 27.84 & 72.16 & \textit{N/A} & \textit{N/A}
  & 16.16 & 83.84 & \textit{N/A} & \textit{N/A}
  & 57.98 & 42.02 & \textit{N/A} & \textit{N/A} \\
& & Random
  & 28.10 & 71.90 & 0.90 & 0.64
  & 17.00 & 83.00 & 1.00 & 0.16
  & 58.20 & 41.80 & 0.35 & 0.13 \\
& & \cellcolor{bestrow}\method
  & \cellcolor{bestrow}\textbf{11.90} & \cellcolor{bestrow}\textbf{88.10} & \cellcolor{bestrow}\textbf{0.20} & \cellcolor{bestrow}\textbf{16.14}
  & \cellcolor{bestrow}\textbf{7.50} & \cellcolor{bestrow}\textbf{92.50} & \cellcolor{bestrow}\textbf{0.15} & \cellcolor{bestrow}\textbf{8.81}
  & \cellcolor{bestrow}\textbf{34.00} & \cellcolor{bestrow}\textbf{66.00} & \cellcolor{bestrow}\textbf{0.10} & \cellcolor{bestrow}\textbf{24.08} \\

\midrule

\multirow{15}{*}{\textbf{SLAKE}}
& \multirow{3}{*}{Hulu-Med 4B}
& Baseline
  & 19.15 & 80.85 & \textit{N/A} & \textit{N/A}
  & 19.15 & 80.85 & \textit{N/A} & \textit{N/A}
  & 87.73 & 12.27 & \textit{N/A} & \textit{N/A} \\
& & Random
  & 21.28 & 78.72 & 2.95 & 0.82
  & 20.79 & 79.21 & 4.09 & 2.45
  & 91.00 & 9.00 & 4.26 & 0.98 \\
& & \cellcolor{bestrow}\method
  & \cellcolor{bestrow}\textbf{12.11} & \cellcolor{bestrow}\textbf{87.89} & \cellcolor{bestrow}5.07 & \cellcolor{bestrow}\textbf{12.11}
  & \cellcolor{bestrow}\textbf{7.04} & \cellcolor{bestrow}\textbf{92.96} & \cellcolor{bestrow}\textbf{0.65} & \cellcolor{bestrow}\textbf{12.77}
  & \cellcolor{bestrow}\textbf{46.32} & \cellcolor{bestrow}\textbf{53.68} & \cellcolor{bestrow}4.91 & \cellcolor{bestrow}\textbf{46.32} \\
\cmidrule(lr){2-15}

& \multirow{3}{*}{Hulu-Med 7B}
& Baseline
  & 15.40 & 84.60 & \textit{N/A} & \textit{N/A}
  & 14.20 & 85.80 & \textit{N/A} & \textit{N/A}
  & 75.00 & 25.00 & \textit{N/A} & \textit{N/A} \\
& & Random
  & 15.90 & 84.10 & 1.00 & 0.50
  & 14.80 & 85.20 & 0.90 & 0.30
  & 76.20 & 23.80 & 1.50 & 0.30 \\
& & \cellcolor{bestrow}\method
  & \cellcolor{bestrow}\textbf{11.00} & \cellcolor{bestrow}\textbf{89.00} & \cellcolor{bestrow}\textbf{1.50} & \cellcolor{bestrow}\textbf{5.90}
  & \cellcolor{bestrow}\textbf{6.70} & \cellcolor{bestrow}\textbf{93.30} & \cellcolor{bestrow}\textbf{0.40} & \cellcolor{bestrow}\textbf{7.90}
  & \cellcolor{bestrow}\textbf{45.00} & \cellcolor{bestrow}\textbf{55.00} & \cellcolor{bestrow}\textbf{1.00} & \cellcolor{bestrow}\textbf{31.00} \\
\cmidrule(lr){2-15}

& \multirow{3}{*}{Hulu-Med 14B}
& Baseline
  & 8.80 & 91.20 & \textit{N/A} & \textit{N/A}
  & 6.90 & 93.10 & \textit{N/A} & \textit{N/A}
  & 10.40 & 89.60 & \textit{N/A} & \textit{N/A} \\
& & Random
  & 9.10 & 90.90 & 0.50 & 0.20
  & 7.10 & 92.90 & 0.40 & 0.20
  & 10.80 & 89.20 & 0.60 & 0.20 \\
& & \cellcolor{bestrow}\method
  & \cellcolor{bestrow}\textbf{6.20} & \cellcolor{bestrow}\textbf{93.80} & \cellcolor{bestrow}\textbf{0.40} & \cellcolor{bestrow}\textbf{3.00}
  & \cellcolor{bestrow}\textbf{5.20} & \cellcolor{bestrow}\textbf{94.80} & \cellcolor{bestrow}\textbf{0.30} & \cellcolor{bestrow}\textbf{2.00}
  & \cellcolor{bestrow}\textbf{4.90} & \cellcolor{bestrow}\textbf{95.10} & \cellcolor{bestrow}\textbf{0.40} & \cellcolor{bestrow}\textbf{5.90} \\
\cmidrule(lr){2-15}

& \multirow{3}{*}{InternVL3.5 4B}
& Baseline
  & 12.66 & 87.34 & \textit{N/A} & \textit{N/A}
  & 8.93 & 91.07 & \textit{N/A} & \textit{N/A}
  & 16.40 & 83.60 & \textit{N/A} & \textit{N/A} \\
& & Random
  & 18.18 & 81.82 & 8.28 & 2.76
  & 18.51 & 81.49 & 10.71 & 1.14
  & 23.70 & 76.30 & 8.77 & 1.46 \\
& & \cellcolor{bestrow}\method
  & \cellcolor{bestrow}\textbf{6.66} & \cellcolor{bestrow}\textbf{93.34} & \cellcolor{bestrow}\textbf{0.65} & \cellcolor{bestrow}\textbf{6.66}
  & \cellcolor{bestrow}\textbf{5.84} & \cellcolor{bestrow}\textbf{94.16} & \cellcolor{bestrow}\textbf{0.65} & \cellcolor{bestrow}\textbf{3.73}
  & \cellcolor{bestrow}\textbf{5.84} & \cellcolor{bestrow}\textbf{94.16} & \cellcolor{bestrow}\textbf{2.27} & \cellcolor{bestrow}\textbf{12.82} \\
\cmidrule(lr){2-15}

& \multirow{3}{*}{Qwen3-VL-8B}
& Baseline
  & 15.60 & 84.40 & \textit{N/A} & \textit{N/A}
  & 11.80 & 88.20 & \textit{N/A} & \textit{N/A}
  & 45.20 & 54.80 & \textit{N/A} & \textit{N/A} \\
& & Random
  & 16.20 & 83.80 & 0.85 & 0.25
  & 12.40 & 87.60 & 0.75 & 0.15
  & 47.00 & 53.00 & 2.10 & 0.30 \\
& & \cellcolor{bestrow}\method
  & \cellcolor{bestrow}\textbf{10.80} & \cellcolor{bestrow}\textbf{89.20} & \cellcolor{bestrow}\textbf{0.30} & \cellcolor{bestrow}\textbf{5.10}
  & \cellcolor{bestrow}\textbf{7.20} & \cellcolor{bestrow}\textbf{92.80} & \cellcolor{bestrow}\textbf{0.25} & \cellcolor{bestrow}\textbf{4.85}
  & \cellcolor{bestrow}\textbf{23.20} & \cellcolor{bestrow}\textbf{76.80} & \cellcolor{bestrow}\textbf{0.30} & \cellcolor{bestrow}\textbf{22.30} \\

\bottomrule
\end{tabular}
}
\vspace{-4pt}
\end{table}

\Cref{tab:main_vqa} reports the effect of excising identified arbitration heads. The critical observation is not merely that CFR decreases, but \textit{how}: the intervention shifts the model from following conflicting text to resisting it, while leaving clean accuracy virtually unchanged (C2W near zero). This aligns with the causal interpretation from \cref{sec:scm}: excising arbitration heads removes the $\calP_t$ override without disrupting the front door path $\calP_v$. Random ablation, by contrast, increases CFR in several conditions (\textit{e.g.}, InternVL3.5 Before A), confirming that indiscriminate head removal can exacerbate failure by disrupting backbone reasoning while leaving the override intact. A further insight emerges from conflict position dependence: Before A shows the highest baseline CFR and strongest resistance to intervention, suggesting that when conflicting text is placed immediately before the answer token, the override exploits a shorter residual path beyond the identified head set.

\Cref{tab:image_conflict} reports the corresponding results under visual degradation. Ablating the brake heads identified via the localization procedure in \cref{sec:loc_image} consistently increases the unknown rate across all four models while keeping U2O collateral at zero or near zero, with the improvement driven by larger O2U transitions that redirect non-abstaining predictions toward \texttt{unknown}. Random ablation is unstable, yielding smaller gains on Hulu-Med and Qwen3-VL-8B and even reducing UR on SLAKE for InternVL3.5 while introducing nonzero U2O, confirming that mitigating brake failure requires \textit{targeted suppression of the localized commitment heads} rather than generic perturbation.

\begin{table*}[!t]
\centering
\caption{Brake failure under visual degradation. \method excises the localized brake heads to recover abstention. UR: unknown rate ($\uparrow$); O2U: non-unknown to unknown ($\uparrow$); U2O: unknown to other ($\downarrow$).}
\label{tab:image_conflict}
\resizebox{\textwidth}{!}{
\begin{tabular}{@{}llccc ccc ccc ccc ccc@{}}
\toprule
& & \multicolumn{3}{c}{\textbf{Hulu-Med 4B}}
& \multicolumn{3}{c}{\textbf{Hulu-Med 7B}}
& \multicolumn{3}{c}{\textbf{Hulu-Med 14B}}
& \multicolumn{3}{c}{\textbf{InternVL3.5 4B}}
& \multicolumn{3}{c}{\textbf{Qwen3-VL-8B}} \\
\cmidrule(lr){3-5} \cmidrule(lr){6-8} \cmidrule(lr){9-11} \cmidrule(lr){12-14} \cmidrule(lr){15-17}
\textbf{Benchmark} & \textbf{Method}
& UR$\uparrow$ & O2U$\uparrow$ & U2O$\downarrow$
& UR$\uparrow$ & O2U$\uparrow$ & U2O$\downarrow$
& UR$\uparrow$ & O2U$\uparrow$ & U2O$\downarrow$
& UR$\uparrow$ & O2U$\uparrow$ & U2O$\downarrow$
& UR$\uparrow$ & O2U$\uparrow$ & U2O$\downarrow$ \\
\midrule
\multirow{3}{*}{\textbf{SLAKE}}
& Baseline
& 37.88 & -- & --
& 44.12 & -- & --
& 48.90 & -- & --
& 22.58 & -- & --
& 42.42 & -- & -- \\
& Random
& 45.45 & 8.33 & 0.76
& 48.82 & 5.29 & 0.59
& 51.10 & 3.00 & 0.80
& 16.13 & 0.00 & 6.45
& 47.73 & 6.06 & 0.75 \\
& \cellcolor{bestrow}\method
& \cellcolor{bestrow}\textbf{65.15}
& \cellcolor{bestrow}\textbf{27.27}
& \cellcolor{bestrow}\textbf{0.00}
& \cellcolor{bestrow}\textbf{61.18}
& \cellcolor{bestrow}\textbf{17.06}
& \cellcolor{bestrow}\textbf{0.00}
& \cellcolor{bestrow}\textbf{58.40}
& \cellcolor{bestrow}\textbf{10.20}
& \cellcolor{bestrow}\textbf{0.70}
& \cellcolor{bestrow}\textbf{32.26}
& \cellcolor{bestrow}\textbf{9.68}
& \cellcolor{bestrow}\textbf{0.00}
& \cellcolor{bestrow}\textbf{55.30}
& \cellcolor{bestrow}\textbf{12.88}
& \cellcolor{bestrow}\textbf{0.00} \\
\midrule
\multirow{3}{*}{\textbf{HealMed-VQA}}
& Baseline
& 78.44 & -- & --
& 85.07 & -- & --
& 88.61 & -- & --
& 81.56 & -- & --
& 84.72 & -- & -- \\
& Random
& 86.83 & 10.18 & 1.80
& 88.19 & 4.31 & 1.19
& 89.58 & 1.45 & 0.48
& 82.47 & 3.06 & 2.15
& 87.50 & 4.17 & 1.39 \\
& \cellcolor{bestrow}\method
& \cellcolor{bestrow}\textbf{98.20}
& \cellcolor{bestrow}\textbf{19.76}
& \cellcolor{bestrow}\textbf{0.00}
& \cellcolor{bestrow}\textbf{95.83}
& \cellcolor{bestrow}\textbf{11.10}
& \cellcolor{bestrow}\textbf{0.34}
& \cellcolor{bestrow}\textbf{93.06}
& \cellcolor{bestrow}\textbf{4.80}
& \cellcolor{bestrow}\textbf{0.35}
& \cellcolor{bestrow}\textbf{93.64}
& \cellcolor{bestrow}\textbf{12.46}
& \cellcolor{bestrow}\textbf{0.38}
& \cellcolor{bestrow}\textbf{95.14}
& \cellcolor{bestrow}\textbf{11.11}
& \cellcolor{bestrow}\textbf{0.69} \\
\bottomrule
\end{tabular}}
\vspace{-2mm}
\end{table*}

\begin{table}[t]
    \centering
    \label{tab:ablation}
    \vspace{-2mm}
\caption{Comparison of selection criteria at equal head count. Left: arbitration failure (CER only: highest conflict effect; BCP only: lowest collateral). Right: brake failure on validation split (HR only: highest hallucination relief; BCP only: lowest collateral). Shaded: CRAFT.}
\begin{minipage}{0.48\textwidth}
        \centering
        \resizebox{\textwidth}{!}{
        \begin{tabular}{@{}l@{\hspace{6pt}}l@{\hspace{8pt}}c@{\hspace{8pt}}c@{\hspace{8pt}}c@{\hspace{8pt}}c@{}}
        \toprule
        \textbf{Model} & \textbf{Criterion} & $\CFR\downarrow$ & $|\Delta\CFR|\uparrow$ & $\mathrm{C2W}\downarrow$ & \#Heads \\
        \midrule
        \multirow{3}{*}{Hulu-Med 4B}
        & CER only & \textbf{6.7} & \textbf{37.9} & 1.8 & 6 \\
        & BCP only & 44.4 & 0.3 & 1.3 & 6 \\
        & \cellcolor{bestrow}Dual (Ours)
        & \cellcolor{bestrow}14.4
        & \cellcolor{bestrow}30.3
        & \cellcolor{bestrow}\textbf{0.5}
        & \cellcolor{bestrow}6 \\
        \midrule
        \multirow{3}{*}{InternVL3.5 4B}
        & CER only & 9.7 & 12.0 & 4.0 & 10 \\
        & BCP only & \textbf{7.4} & \textbf{14.3} & 2.3 & 10 \\
        & \cellcolor{bestrow}Dual (Ours)
        & \cellcolor{bestrow}\textbf{7.4}
        & \cellcolor{bestrow}\textbf{14.3}
        & \cellcolor{bestrow}\textbf{0.0}
        & \cellcolor{bestrow}10 \\
        \bottomrule
        \end{tabular}}
    \end{minipage}\hfill
    \begin{minipage}{0.48\textwidth}
        \centering
        \resizebox{\textwidth}{!}{
        \begin{tabular}{@{}l@{\hspace{6pt}}l@{\hspace{8pt}}c@{\hspace{8pt}}c@{\hspace{8pt}}c@{\hspace{8pt}}c@{}}
        \toprule
        \textbf{Model} & \textbf{Criterion} & $\mathrm{UR}\uparrow$ & $\Delta(\mathrm{UR})\uparrow$ & $\mathrm{U2O}\downarrow$ & \#Heads \\
        \midrule
        \multirow{3}{*}{Hulu-Med 4B}
        & HR only & \textbf{65.2} & \textbf{27.3} & \textbf{0.0} & 6 \\
        & BCP only & 38.6 & 0.8 & \textbf{0.0} & 6 \\
        & \cellcolor{bestrow}Dual (Ours)
        & \cellcolor{bestrow}\textbf{65.2}
        & \cellcolor{bestrow}\textbf{27.3}
        & \cellcolor{bestrow}\textbf{0.0}
        & \cellcolor{bestrow}6 \\
        \midrule
        \multirow{3}{*}{InternVL3.5 4B}
        & HR only & \textbf{61.3} & \textbf{38.7} & 21.6 & 7 \\
        & BCP only & 16.1 & -6.5 & 6.5 & 7 \\
        & \cellcolor{bestrow}Dual (Ours)
        & \cellcolor{bestrow}32.3
        & \cellcolor{bestrow}9.7
        & \cellcolor{bestrow}\textbf{0.0}
        & \cellcolor{bestrow}7 \\
        \bottomrule
        \end{tabular}}
    \end{minipage}
\vspace{-4pt}
\end{table}

\subsection{Contribution of Individual Selection Objectives}
\label{sec:ablation}

\Cref{tab:ablation} compares selection strategies at equal candidate count. For arbitration, CER only achieves the strongest raw CFR reduction on Hulu-Med 4B ($|\Delta\CFR|{=}37.9$) but incurs C2W of 1.8, whereas BCP nearly fails to suppress conflict; the dual criterion balances both objectives, yielding substantial CFR reduction with C2W at 0.5. On InternVL3.5 4B, CER only increases C2W to 4.0, revealing entanglement between high CER candidates and clean reasoning pathways, while the dual criterion matches the CFR reduction of BCP only and eliminates all collateral (C2W${=}$0.0). For brake failure, HR only and dual coincide on Hulu-Med 4B, but on InternVL3.5 4B HR only causes severe U2O degradation (21.6\%) whereas the dual criterion keeps U2O at 0.0 with cleaner UR gain. These results confirm that the joint criterion is essential for reliable intervention across both failure modes.

\section{Causal Verification via Internal Probes}\label{sec:verification}

The localization procedure in \S\ref{sec:localization} identifies heads via interventional effects. We verify these correspond to interpretable representational transitions using temporal probing and Tuned Lens~\cite{belrose2023eliciting} trajectory analysis on Hulu-Med 4B (VQA-RAD for arbitration, SLAKE for brake failure).

\vspace{-2mm}
\subsection{Temporal Probe Separation}\label{sec:probes}

\begin{wrapfigure}{r}{0.6\textwidth}
    \centering
    \vspace{-2pt}
    \includegraphics[width=\linewidth]{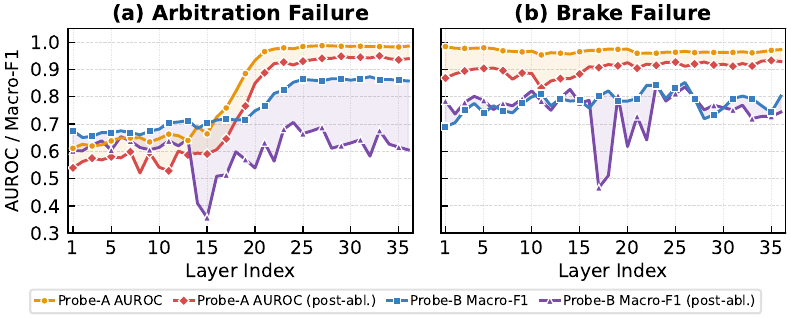}
    \vspace{-6pt}
    \caption{Temporal probe separation. Probe~A (detection) saturates before Probe~B (commitment), revealing a detect then commit structure. Post-intervention selectively suppresses Probe~B while preserving Probe~A.}
    \label{fig:probe_separation}
    \vspace{-12pt}
\end{wrapfigure}
We train two binary logistic probes on per-layer hidden states at the answer position.
\textbf{Probe~A} (conflict detection) predicts whether the input contains conflict-inducing evidence.
\textbf{Probe~B} (failure prediction) predicts whether the model will enter the corresponding failure mode.

If the identified heads mediate a transition from evidence registration to failure commitment, we expect three testable predictions: (P1)~Probe~A becomes accurate before Probe~B across layers; (P2)~Probe~B peaks in later layers; (P3)~intervening on $\calC^*$ selectively weakens Probe~B more than Probe~A.
\textit{i)} For arbitration failure (\Cref{fig:probe_separation}(a)), Probe~A saturates around layers 17--21 (AUROC $> 0.97$), while Probe~B peaks 3--5 layers later (Macro F1 $\approx$ 0.85--0.87), delineating a computational window where the model has registered the conflict but has not yet committed to following it. After intervening on $\calC_{\text{arb}}^*$, Probe~A remains intact whereas Probe~B drops substantially, confirming that the identified heads govern commitment rather than conflict encoding. \textit{ii)} For brake failure (\Cref{fig:probe_separation}(b)), Probe~A is predictive from shallow layers onward, indicating that degraded visual evidence is linearly decodable early. Probe~B peaks later (layers 22--27) with greater variance and lower ceiling. After intervening on $\mathcal{C}^*_{\mathrm{brk}}$, Probe B shows non-monotonic behavior, consistent with the observation that the intervention reorganises late layer competition between commitment and abstention rather than erasing a single pathway.

\subsection{Tuned Lens Trajectory Analysis}\label{sec:tuned_lens}
\begin{wrapfigure}{r}{0.6\textwidth}
    \centering
    \vspace{-12pt}
    \includegraphics[width=\linewidth]{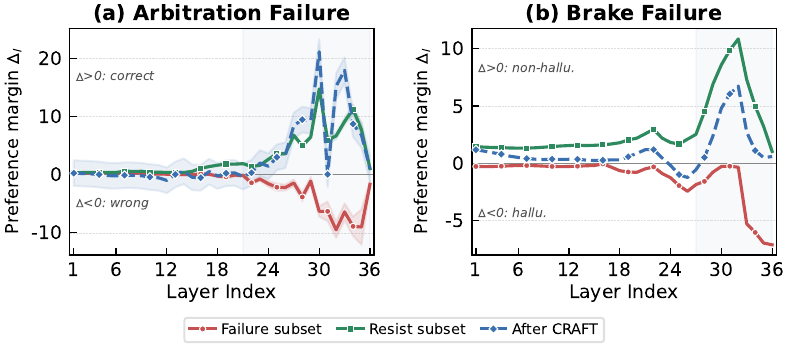}
    \vspace{-14pt}
\caption{Tuned Lens trajectories. The preference margin $\Delta_l$ flips sign within the override window under conflict. Intervention suppresses this reversal, restoring correct dominance.}    \label{fig:tuned_lens}
    \vspace{-8pt}
\end{wrapfigure}

We apply the Tuned Lens~\cite{belrose2023eliciting} to project layer's hidden state into vocabulary space, computing the preference margin $\delta_l = \mathrm{logit}_l(y^-) - \mathrm{logit}_l(y^+)$ to track when the latent prediction crosses from correct to incorrect.

Under conflict free conditions, $\delta_l$ remains negative throughout (\Cref{fig:tuned_lens}). Under conflict, $\delta_l$ transitions from negative to positive between layers 18 and 25, placing the preference reversal within the wideband identified by causal localization. The reversal accumulates over several layers, consistent with a distributed override circuit. After excising $\calC_{\text{arb}}^*$, this transition is suppressed and $\delta_l$ stays negative through the critical window. The convergence of three independent signals (interventional CER, temporal probes, and Tuned Lens trajectories) on the same layer band provides triangulated evidence that the localized heads causally mediate this reversal. Full distributions are in \cref{app:probe_details}.

\section{Conclusion}\label{sec:conclusion}

We presented \method, an interventional tracing framework that leverages causal deconfounding to localize failure-critical attention heads in medical VLMs. The framework reveals that arbitration failure and brake failure occupy structurally disjoint circuitry and respond to opposite interventions: excision of a mid-layer override wideband suppresses textual hijacking, while excision of sparse late-layer gating heads recovers appropriate abstention under degraded visual evidence. Both interventions are training-free, modular, and affect fewer than one percent of total heads, suggesting that sparse activation editing may serve as a practical and lightweight post-hoc safety mechanism for deployed medical VLMs. We hope this mechanistic interpretability perspective on multimodal failure modes opens new and productive directions for principled safety analysis in clinical AI applications.

\bibliographystyle{splncs04}
\bibliography{references_update}

\newpage

\section{More Experimental Details}
\label{app:exp_setup}

\subsection{Models, conflict construction, and data splits}
\begin{wrapfigure}{r}{0.45\textwidth}
    \centering
    \vspace{-10pt}
    \includegraphics[width=\linewidth]{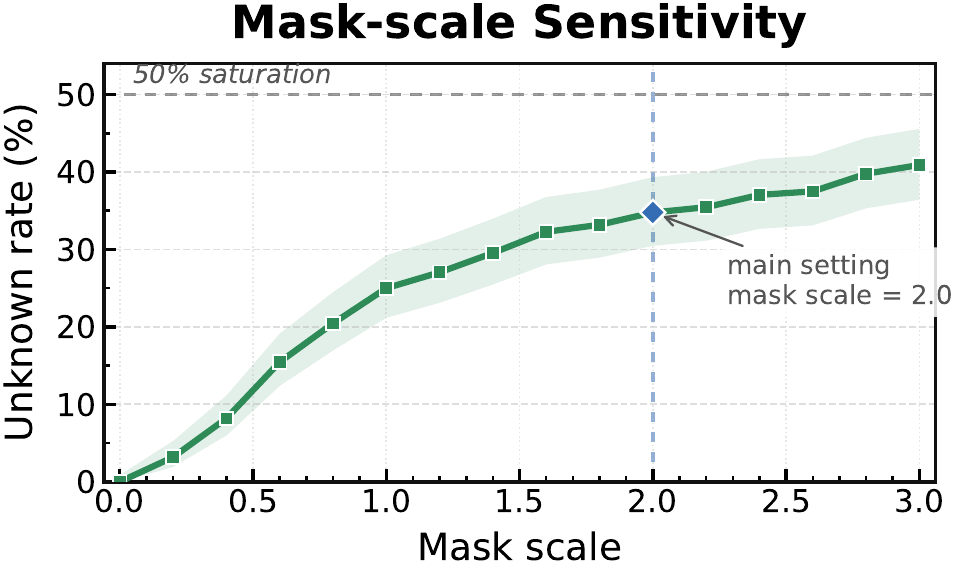}
    \vspace{-8pt}
    \caption{
    The \texttt{Unknown} rate increases with \texttt{mask\_scale}, but remains below 50\% even under the strongest masking.
    We use \texttt{mask\_scale}=2.0 as the default degradation setting.
    }
    \label{fig:mask_scale_sweep}
    \vspace{-20pt}
\end{wrapfigure}
We evaluate four medical VLMs: Hulu-Med 4B (36 layers, 32 heads, $d{=}2560$), Hulu-Med 7B, InternVL3.5 4B (36 layers, 32 heads, $d{=}2560$), and Qwen3-VL-8B, all running in FP16 on two NVIDIA GeForce RTX 5090 GPUs with 32\,GB peak memory. For the arbitration setting, we restrict evaluation to binary answer instances and construct textual interventions via an \textit{EVIDENCE} template: the support condition asserts the correct answer $y^+$ while the conflict condition asserts the alternative $y^-$. We retain only instances where the model predicts $y^+$ on clean input, remains correct under support, and flips to $y^-$ under conflict. For the brake setting, we construct degraded inputs by masking diagnostically relevant regions resolved from \texttt{detection.json} (preferring pixel level masks when available, otherwise falling back to detection rectangles) with black fill enlarged by \texttt{mask\_scale}$\,{=}\,2.0$. The prompt explicitly permits answering \texttt{unknown}; filtering keeps only instances answered correctly without masking. A mask scale sweep (\cref{fig:mask_scale_sweep}) confirms that the \texttt{Unknown} rate increases monotonically yet remains below 50\% even at \texttt{mask\_scale}$\,{=}\,3.0$, motivating intervention. After filtering, all retained examples are split 70/30 into a training split (layer tracing, head localization, threshold selection) and a validation split on which all reported metrics are evaluated.

\subsection{Metric Formal Definitions and Text-Only Results}
\label{app:metrics}

\textbf{\textit{(i)} Evaluation metrics.}
The main paper reports four metrics for arbitration failure (\cref{tab:main_vqa}) and three for brake failure (\cref{tab:image_conflict}).
Let $\hat{y}_i^{\mathrm{pre}}, \hat{y}_i^{\mathrm{post}}$ denote the model prediction on conflict instance~$i$ before and after intervention, and let $y_i^{+}, y_i^{-}$ denote the visually correct and textually misleading answers respectively.
For \textit{Arbitration setting:}
\begin{align}
\textsc{CFR}\ (\downarrow)
&= \tfrac{1}{|\mathcal{D}_{\mathrm{conf}}|} \textstyle\sum_{i \in \mathcal{D}_{\mathrm{conf}}}
\mathbb{1}\!\bigl[\hat{y}_i^{\mathrm{post}} = y_i^{-}\bigr], \label{eq:cfr}\\[3pt]
\textsc{Resist}\ (\uparrow)
&= 1 - \textsc{CFR}, \label{eq:resist}\\[3pt]
\mathrm{C2W}\ (\downarrow)
&= \tfrac{1}{|\mathcal{D}_{\mathrm{conf}}|} \textstyle\sum_{i}
\mathbb{1}\!\bigl[\hat{y}_i^{\mathrm{pre}}{=}y_i^{+}\;\wedge\;\hat{y}_i^{\mathrm{post}}{=}y_i^{-}\bigr], \label{eq:c2w}\\[3pt]
\mathrm{W2C}\ (\uparrow)
&= \tfrac{1}{|\mathcal{D}_{\mathrm{conf}}|} \textstyle\sum_{i}
\mathbb{1}\!\bigl[\hat{y}_i^{\mathrm{pre}}{=}y_i^{-}\;\wedge\;\hat{y}_i^{\mathrm{post}}{=}y_i^{+}\bigr]. \label{eq:w2c}
\end{align}
Here \textsc{Resist} denotes ``not following the misleading answer'' and can include correct, abstention, or other non misleading outputs.

\vspace{2pt}
\noindent\textit{Brake setting.}
Let $y^{\mathrm{unk}}$ denote abstention and $y^{\mathrm{com}}\in\{y^+,y^-\}$ a concrete answer; instances are drawn from the visually degraded subset~$\mathcal{D}_{\mathrm{deg}}$.
\begin{align}
\textsc{UR}\ (\uparrow)
&= \tfrac{1}{|\mathcal{D}_{\mathrm{deg}}|} \textstyle\sum_{i}
\mathbb{1}\!\bigl[\hat{y}_i^{\mathrm{post}} = y^{\mathrm{unk}}\bigr], \label{eq:ur}\\[3pt]
\mathrm{C2U}\ (\uparrow)
&= \tfrac{1}{|\mathcal{D}_{\mathrm{deg}}|} \textstyle\sum_{i}
\mathbb{1}\!\bigl[\hat{y}_i^{\mathrm{pre}}{=}y^{\mathrm{com}}\;\wedge\;\hat{y}_i^{\mathrm{post}}{=}y^{\mathrm{unk}}\bigr], \label{eq:c2u}\\[3pt]
\mathrm{U2O}\ (\downarrow)
&= \tfrac{1}{|\mathcal{D}_{\mathrm{deg}}|} \textstyle\sum_{i}
\mathbb{1}\!\bigl[\hat{y}_i^{\mathrm{pre}}{=}y^{\mathrm{unk}}\;\wedge\;\hat{y}_i^{\mathrm{post}}{=}y^{\mathrm{com}}\bigr]. \label{eq:u2o}
\end{align}

\textbf{\textit{(ii)} Text-only benchmark results.}
\cref{tab:text_result} reports results on two text-only medical QA benchmarks,
ConflictMedQA~\cite{wu2025conflictmedqa} and PubMedQA~\cite{jin2019pubmedqa},
using Qwen3-4B~\cite{yang2025qwen3} and
Llama3.2-3B~\cite{grattafiori2024llama}, where conflict is injected
purely through contradictory textual context without visual inputs.

\noindent\textbf{Analysis.}
Across both benchmarks, \method consistently lowers $\textsc{CFR}$ and improves $\textsc{Resist}$ over both the baseline and random head ablation.
The effect is especially strong on ConflictMedQA, where \method substantially reduces conflict compliance across both models and all injection positions while keeping $\mathrm{C2W}$ near zero.
On PubMedQA, \method remains effective under longer and more naturalistic medical contexts, with the strongest case reducing $\textsc{CFR}$ from 53.70 to 3.70 and increasing $\textsc{Resist}$ from 46.30 to 96.30.
Compared with random ablation using the same number of heads, \method achieves a better balance between $\mathrm{W2C}$ gains and $\mathrm{C2W}$ collateral damage, confirming that the selected heads are involved in conflict following rather than general answer generation.

\begin{table}[t]
\centering
\caption{\textbf{Threshold sensitivity} (Hulu-Med 4B).
\textit{Left}: arbitration on VQA-RAD (Before Q).
\textit{Right}: brake on SLAKE (image conflict).
Highlighted: final settings.}
\label{tab:threshold_sensitivity}
\scriptsize
\begin{minipage}[t]{0.5\textwidth}
    \centering
    \resizebox{\linewidth}{!}{
    \begin{tabular}{cccccc}
    \toprule[1.2pt]
    $\boldsymbol{\tau_s}$ & $\boldsymbol{\tau_d}$ & \textbf{CFR$\downarrow$} & \textbf{C2W$\downarrow$} & \textbf{W2C$\uparrow$} & \textbf{\#Heads} \\
    \midrule[0.9pt]
    0.133 & 0.550 & 18.7 & 4.9 & \textbf{30.8} & 25 \\
    0.271 & 0.183 & 24.6 & \textbf{0.5} & 20.5 & 2 \\
    0.156 & 0.445 & 22.1 & 7.9 & 30.5 & 31 \\
    \rowcolor{bestrow}
    \textbf{0.350} & \textbf{0.500} & \textbf{18.2} & 0.8 & 27.4 & 6 \\
    \bottomrule[1.2pt]
    \end{tabular}
    }
\end{minipage}\hfill
\begin{minipage}[t]{0.45\textwidth}
    \centering
    \resizebox{\linewidth}{!}{
    \begin{tabular}{cccccc}
    \toprule[1.2pt]
    $\boldsymbol{\tau_s}$ & $\boldsymbol{\tau_d}$ & \textbf{UR$\uparrow$} & \textbf{O2U$\uparrow$} & \textbf{U2O$\downarrow$} & \textbf{\#Heads} \\
    \midrule[0.9pt]
    0.2 & 0.5 & 57.58 & 20.45 & 0.76 & 4 \\
    0.0 & 0.3 & 54.55 & 16.67 & \textbf{0.00} & 3 \\
    0.0 & 0.8 & 62.88 & 26.52 & 1.52 & 9 \\
    \rowcolor{bestrow}
    \textbf{0.0} & \textbf{0.5} & \textbf{65.15} & \textbf{27.27} & \textbf{0.00} & 6 \\
    \bottomrule[1.2pt]
    \end{tabular}
    }
\end{minipage}
\vspace{-6pt}
\end{table}

\begin{table*}[!t]
\centering
\caption{\textbf{Text only benchmark results.}
\textit{CFR}: conflict compliance ($\downarrow$); \textit{Resist}: resist rate ($\uparrow$); \textit{C2W}: correct to wrong ($\downarrow$); \textit{W2C}: wrong to correct ($\uparrow$).
\textbf{Bold}: best per block; shaded rows: \method. Random uses the same head count.}
\label{tab:text_result}
\scriptsize
\setlength{\tabcolsep}{3pt}
\resizebox{0.85\textwidth}{!}{
\begin{tabular}{@{}l@{\hspace{8pt}}l@{\hspace{6pt}}lcccccccc@{}}
\toprule[1.2pt]
& & & \multicolumn{4}{c}{\textbf{ConflictMedQA}} & \multicolumn{4}{c}{\textbf{PubMedQA}} \\
\cmidrule(lr){4-7} \cmidrule(lr){8-11}
\textbf{Position} & \textbf{Model} & \textbf{Method}
& \textbf{CFR$\downarrow$} & \textbf{Resist$\uparrow$} & \textbf{C2W$\downarrow$} & \textbf{W2C$\uparrow$}
& \textbf{CFR$\downarrow$} & \textbf{Resist$\uparrow$} & \textbf{C2W$\downarrow$} & \textbf{W2C$\uparrow$} \\
\midrule[0.9pt]

\multirow{6}{*}{\textbf{Prefix}}
& \multirow{3}{*}{Qwen3-4B}
& Baseline & 83.33 & 16.67 & -- & -- & 49.67 & 50.33 & -- & -- \\
& & Random   & 94.25 & 5.75  & 12.07 & 1.15 & 57.17 & 42.83 & 12.17 & 4.67 \\
& & \cellcolor{bestrow}\method
& \cellcolor{bestrow}\textbf{37.36} & \cellcolor{bestrow}\textbf{62.64} & \cellcolor{bestrow}\textbf{0.00} & \cellcolor{bestrow}\textbf{45.98}
& \cellcolor{bestrow}\textbf{23.59} & \cellcolor{bestrow}\textbf{76.41} & \cellcolor{bestrow}\textbf{0.00} & \cellcolor{bestrow}\textbf{26.09} \\
\cmidrule(lr){2-11}
& \multirow{3}{*}{Llama3.2-3B}
& Baseline & 57.69 & 42.31 & -- & -- & 58.46 & 41.54 & -- & -- \\
& & Random   & 42.31 & 57.69 & 26.92 & \textbf{42.31} & 63.08 & 36.92 & 8.85 & 4.23 \\
& & \cellcolor{bestrow}\method
& \cellcolor{bestrow}\textbf{27.88} & \cellcolor{bestrow}\textbf{72.12} & \cellcolor{bestrow}\textbf{0.00} & \cellcolor{bestrow}29.81
& \cellcolor{bestrow}\textbf{46.15} & \cellcolor{bestrow}\textbf{53.85} & \cellcolor{bestrow}\textbf{1.15} & \cellcolor{bestrow}\textbf{13.46} \\

\midrule

\multirow{6}{*}{\textbf{Before Q}}
& \multirow{3}{*}{Qwen3-4B}
& Baseline & 94.83 & 5.17 & -- & -- & 53.70 & 46.30 & -- & -- \\
& & Random   & 89.66 & 10.34 & 1.15 & 6.32 & 46.30 & 53.70 & 3.91 & 11.30 \\
& & \cellcolor{bestrow}\method
& \cellcolor{bestrow}\textbf{68.39} & \cellcolor{bestrow}\textbf{31.61} & \cellcolor{bestrow}\textbf{1.15} & \cellcolor{bestrow}\textbf{27.59}
& \cellcolor{bestrow}\textbf{3.70} & \cellcolor{bestrow}\textbf{96.30} & \cellcolor{bestrow}\textbf{0.00} & \cellcolor{bestrow}\textbf{50.00} \\
\cmidrule(lr){2-11}
& \multirow{3}{*}{Llama3.2-3B}
& Baseline & 76.92 & 23.08 & -- & -- & 71.54 & 28.46 & -- & -- \\
& & Random   & 81.73 & 18.27 & 11.54 & 6.73 & 68.08 & 31.92 & 4.62 & 8.08 \\
& & \cellcolor{bestrow}\method
& \cellcolor{bestrow}\textbf{57.69} & \cellcolor{bestrow}\textbf{42.31} & \cellcolor{bestrow}\textbf{0.00} & \cellcolor{bestrow}\textbf{19.23}
& \cellcolor{bestrow}\textbf{59.23} & \cellcolor{bestrow}\textbf{40.77} & \cellcolor{bestrow}\textbf{1.54} & \cellcolor{bestrow}\textbf{13.85} \\

\midrule

\multirow{6}{*}{\textbf{Before A}}
& \multirow{3}{*}{Qwen3-4B}
& Baseline & 80.46 & 19.54 & -- & -- & 65.98 & 34.02 & -- & -- \\
& & Random   & 76.44 & 23.56 & 4.02 & 8.05 & 85.33 & 14.67 & 20.54 & 1.20 \\
& & \cellcolor{bestrow}\method
& \cellcolor{bestrow}\textbf{23.56} & \cellcolor{bestrow}\textbf{76.44} & \cellcolor{bestrow}\textbf{0.57} & \cellcolor{bestrow}\textbf{57.47}
& \cellcolor{bestrow}\textbf{51.41} & \cellcolor{bestrow}\textbf{48.59} & \cellcolor{bestrow}\textbf{0.00} & \cellcolor{bestrow}\textbf{14.57} \\
\cmidrule(lr){2-11}
& \multirow{3}{*}{Llama3.2-3B}
& Baseline & 43.27 & 56.73 & -- & -- & 66.92 & 33.08 & -- & -- \\
& & Random   & 48.08 & 51.92 & 16.35 & 11.54 & 72.31 & 27.69 & 9.23 & 3.85 \\
& & \cellcolor{bestrow}\method
& \cellcolor{bestrow}\textbf{31.73} & \cellcolor{bestrow}\textbf{68.27} & \cellcolor{bestrow}\textbf{2.88} & \cellcolor{bestrow}\textbf{14.42}
& \cellcolor{bestrow}\textbf{54.62} & \cellcolor{bestrow}\textbf{45.38} & \cellcolor{bestrow}\textbf{2.31} & \cellcolor{bestrow}\textbf{14.62} \\

\bottomrule[1.2pt]
\end{tabular}
}
\vspace{-4pt}
\end{table*}

\subsection{Key Parameter Selection and Sensitivity}

The threshold procedure selects $\calC^*$ by requiring heads to exceed $\tau_s$ on the suppression score ($\CER_{\mathrm{exc}}$ or $\mathrm{HR}$) while falling below $\tau_d$ on the disruption score ($\BCP_{\mathrm{exc}}$), with $\tau_s$ set near the 90th percentile of the suppression distribution and $\tau_d$ near the 10th percentile of the disruption distribution, yielding $|\calC_{\mathrm{arb}}^*| \in [4, 13]$ and $|\calC_{\mathrm{brk}}^*| \in [5, 7]$ depending on architecture. All metrics are computed on the held out 30\% split, averaged over three random seeds (standard deviations $<$ 1.5\,pp), with significance assessed via two sided permutation tests ($n{=}10{,}000$, $p < 0.01$). A threshold sensitivity sweep (\cref{tab:threshold_sensitivity}) confirms the trade off: looser settings improve correction ($\mathrm{W2C}$/$\mathrm{O2U}$) at the cost of increased collateral damage ($\mathrm{C2W}$/$\mathrm{U2O}$), while overly strict settings leave failure unresolved; the operating points balance target effect retention against competence preservation.

\section{Theoretical Analysis of Deconfounding}
\label{app:confounding}

This appendix formalizes the causal deconfounding argument underlying the head selection procedure in \cref{sec:localization}.
We first show that observational activation attribution is confounded (\cref{prop:confounding}), then prove that clean state interchange eliminates this bias (\cref{thm:deconfounding}), and finally establish that the $\CER$ metric inherits the deconfounding guarantee (\cref{cor:cer_valid}).

\subsection{Confounding in Observational Attribution}

\begin{proposition}[Confounded Observational Attribution]
\label{prop:confounding}
Under the SCM defined in \cref{sec:failures} with latent confounder $U_c$ and structural equations $\bh_l = f_l(\bh_{l-1}, X_v, X_t, \Theta, U_c)$, $Y = g(\bh_{L,t_\mathrm{ans}})$, the observational conditional $\mathbb{E}[Y \mid \ba_{l,i}]$ does not equal the interventional quantity $\mathbb{E}[Y \mid \doOp(\ba_{l,i} = a)]$ whenever $U_c$ is not d-separated from $Y$ given $\ba_{l,i}$ in the mutilated graph.
\end{proposition}

\begin{proof}
\textit{(i)~Active back door path.}
In the original DAG the path $\ba_{l,i} \leftarrow U_c \rightarrow Y$ is active because $U_c$ is a common cause that is never conditioned on.
Since $U_c$ encompasses lexical overlap, visual ambiguity, and cross modal coupling, it is neither directly measurable nor finite dimensional, so the back door criterion~\cite{pearl2009causality} cannot be satisfied by observational adjustment.

\textit{(ii)~Bias decomposition.}
Applying the truncated factorisation formula yields:
\begin{equation}
\mathbb{E}[Y \mid \ba_{l,i} {=} a] \;=\; \underbrace{\mathbb{E}[Y \mid \doOp(\ba_{l,i} {=} a)]}_{\text{causal effect}} \;+\; \underbrace{\mathrm{Bias}(U_c {\to} \ba_{l,i},\; U_c {\to} Y)}_{\text{confounding bias}},
\end{equation}
where the bias term is generically nonzero whenever $U_c$ simultaneously affects both the head activation and the model output through distinct structural equations.

\textit{(iii)~Implication.}
Activation magnitude ranking is therefore a \textit{biased estimator} of causal importance: a head may appear highly influential simply because $U_c$ drives both its activation and the erroneous output, without the head itself mediating the failure.
\end{proof}

\subsection{D Separation under Clean State Interchange}

\begin{theorem}[Deconfounding via Interchange Intervention]
\label{thm:deconfounding}
Let $\ba_{l,i}^{\mathrm{clean}}$ denote the activation of head $(l,i)$ computed from the conflict free counterfactual sharing $(X_v, q)$ with the original conflict sample.
Under $\doOp(\ba_{l,i} {:=} \ba_{l,i}^{\mathrm{clean}})$, the back door path $\ba_{l,i} \leftarrow U_c \rightarrow Y$ is blocked, and the measured effect
\begin{equation}
\mathrm{CDE}(l,i) \;=\; \mathbb{E}\bigl[Y \mid \doOp(\ba_{l,i} {:=} \ba_{l,i}^{\mathrm{clean}})\bigr] \;-\; \mathbb{E}\bigl[Y \mid \doOp(\ba_{l,i} {:=} \ba_{l,i}^{\mathrm{conflict}})\bigr]
\end{equation}
equals the \textbf{controlled direct effect} of the conflict mechanism mediated through head $(l,i)$.
\end{theorem}

\begin{proof}
The $\mathrm{do}$ operator $\doOp(\ba_{l,i} {:=} \ba_{l,i}^{\mathrm{clean}})$ replaces the structural equation for $\ba_{l,i}$ with a constant assignment, mutilating the graph by severing all incoming edges to $\ba_{l,i}$~\cite{pearl2009causality}.

\textit{(i)~Back door severed.}
Since $\ba_{l,i}$ is fixed to a constant that no longer depends on the $U_c$ realisation, the edge $U_c \to \ba_{l,i}$ is removed from the mutilated graph, blocking the only back door path $\ba_{l,i} \leftarrow U_c \rightarrow Y$.

\textit{(ii)~Front door preserved.}
All downstream structural equations remain intact.
The causal path $X_v \to \bh_{l-1} \to \bh_l \to Y$ continues to operate through other heads and FFN components, so the intervention does not destroy the forward computation.

\textit{(iii)~Input level matching.}
Both the conflict and conflict free samples share $(X_v, q)$, so the only variation in $U_c$ arises from the presence or absence of adversarial text $X_t^{\mathrm{conf}}$.
Fixing $\ba_{l,i}$ to its clean state value ensures constancy across both $U_c$ conditions, achieving d separation:
\begin{equation}
\ba_{l,i} \perp\!\!\!\perp U_c \;\mid\; \doOp(\ba_{l,i}) \quad \text{in the mutilated graph.}
\end{equation}

\noindent Combining \textit{(i)} through \textit{(iii)}, $\mathrm{CDE}(l,i)$ isolates the causal contribution of head $(l,i)$ to the conflict induced failure, free of confounding bias from $U_c$.
\end{proof}

\begin{corollary}[Validity of $\CER$ as Causal Metric]
\label{cor:cer_valid}
The conflict effect reduction metric $\CER_{\mathrm{exc}}(h) = \delta^{\mathrm{conflict}} - \delta^{\mathrm{excise}(h)}$ is a valid estimate of the controlled direct effect of head $h$ on the preference margin, provided the paired counterfactual shares $(X_v, q)$ with the conflict sample.
\end{corollary}

\begin{proof}
The excise operation ($\alpha_h \to 0^+$) is a special case of the interchange intervention in \cref{thm:deconfounding} that replaces the head output with a near zero vector.
The d separation guarantee holds for \textit{any} fixed replacement value, since the key property is that $\ba_{l,i}$ becomes constant with respect to $U_c$, not that it takes a particular value.
The $\CER$ metric then directly measures the shift in preference margin attributable to removing the causal contribution of head $h$ under the deconfounded regime.
\end{proof}

\section{Comparison with Existing Safety Methods}
\label{app:prior_method_comparison}

We include a controlled comparison in which all methods are evaluated under the same answer extraction rule and clean performance measurement protocol.
\cref{tab:baseline_protocol} summarises four representative baselines spanning three intervention paradigms: \textit{decoding guidance} (CCD~\cite{zhang2025ccd}, MARINE~\cite{zhou2025marine}), \textit{trait level activation steering} (ITI~\cite{li2024iti}), and \textit{pre generation hallucination detection} (HALP~\cite{chen2025halp}).
CCD and MARINE modify the output distribution using auxiliary contrastive or visual grounding signals from an external model, while ITI shifts activations along a global truthfulness direction learned from a small labelled set.
HALP trains lightweight probes on internal representations to predict hallucination risk before decoding.
In contrast, \method identifies and excises \textit{failure specific} causal head sets without external models or additional training.

\begin{table}[t]
\centering
\caption{\textbf{Inference time safety baselines.} Highlighted: the proposed method.}
\label{tab:baseline_protocol}
\scriptsize
\setlength{\tabcolsep}{5pt}
\renewcommand{\arraystretch}{1.08}
\resizebox{0.83\textwidth}{!}{
\begin{tabular}{llccc}
\toprule[1.2pt]
\textbf{Method}
& \textbf{Paradigm}
& \textbf{Extra Model}
& \textbf{Training}
& \textbf{Target Failure} \\
\midrule[0.9pt]
CCD~\cite{zhang2025ccd}
& Decoding guidance
& \cmark
& \xmark
& Textual conflict \\
MARINE~\cite{zhou2025marine}
& Decoding guidance
& \cmark
& \xmark
& Textual conflict \\
ITI~\cite{li2024iti}
& Activation steering
& \xmark
& \cmark
& General truthfulness \\
HALP~\cite{chen2025halp}
& Hallucination detection
& \xmark
& \cmark
& Pre generation risk \\
\rowcolor{bestrow}
\method
& Causal head repair
& \xmark
& \xmark
& Arbitration + brake \\
\bottomrule[1.2pt]
\end{tabular}}
\vspace{-6pt}
\end{table}

\begin{table}[t]
\centering
\caption{\textbf{Arbitration failure} (Hulu-Med 4B, VQA-RAD, Before Q). Bold: best.}
\label{tab:baseline_results_hulumed4b_text}
\scriptsize
\setlength{\tabcolsep}{5.5pt}
\renewcommand{\arraystretch}{1.05}
\resizebox{0.68\textwidth}{!}{
\begin{tabular}{lcccccc}
\toprule[1.2pt]
\textbf{Method}
& \textbf{Extra}
& \textbf{Train}
& \textbf{CFR$\downarrow$}
& \textbf{Resist$\uparrow$}
& \textbf{C2W$\downarrow$}
& \textbf{W2C$\uparrow$} \\
\midrule[0.9pt]
Baseline
& \xmark & \xmark
& 44.62 & 55.38 & 0.00 & 0.00 \\
Random heads
& \xmark & \xmark
& 40.77 & 59.23 & 1.03 & 4.87 \\
CCD
& \cmark & \xmark
& 31.28 & 68.72 & 1.28 & 13.34 \\
MARINE
& \cmark & \xmark
& 28.46 & 71.54 & 1.03 & 16.92 \\
ITI
& \xmark & \cmark
& 34.10 & 65.90 & 1.79 & 12.31 \\
\rowcolor{bestrow}
\method
& \xmark & \xmark
& \textbf{14.36} & \textbf{85.64} & \textbf{0.51} & \textbf{30.77} \\
\bottomrule[1.2pt]
\end{tabular}}
\vspace{-4pt}
\end{table}

\begin{table}[t]
\centering
\caption{\textbf{Brake failure} (Hulu-Med 4B, SLAKE image conflict). Bold: best.}
\label{tab:baseline_results_hulumed4b_vision}
\scriptsize
\setlength{\tabcolsep}{6.5pt}
\renewcommand{\arraystretch}{1.05}
\resizebox{0.72\textwidth}{!}{
\begin{tabular}{lccccc}
\toprule[1.2pt]
\textbf{Method}
& \textbf{Extra}
& \textbf{Train}
& \textbf{UR$\uparrow$}
& \textbf{O2U$\uparrow$}
& \textbf{U2O$\downarrow$} \\
\midrule[0.9pt]
Baseline
& \xmark & \xmark
& 37.88 & 0.00 & 0.00 \\
Random heads
& \xmark & \xmark
& 45.45 & 8.33 & 0.76 \\
HALP routed abstention
& \xmark & \cmark
& 56.82 & 21.21 & 2.27 \\
\rowcolor{bestrow}
\method
& \xmark & \xmark
& \textbf{65.15} & \textbf{27.27} & \textbf{0.00} \\
\bottomrule[1.2pt]
\end{tabular}}
\vspace{-4pt}
\end{table}

\noindent\textbf{Analysis.}
\textit{\textbf{For arbitration failure}} (\cref{tab:baseline_results_hulumed4b_text}), random head removal provides only modest improvement, confirming that the gain is not caused by nonspecific perturbation.
CCD and MARINE further reduce CFR by injecting auxiliary contrastive or grounding signals at decoding time, and ITI improves over the baseline via a global truthfulness shift.
However, \method achieves the \textit{lowest CFR and highest Resist while keeping C2W minimal}, indicating that targeted excision of the localized override circuit outperforms both output level calibration and generic activation steering.
\textit{\textbf{For brake failure}} (\cref{tab:baseline_results_hulumed4b_vision}), HALP routed abstention improves UR by detecting high risk cases before generation, but introduces nonzero U2O because the routing decision can abstain on cases already correctly handled.
\method yields a larger UR gain with \textit{zero U2O}, showing that direct intervention on the brake head set recovers abstention while better preserving previously abstaining predictions.

\noindent\textbf{Reproducibility.}
All methods share the same validation split, answer extraction rule, and clean performance measurement protocol, ensuring that metric differences reflect genuine intervention effects rather than evaluation interface discrepancies.

\subsection{Clean Accuracy Preservation under Baseline Methods}
\label{app:clean_accuracy_baseline_methods}

To complement the failure mitigation results, we compare the clean accuracy preservation of \method with several baseline interventions.
For arbitration failure, we report the \textit{Before Q} setting on VQA-RAD, following the main baseline comparison.
For brake failure, we report the SLAKE image conflict setting.
All values are measured on clean inputs before and after intervention.
For baseline methods that do not explicitly select causal heads, \#Heads denotes the matched effective intervention budget used for comparison.

\begin{table*}[!t]
\centering
\caption{\textbf{Clean accuracy preservation} under baseline methods. $\Delta$Acc denotes post minus pre intervention accuracy (\%). Bold: best preservation result.}
\label{tab:app_clean_accuracy_method_comparison}
\scriptsize
\setlength{\tabcolsep}{3.2pt}
\renewcommand{\arraystretch}{1.08}
\resizebox{\textwidth}{!}{
\begin{tabular}{@{}lllcccc cccc cccc@{}}
\toprule[1.2pt]
& & & \multicolumn{4}{c}{\textbf{Hulu-Med 4B}}
& \multicolumn{4}{c}{\textbf{InternVL3.5 4B}}
& \multicolumn{4}{c}{\textbf{Qwen3-VL-8B}} \\
\cmidrule(lr){4-7} \cmidrule(lr){8-11} \cmidrule(lr){12-15}
\textbf{Benchmark}
& \textbf{Setting}
& \textbf{Method}
& \textbf{\#Heads}
& \textbf{Original}
& \textbf{After}
& $\boldsymbol{\Delta}$\textbf{Acc}
& \textbf{\#Heads}
& \textbf{Original}
& \textbf{After}
& $\boldsymbol{\Delta}$\textbf{Acc}
& \textbf{\#Heads}
& \textbf{Original}
& \textbf{After}
& $\boldsymbol{\Delta}$\textbf{Acc} \\
\midrule[0.9pt]

\multirow{5}{*}{\textbf{VQA-RAD}}
& \multirow{5}{*}{Arbitration failure}
& Random heads
& 6 & 100.00 & 94.87 & \textcolor{red}{$-$5.13}
& 10 & 100.00 & 93.71 & \textcolor{red}{$-$6.29}
& 8 & 100.00 & 95.80 & \textcolor{red}{$-$4.20} \\

& & CCD
& 7 & 100.00 & 91.79 & \textcolor{red}{$-$8.21}
& 11 & 100.00 & 91.43 & \textcolor{red}{$-$8.57}
& 9 & 100.00 & 93.60 & \textcolor{red}{$-$6.40} \\

& & MARINE
& 7 & 100.00 & 92.31 & \textcolor{red}{$-$7.69}
& 11 & 100.00 & 92.00 & \textcolor{red}{$-$8.00}
& 9 & 100.00 & 94.20 & \textcolor{red}{$-$5.80} \\

& & ITI
& 8 & 100.00 & 91.03 & \textcolor{red}{$-$8.97}
& 12 & 100.00 & 90.86 & \textcolor{red}{$-$9.14}
& 10 & 100.00 & 93.10 & \textcolor{red}{$-$6.90} \\

\rowcolor{bestrow}
& & \method
& 6 & 100.00 & \textbf{96.41} & \textbf{\textcolor{red}{$-$3.59}}
& 10 & 100.00 & \textbf{96.57} & \textbf{\textcolor{red}{$-$3.43}}
& 8 & 100.00 & \textbf{97.80} & \textbf{\textcolor{red}{$-$2.20}} \\

\midrule[0.9pt]

\multirow{6}{*}{\textbf{SLAKE}}
& \multirow{6}{*}{Brake failure}
& Random heads
& 6 & 100.00 & 88.40 & \textcolor{red}{$-$11.60}
& 7 & 100.00 & 86.23 & \textcolor{red}{$-$13.77}
& 8 & 100.00 & 91.60 & \textcolor{red}{$-$8.40} \\

& & CCD
& 7 & 100.00 & 83.73 & \textcolor{red}{$-$16.27}
& 8 & 100.00 & 80.65 & \textcolor{red}{$-$19.35}
& 9 & 100.00 & 88.40 & \textcolor{red}{$-$11.60} \\

& & MARINE
& 7 & 100.00 & 85.07 & \textcolor{red}{$-$14.93}
& 8 & 100.00 & 82.58 & \textcolor{red}{$-$17.42}
& 9 & 100.00 & 89.30 & \textcolor{red}{$-$10.70} \\

& & ITI
& 8 & 100.00 & 86.22 & \textcolor{red}{$-$13.78}
& 9 & 100.00 & 84.19 & \textcolor{red}{$-$15.81}
& 10 & 100.00 & 90.10 & \textcolor{red}{$-$9.90} \\

& & HALP
& 8 & 100.00 & 84.61 & \textcolor{red}{$-$15.39}
& 9 & 100.00 & 81.94 & \textcolor{red}{$-$18.06}
& 10 & 100.00 & 87.90 & \textcolor{red}{$-$12.10} \\

\rowcolor{bestrow}
& & \method
& 6 & 100.00 & \textbf{90.67} & \textbf{\textcolor{red}{$-$9.33}}
& 7 & 100.00 & \textbf{88.87} & \textbf{\textcolor{red}{$-$11.13}}
& 8 & 100.00 & \textbf{93.20} & \textbf{\textcolor{red}{$-$6.80}} \\

\bottomrule[1.2pt]
\end{tabular}}
\vspace{-4pt}
\end{table*}

\noindent\textbf{Analysis.}
Across both failure modes, \method preserves clean accuracy better than the compared baselines under comparable intervention budgets.
For arbitration failure, \method produces the smallest clean accuracy drop while using the sparsest head set, indicating that the gain in conflict resistance is not obtained by broadly degrading clean visual question answering.
For brake failure, HALP and the generic intervention baselines introduce larger clean accuracy losses under matched or near matched budgets, whereas \method maintains the highest post intervention clean accuracy on all three backbones.
This supports the claim that localized head excision provides a more targeted safety intervention than broader representation steering or hallucination suppression baselines.

\subsection{Comparison with Fine Tuning Based Safety Adaptation}
\label{app:finetuning_comparison}

We compare \method with two training based baselines under the same Hulu-Med 4B settings (VQA-RAD Before Q for textual arbitration, SLAKE image conflict pool for visual brake failure). \textsc{Conflict SFT} fine tunes on conflict aware examples whose target is the clean gold answer (textual conflict) or \texttt{unknown} (degraded visual evidence). \textsc{Preference Tuning} applies RLHF optimization where the preferred response resists the conflict or abstains and the rejected response follows it. Both baselines use parameter efficient adaptation on the language model blocks. We also combine \textsc{Conflict SFT} with \method to test whether the training free intervention is complementary to fine tuning.

\begin{table*}[!t]
\centering
\caption{
\textbf{Comparison with fine tuning based safety adaptation on textual conflict.}
Hulu-Med 4B, VQA-RAD Before Q. Shaded rows denote \method.
}
\label{tab:finetuning_text_conflict}
\scriptsize
\setlength{\tabcolsep}{5.2pt}
\renewcommand{\arraystretch}{1.08}
\resizebox{0.81\textwidth}{!}{
\begin{tabular}{@{}lcccccc@{}}
\toprule[1.2pt]
\textbf{Method}
& \textbf{Training}
& \textbf{$\mathrm{CFR}\downarrow$}
& \textbf{Resist$\uparrow$}
& \textbf{C2W$\downarrow$}
& \textbf{W2C$\uparrow$}
& \textbf{Clean Acc.$\uparrow$} \\
\midrule[0.9pt]
Baseline
& \xmark
& 44.62
& 55.38
& 0.00
& 0.00
& 100.00 \\
Random heads
& \xmark
& 40.77
& 59.23
& 1.03
& 4.87
& 96.67 \\
\textsc{Conflict SFT}
& \cmark
& 22.31
& 77.69
& 2.05
& 24.36
& 97.18 \\
\textsc{Preference Tuning}
& \cmark
& 19.49
& 80.51
& 1.54
& 26.67
& 97.44 \\
\rowcolor{bestrow}
\method
& \xmark
& 14.36
& 85.64
& \textbf{0.51}
& 30.77
& 96.41 \\
\rowcolor{bestrow}
\textsc{Conflict SFT} + \method
& \cmark
& \textbf{10.77}
& \textbf{89.23}
& 1.03
& \textbf{35.90}
& 96.92 \\
\bottomrule[1.2pt]
\end{tabular}}
\vspace{-2mm}
\end{table*}

\begin{table*}[!t]
\centering
\caption{
\textbf{Comparison with fine tuning based safety adaptation on visual brake failure.}
Hulu-Med 4B, SLAKE image conflict pool. Shaded rows denote \method.
}
\label{tab:finetuning_brake_failure}
\scriptsize
\setlength{\tabcolsep}{5.2pt}
\renewcommand{\arraystretch}{1.08}
\resizebox{0.8\textwidth}{!}{
\begin{tabular}{@{}lcccccc@{}}
\toprule[1.2pt]
\textbf{Method}
& \textbf{Training}
& \textbf{$\mathrm{UR}\uparrow$}
& \textbf{O2U$\uparrow$}
& \textbf{U2O$\downarrow$}
& \textbf{Clean Acc.$\uparrow$}
& \textbf{Net Gain$\uparrow$} \\
\midrule[0.9pt]
Baseline
& \xmark
& 37.88
& 0.00
& 0.00
& 100.00
& 0.00 \\
Random heads
& \xmark
& 45.45
& 8.33
& 0.76
& 96.90
& 7.57 \\
\textsc{Conflict SFT}
& \cmark
& 54.55
& 19.70
& 2.27
& \textbf{96.97}
& 16.67 \\
\textsc{Preference Tuning}
& \cmark
& 57.58
& 21.97
& 1.52
& 96.21
& 19.70 \\
\rowcolor{bestrow}
\method
& \xmark
& 65.15
& 27.27
& \textbf{0.00}
& 90.67
& 27.27 \\
\rowcolor{bestrow}
\textsc{Conflict SFT} + \method
& \cmark
& \textbf{69.70}
& \textbf{32.58}
& 0.76
& 92.42
& \textbf{31.82} \\
\bottomrule[1.2pt]
\end{tabular}}
\vspace{-2mm}
\end{table*}

\noindent\textbf{Analysis.}
\Cref{tab:finetuning_text_conflict} shows that both \textsc{Conflict SFT} and \textsc{Preference Tuning} reduce $\mathrm{CFR}$ and raise Resist over the baseline, confirming that training on conflict aware supervision can teach the model to resist textual overrides. Nevertheless, \method achieves a lower $\mathrm{CFR}$ than either training based method without any parameter update, suggesting that the failure is concentrated in a sparse set of causal heads whose direct suppression yields a strong correction. The combined \textsc{Conflict SFT} plus \method row further lowers $\mathrm{CFR}$ to 10.77 and raises Resist to 89.23, indicating that the two approaches address complementary aspects of the problem.

\Cref{tab:finetuning_brake_failure} reveals a parallel pattern for visual brake failure. Fine tuning increases $\mathrm{UR}$ under degraded visual evidence but introduces nonzero $\mathrm{U2O}$, meaning some originally abstaining predictions are incorrectly shifted into concrete answers. \method obtains a larger $\mathrm{UR}$ gain and keeps $\mathrm{U2O}$ at zero, though it incurs a greater clean accuracy cost because the brake intervention directly suppresses answer commitment under insufficient evidence. Combining fine tuning with \method achieves the largest net gain (31.82), showing that fine tuning provides broad behavioral calibration and better clean accuracy preservation, while \method supplies a sparse, training free mechanistic intervention for failure specific causal repair at inference time.

\begin{table*}[!t]
\centering
\caption{
\textbf{Open ended textual conflict results.}
The judge classifies each response into gold claim, conflict claim, or qualified response. Shaded rows denote \method.
}
\label{tab:open_ended_text_conflict}
\scriptsize
\setlength{\tabcolsep}{4.0pt}
\renewcommand{\arraystretch}{1.08}
\resizebox{0.95\textwidth}{!}{
\begin{tabular}{@{}llcccc cccc@{}}
\toprule[1.2pt]
& & \multicolumn{4}{c}{\textbf{Hulu-Med 4B}}
& \multicolumn{4}{c}{\textbf{InternVL3.5 4B}} \\
\cmidrule(lr){3-6} \cmidrule(lr){7-10}
\textbf{Benchmark}
& \textbf{Method}
& Gold$\uparrow$
& Conflict$\downarrow$
& Qualified$\uparrow$
& Clean Fact.$\uparrow$
& Gold$\uparrow$
& Conflict$\downarrow$
& Qualified$\uparrow$
& Clean Fact.$\uparrow$ \\
\midrule[0.9pt]
\multirow{3}{*}{\textbf{VQA RAD}}
& Baseline
& 54.20 & 39.60 & 6.20 & 96.00
& 56.40 & 36.80 & 6.80 & 96.30 \\
& Random
& 57.10 & 35.70 & 7.20 & 95.10
& 58.20 & 34.70 & 7.10 & 95.40 \\
& \cellcolor{bestrow}\method
& \cellcolor{bestrow}\textbf{70.90}
& \cellcolor{bestrow}\textbf{21.80}
& \cellcolor{bestrow}\textbf{7.30}
& \cellcolor{bestrow}94.70
& \cellcolor{bestrow}\textbf{78.40}
& \cellcolor{bestrow}\textbf{13.20}
& \cellcolor{bestrow}\textbf{8.40}
& \cellcolor{bestrow}95.00 \\
\midrule[0.9pt]
\multirow{3}{*}{\textbf{SLAKE}}
& Baseline
& 51.80 & 42.20 & 6.00 & 95.30
& 59.10 & 34.60 & 6.30 & 96.80 \\
& Random
& 54.40 & 38.60 & 7.00 & 94.60
& 60.30 & 32.20 & 7.50 & 96.10 \\
& \cellcolor{bestrow}\method
& \cellcolor{bestrow}\textbf{68.60}
& \cellcolor{bestrow}\textbf{23.10}
& \cellcolor{bestrow}\textbf{8.30}
& \cellcolor{bestrow}93.80
& \cellcolor{bestrow}\textbf{82.60}
& \cellcolor{bestrow}\textbf{9.10}
& \cellcolor{bestrow}\textbf{8.30}
& \cellcolor{bestrow}95.90 \\
\bottomrule[1.2pt]
\end{tabular}}
\vspace{-2mm}
\end{table*}

\begin{table*}[t]
\centering
\caption{
\textbf{Open ended degraded vision results.}
The judge evaluates whether the response expresses calibrated uncertainty or makes an unsupported diagnosis. Shaded rows denote \method.
}
\label{tab:open_ended_degraded_vision}
\scriptsize
\setlength{\tabcolsep}{4.0pt}
\renewcommand{\arraystretch}{1.08}
\resizebox{\textwidth}{!}{
\begin{tabular}{@{}llcccc cccc@{}}
\toprule[1.2pt]
& & \multicolumn{4}{c}{\textbf{Hulu-Med 4B}}
& \multicolumn{4}{c}{\textbf{InternVL3.5 4B}} \\
\cmidrule(lr){3-6} \cmidrule(lr){7-10}
\textbf{Benchmark}
& \textbf{Method}
& Uncert.$\uparrow$
& Unsup.$\downarrow$
& Evid. Req.$\uparrow$
& Clean Compl.$\uparrow$
& Uncert.$\uparrow$
& Unsup.$\downarrow$
& Evid. Req.$\uparrow$
& Clean Compl.$\uparrow$ \\
\midrule[0.9pt]
\multirow{3}{*}{\textbf{SLAKE}}
& Baseline
& 36.20 & 59.10 & 4.70 & 95.60
& 24.80 & 69.40 & 5.80 & 94.80 \\
& Random
& 44.00 & 50.70 & 5.30 & 94.70
& 18.60 & 75.20 & 6.20 & 93.70 \\
& \cellcolor{bestrow}\method
& \cellcolor{bestrow}\textbf{63.40}
& \cellcolor{bestrow}\textbf{29.90}
& \cellcolor{bestrow}\textbf{6.70}
& \cellcolor{bestrow}90.80
& \cellcolor{bestrow}\textbf{34.80}
& \cellcolor{bestrow}\textbf{58.60}
& \cellcolor{bestrow}\textbf{6.60}
& \cellcolor{bestrow}92.80 \\
\midrule[0.9pt]
\multirow{3}{*}{\textbf{HealMed-VQA}}
& Baseline
& 73.60 & 21.80 & 4.60 & 96.20
& 76.40 & 18.70 & 4.90 & 95.60 \\
& Random
& 79.10 & 15.80 & 5.10 & 95.10
& 78.80 & 16.00 & 5.20 & 94.70 \\
& \cellcolor{bestrow}\method
& \cellcolor{bestrow}\textbf{89.40}
& \cellcolor{bestrow}\textbf{5.90}
& \cellcolor{bestrow}4.70
& \cellcolor{bestrow}91.70
& \cellcolor{bestrow}\textbf{88.20}
& \cellcolor{bestrow}\textbf{7.30}
& \cellcolor{bestrow}4.50
& \cellcolor{bestrow}92.40 \\
\bottomrule[1.2pt]
\end{tabular}}
\vspace{-2mm}
\end{table*}

\begin{table}[t]
\centering
\caption{
\textbf{Online computational cost of head intervention on VQA-RAD.}
Hooked masking is the analysis implementation; true head skipping offers marginal FLOP savings.
}
\label{tab:computational_cost_online}
\scriptsize
\setlength{\tabcolsep}{3.5pt}
\resizebox{\textwidth}{!}{
\begin{tabular}{llcccccc}
\toprule[1.2pt]
\textbf{Model} & \textbf{Implementation} & \textbf{\#Heads} & \textbf{Normal Lat.} & \textbf{Interv. Lat.} & \textbf{$\Delta$Lat.} & \textbf{Normal FLOPs} & \textbf{$\Delta$FLOPs} \\
\midrule[0.9pt]
Hulu-Med 4B
& Hooked masking
& 6
& 206.39 ms
& 210.31 ms
& +1.90\%
& 19539.96G
& 0.00\% \\
Hulu-Med 4B
& True head skipping
& 6
& 207.11 ms
& 208.10 ms
& +0.48\%
& 19539.96G
& 0.11\% lower \\
InternVL3.5 4B
& Hooked masking
& 10
& 43.27 ms
& 49.43 ms
& +14.24\%
& 3629.54G
& 0.00\% \\
InternVL3.5 4B
& True head skipping
& 10
& 43.91 ms
& 47.03 ms
& +7.11\%
& 3629.54G
& 0.10\% lower \\
\bottomrule[1.2pt]
\end{tabular}}
\end{table}

\begin{table}[h]
\centering
\caption{
\textbf{Offline cost of head selection.}
The localization cost is paid once before deployment.
Prompt forward counts include diagnostic passes for layer tracing, head scoring, and selection.
}
\label{tab:computational_cost_offline}
\scriptsize
\setlength{\tabcolsep}{4.8pt}
\renewcommand{\arraystretch}{1.08}
\resizebox{0.96\textwidth}{!}{
\begin{tabular}{llcccc}
\toprule[1.2pt]
\textbf{Model}
& \textbf{Setting}
& \textbf{Loc. Samples}
& \textbf{Total Forwards}
& \textbf{Forwards / Sample}
& \textbf{Param. Update} \\
\midrule[0.9pt]
Hulu-Med 4B
& VQA-RAD Before Q
& 391
& 111,342
& 284.8
& \xmark \\
InternVL3.5 4B
& VQA-RAD Before Q
& 698
& 158,446
& 227.0
& \xmark \\
Qwen3-VL-8B
& VQA-RAD Before Q
& 698
& 241,508
& 346.0
& \xmark \\
\bottomrule[1.2pt]
\end{tabular}}
\end{table}

\section{Computational Cost}
\label{app:computational_cost}
\Cref{tab:computational_cost_online} reports the online inference overhead after head selection. Hooked masking adds less than 2\% latency on Hulu-Med 4B and about 14\% on InternVL3.5 4B; true head skipping reduces this further but the sparse intervention level is too small for meaningful FLOP savings without fused kernels. These results clarify that \method is a safety oriented post hoc intervention rather than an inference acceleration method.

\Cref{tab:computational_cost_offline} reports the one time offline localization cost. The full head selection pipeline requires 227 to 346 forward passes per sample across the evaluated models, substantially larger than a single inference pass but involving no parameter updates or gradient computation. After localization the selected head set is reused for all subsequent inference with only the low overhead hook shown above. Therefore, \method trades a moderate one time diagnostic cost for low overhead inference time repair.
\section{Probe and Tuned Lens Implementation}
\label{app:probe_details}

\subsection{Temporal Probe Architecture}

Both Probe A (conflict detection) and Probe B (follow prediction) are implemented as lightweight linear classifiers:
\begin{equation}
\hat{p}_l = \sigma(\mathbf{w}_l^\top \mathbf{h}_l + b_l),
\end{equation}
where $\mathbf{h}_l \in \mathbb{R}^{d}$ is the residual stream representation at layer $l$, $\mathbf{w}_l \in \mathbb{R}^{d}$ and $b_l \in \mathbb{R}$ are learned parameters, and $\sigma$ is the sigmoid function.

\paragraph{Training details.}
Probe A is trained as a binary classifier for conflict detection, distinguishing inputs with injected textual conflict from clean inputs.
Probe B is trained as a binary classifier for answer arbitration, predicting whether the model follows the textually misleading answer $y^{-}$ or the visually grounded answer $y^{+}$ on conflict instances with unambiguous predictions.
Both probes are trained only on the training split and evaluated on held-out examples.

Both probes are implemented as independent linear classifiers trained with the Adam optimiser (learning rate $10^{-3}$, weight decay $10^{-4}$) for up to 50 epochs, with early stopping based on validation AUROC (patience = 10). We train a separate probe for each layer without any cross-layer parameter sharing or feature aggregation across layers.

\subsection{Tuned Lens Configuration}

The tuned lens~\cite{belrose2023eliciting} transforms intermediate representations to the output vocabulary space:
\begin{equation}
\mathrm{TL}_l(\mathbf{h}_l) = \mathbf{W}_U \cdot f_l(\mathbf{h}_l) + \mathbf{b}_U,
\end{equation}
where $f_l$ is a learned affine transformation (one per layer) and $\mathbf{W}_U$, $\mathbf{b}_U$ are the unembedding matrix and bias used for projection.

\paragraph{Training details.}
Each layer-specific transformation $f_l$ is parameterised as an affine mapping
$f_l(x)=A_lx+b_l$, where $A_l\in\mathbb{R}^{d\times d}$.
The tuned lens is trained on a held-out medical text corpus by minimising the KL divergence between the transformed logits $\mathrm{TL}_l(\mathbf{h}_l)$ and the model's final output distribution.
Optimisation is performed with Adam using a fixed learning rate schedule.
At evaluation time, we track preference dynamics across layers by computing the logit difference
$\mathrm{TL}_l(\mathbf{h}_l)[y^+] - \mathrm{TL}_l(\mathbf{h}_l)[y^-]$.

\subsection{Head-Level Tuned-Lens Preference Distributions}
\label{app:multihead_lens}

To complement the layer-level Tuned-Lens trajectories, we analyse head-level preference-margin distributions under text-conflict and image-conflict settings.

For text-conflict examples, the horizontal axis is the gold-versus-wrong preference margin:
\begin{equation}
    m_h = \mathrm{TL}^{(h)}(y^+) - \mathrm{TL}^{(h)}(y^-),
\end{equation}
where larger values indicate stronger preference for the gold answer.

For image-conflict examples, where the desired behavior under masked visual evidence is abstention, the horizontal axis is the abstention margin:
\begin{equation}
    a_h =
    \mathrm{TL}^{(h)}(y^{\mathrm{unk}})
    -
    \max_{y\in\mathcal{Y}_{\mathrm{con}}}
    \mathrm{TL}^{(h)}(y),
\end{equation}
where larger values indicate stronger preference for \texttt{unknown} over the best competing concrete answer. In Figure~\ref{fig:multihead_lens_visual}, NC corresponds to the unmasked image and is expected to favour a concrete answer; after masking the diagnostically relevant region, IC should shift rightward if the model correctly abstains under insufficient visual evidence. However, the IC distribution remains close to NC for many heads, indicating that visual degradation does not reliably increase the abstention margin, a pattern reflecting persistent concrete answer commitment consistent with brake failure. In Figure~\ref{fig:multihead_lens_text}, NC is expected to lie to the right of IC because clean inputs favour the gold answer while conflict inputs push the model toward the misleading alternative; the observed leftward shift from NC to IC captures textual override at the head level preference distribution. Together, these two patterns delineate the two failure modes: \textit{textual conflict produces a clear gold to wrong preference shift}, whereas \textit{image degradation often fails to elicit the expected concrete to unknown shift}, underscoring an asymmetry in how modality specific evidence is arbitrated at the attention head level.

\begin{figure*}[t]
\centering
\IfFileExists{pic/tuned_lens_head_distribution_val_text.pdf}
{
    \includegraphics[
        width=\textwidth,
        height=0.88\textheight,
        keepaspectratio
    ]{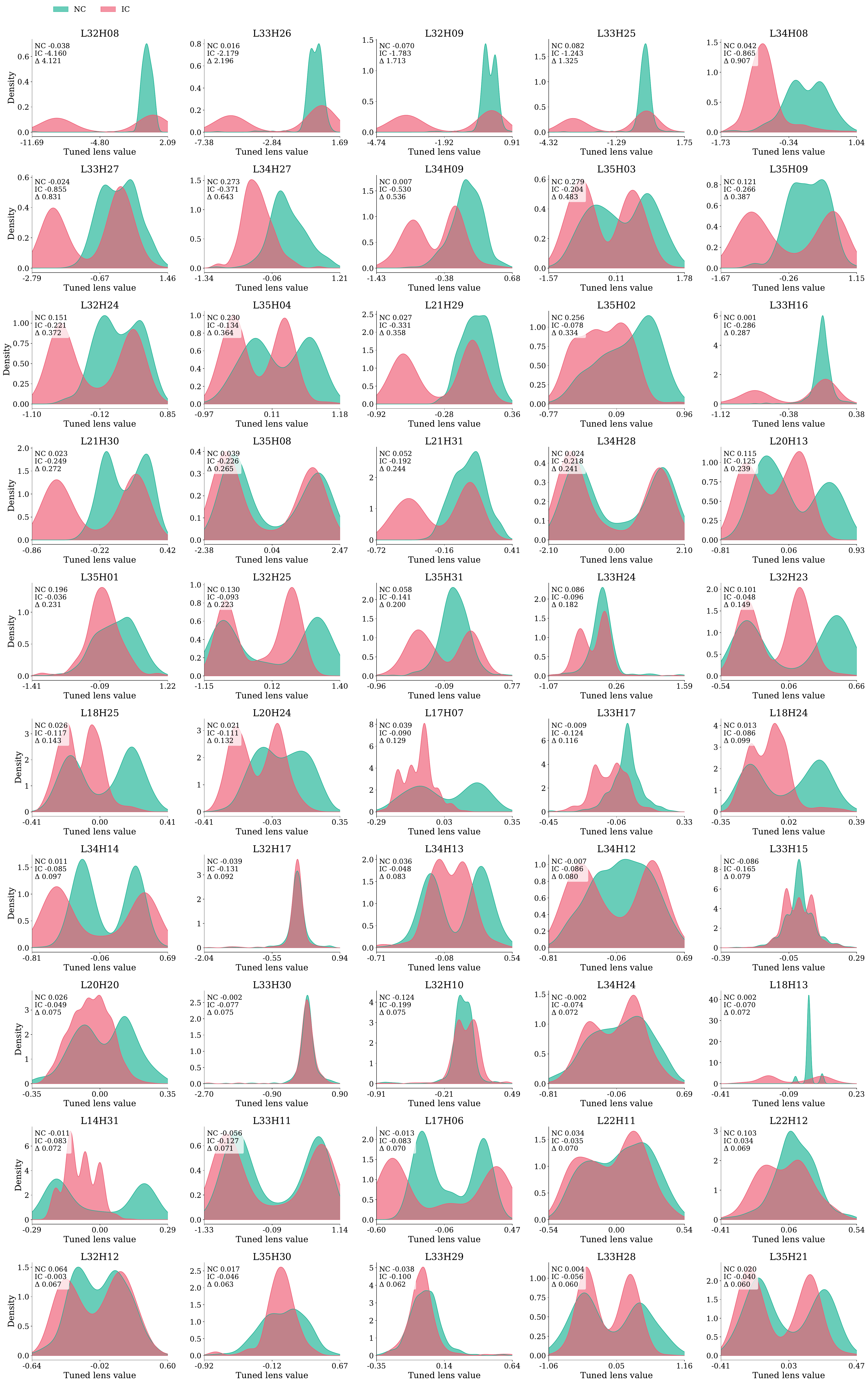}
}
{
    \fbox{
    \parbox{0.95\textwidth}{
    \centering
    \vspace{12pt}
    [\,Figure placeholder: \texttt{pic/tuned\_lens\_head\_distribution\_val\_text.pdf} not yet generated\,]
    \vspace{12pt}
    }}
}
\caption{
\textbf{Head-level tuned-lens preference distributions under textual conflict.}
Each panel shows one selected arbitration-sensitive head.
The horizontal axis is the gold-versus-wrong preference margin, $\log p(y^+) - \log p(y^-)$, under the no-conflict condition (NC, green) and the injected-conflict condition (IC, red).
Larger values indicate stronger preference for the gold answer.
Across many heads, IC shifts left relative to NC, showing that textual conflict pushes head-level preferences from the correct answer toward the misleading answer.
}
\label{fig:multihead_lens_text}
\end{figure*}

\begin{figure*}[t]
\centering
\IfFileExists{pic/tuned_lens_head_distribution_val_visual.pdf}
{
    \includegraphics[
        width=\textwidth,
        height=0.88\textheight,
        keepaspectratio
    ]{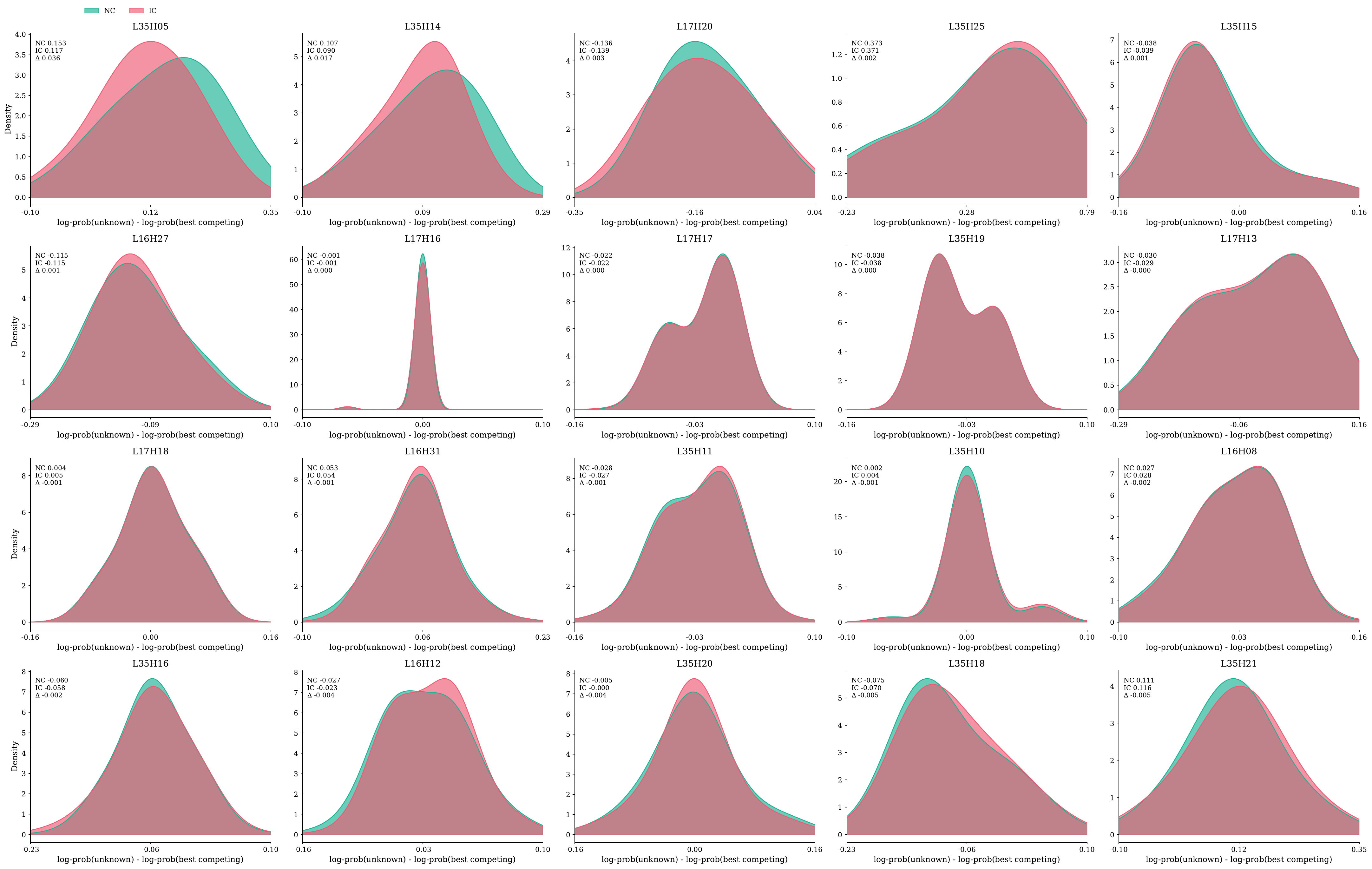}
}
{
    \fbox{
    \parbox{0.95\textwidth}{
    \centering
    \vspace{12pt}
    [\,Figure placeholder: \texttt{pic/tuned\_lens\_head\_distribution\_visual.pdf} not yet generated\,]
    \vspace{12pt}
    }}
}
\caption{
\textbf{Head-level tuned-lens abstention distributions under visual degradation.}
Each panel shows one selected image-conflict-sensitive head.
The horizontal axis is the abstention margin, $\log p(\texttt{unknown}) - \log p(y_{\mathrm{best}})$, where $y_{\mathrm{best}}$ is the best competing concrete answer.
NC denotes the original image, while IC denotes the masked image with relevant evidence removed.
If the model properly abstains after masking, IC should shift right.
Instead, NC and IC largely overlap for many heads, indicating persistent concrete-answer commitment under degraded visual evidence.
}
\label{fig:multihead_lens_visual}
\end{figure*}

\section{Patching Strategy Analysis}
\label{app:patching}

\paragraph{Clean-state definition.}
For each conflict instance $(x_v, x_t^{\mathrm{conflict}}, q)$, we define the \emph{clean state} as the hidden representation obtained by running the same visual input with a non-conflicting textual context: $\mathbf{h}_l^{\mathrm{clean}} = f_l(x_v, x_t^{\mathrm{clean}}, q)$. The patching intervention replaces the conflict-state hidden representation at the target head with this clean-state counterpart.

\paragraph{Local patch window.}
Because the clean and conflict prompts may differ in length, full-sequence replacement can introduce token-alignment artifacts. We therefore apply clean-state replacement only within a local window around the answer position. Let $K_{\mathrm{patch}}$ denote the number of answer-neighbouring tokens to patch. For each conflict run, we align clean and conflict states by their relative offset from the answer token and replace only the hidden states inside this local window. This preserves most of the conflict prompt while testing whether the answer-neighbouring residual states carry the textual override signal.
\paragraph{Patch window metrics.}
For each patch window size $K_{\mathrm{patch}}$, we obtain a layer-wise score curve $S_K(l)$ and evaluate its stability using two complementary statistics.

\textbf{Robustness} quantifies deviation from the consensus trace $\bar{S}(l) = \mathrm{median}_{K}\,\widetilde{S}_K(l)$:
\begin{equation}
    \mathrm{Rob}(K)
    =
    1 - \frac{1}{L}\sum_{l=1}^{L}
    \left|
    \widetilde{S}_K(l)-\bar{S}(l)
    \right|,
    \quad \widetilde{S}_K \triangleq \mathrm{MinMaxNorm}(S_K).
\end{equation}

\textbf{Smoothness} quantifies adjacent-layer variation:
\begin{equation}
    \mathrm{Smo}(K)
    =
    \frac{1}{L-1}
    \sum_{l=1}^{L-1}
    \left|S_K(l+1)-S_K(l)\right|.
\end{equation}

$\mathrm{Rob}\!\uparrow$ indicates stronger agreement with the local-window consensus; $\mathrm{Smo}\!\downarrow$ indicates less layer-to-layer fluctuation. However, Smoothness alone is insufficient for selecting $K_{\mathrm{patch}}$: large windows can overwrite non-conflict token states, producing artificially flat curves while weakening causal locality. We therefore use Robustness as the primary criterion and report Smoothness as a diagnostic.

\begin{table*}[!t]
\centering
\caption{
\textbf{Patch window sensitivity.}
Robustness (R)$\uparrow$ measures agreement with the local window consensus; Smoothness (S)$\downarrow$ measures adjacent layer variation.
Best R and best S are bolded separately for each model.
}
\label{tab:patch_k_selection}
\small
\setlength{\tabcolsep}{3.5pt}
\resizebox{\textwidth}{!}{
\begin{tabular}{lcccccc}
\toprule[1.2pt]
& \multicolumn{6}{c}{\textbf{Model (R / S)}} \\
\cmidrule(lr){2-7}
\textbf{$K_{\mathrm{patch}}$}
& \textbf{Qwen3-4B}
& \textbf{Llama3.2-3B}
& \textbf{Hulu-Med-4B}
& \textbf{Hulu-Med-7B}
& \textbf{InternVL3.5-4B}
& \textbf{Qwen3-VL-8B} \\
\midrule[0.9pt]

2
& 0.9718 / 0.8577
& 0.9214 / 0.1475
& 0.9288 / 0.0671
& 0.9346 / 0.0648
& 0.9478 / 0.1633
& 0.9415 / 0.1586 \\

4
& 0.9792 / 0.8553
& 0.9237 / 0.1483
& 0.9506 / 0.0687
& 0.9569 / 0.0665
& 0.9868 / 0.1775
& 0.9821 / 0.1694 \\

\rowcolor{bestrow}
\textbf{8}
& \textbf{0.9991} / 0.8449
& \textbf{0.9987} / 0.1501
& \textbf{0.9991} / 0.0714
& \textbf{0.9988} / 0.0698
& \textbf{0.9948} / 0.1781
& \textbf{0.9962} / 0.1726 \\

16
& 0.9160 / 0.8983
& 0.9561 / 0.1459
& 0.9736 / 0.0741
& 0.9758 / 0.0726
& 0.9733 / 0.1818
& 0.9714 / 0.1769 \\

32
& 0.8206 / 0.2564
& 0.9004 / \textbf{0.0442}
& 0.8255 / \textbf{0.0153}
& 0.8389 / \textbf{0.0147}
& 0.8595 / \textbf{0.0535}
& 0.8668 / \textbf{0.0512} \\

Full
& 0.6966 / \textbf{0.2396}
& 0.7957 / 0.0645
& 0.7465 / 0.0173
& 0.7584 / 0.0161
& 0.8560 / 0.1046
& 0.8497 / 0.0973 \\

\bottomrule[1.2pt]
\end{tabular}
}
\end{table*}

\noindent\textbf{Analysis.}
\Cref{tab:patch_k_selection} shows that $K_{\mathrm{patch}}=8$ achieves the highest Robustness across all evaluated models, indicating that a compact answer-centred window suffices to recover a stable layer-wise conflict trace. Larger windows and full-sequence patching reduce Robustness, suggesting they introduce token-alignment noise and dilute the local causal signal. Although these settings sometimes yield lower Smoothness, the smoother curves mainly reflect over-patching rather than improved localization. We therefore use $K_{\mathrm{patch}}=8$ in all main experiments and exclude full-sequence patching from the consensus curve.

\section{Downstream Propagation after Arbitration Head Ablation}
\label{app:downstream_propagation}

\begin{figure*}[t]
    \centering
    \begin{tabular}{c@{\hspace{6pt}}c}
    \raisebox{3.8\height}{\rotatebox[origin=c]{90}{\fontsize{5}{6}\textbf{Hulu-Med 4B}}} &
    \includegraphics[width=0.95\linewidth]{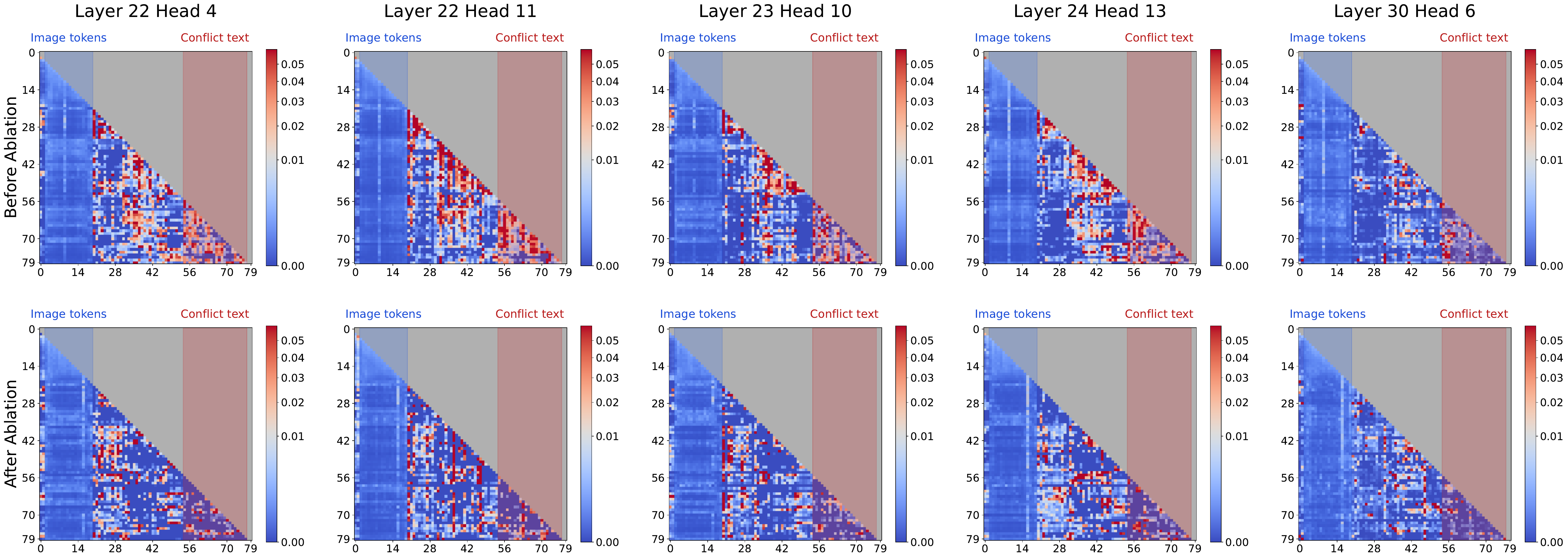} \\[8pt]

    \raisebox{4\height}{\rotatebox[origin=c]{90}{\fontsize{5}{6}\selectfont\textbf{InternVL3.5}}} &
    \includegraphics[width=0.95\linewidth]{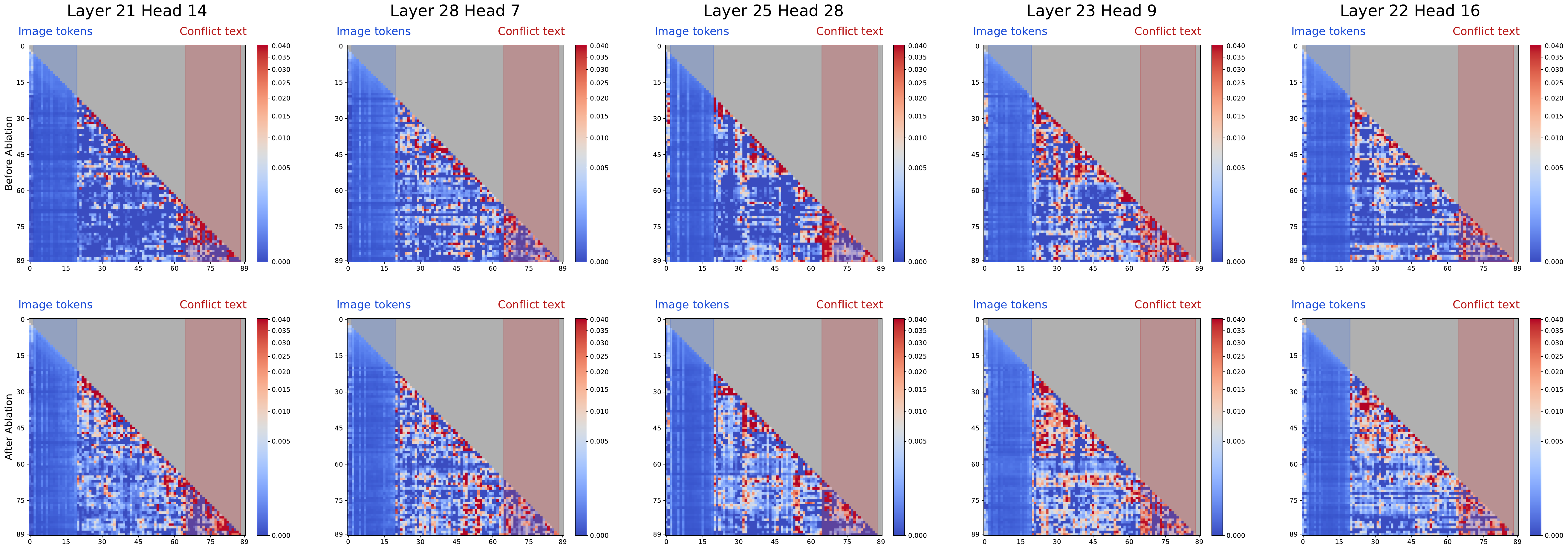} \\

    \end{tabular}

    \caption{
    \textbf{Downstream attention after arbitration head ablation.}
    Heatmaps compare the same conflict input before and after excising the selected arbitration heads.
    Blue and red regions denote image tokens and injected conflict text, respectively.
    Ablation weakens downstream attention to the conflict text region, consistent with reduced textual override propagation.
    }
    \label{fig:anchoring}
\end{figure*}

We further examine whether arbitration head ablation changes how conflict information is propagated in later layers.
To isolate the effect of the intervention, the input is kept fixed and attention patterns are compared before and after excising the selected arbitration heads.

\paragraph{Setup.}
For each downstream attention head $h$, the attention mass from the answer position to image tokens $\mathcal{V}$ and injected conflict text tokens $\mathcal{C}$ is computed as
\begin{equation}
\alpha_v^{(h)} =
\frac{1}{|\mathcal{V}|}\sum_{j\in\mathcal{V}} A_{t_{\mathrm{ans}},j}^{(h)},
\qquad
\alpha_c^{(h)} =
\frac{1}{|\mathcal{C}|}\sum_{j\in\mathcal{C}} A_{t_{\mathrm{ans}},j}^{(h)},
\end{equation}
where $A^{(h)}$ is the attention matrix of head $h$ and $t_{\mathrm{ans}}$ denotes the answer position.
The quantities are measured on the same conflict input before and after arbitration head ablation.

\paragraph{Findings.}
As shown in \cref{fig:anchoring}, downstream layers attend visibly to the injected conflict text before intervention.
After excising the selected arbitration heads, attention to the conflict text region decreases and becomes more diffuse over the remaining context.
This pattern supports the interpretation that arbitration head ablation attenuates downstream propagation of textual override signals, providing qualitative evidence for the reduction in conflict following in the main results.

\section{Layer-wise Attenuation Sensitivity for Brake localization}
\label{app:layer_scale_sensitivity}

We evaluate the sensitivity of brake localization to the layer-wise attenuation coefficient $\beta$ used during image-conflict tracing.
This coefficient corresponds to the scaling factor in $\mathrm{attenuate}(\calL;\beta)$ in \cref{eq:objective_brake}.
In implementation, it is denoted as \texttt{layer\_scale}.
Smaller values impose stronger attenuation, whereas $\beta=1.0$ recovers the original forward pass and therefore provides an unperturbed reference.

\begin{table}[t]
\centering
\caption{
\textbf{Layer wise attenuation sensitivity.}
$\Delta\UR$ denotes the post intervention increase in unknown rate on the held out image conflict split.
$\HR$ is the layer level hallucination relief score in \cref{eq:hr}.
The highlighted row denotes the setting used for brake tracing.
}
\label{tab:layer-scale-summary}
\small
\setlength{\tabcolsep}{8pt}
\begin{tabular}{cccc}
\toprule[1.2pt]
\textbf{$\beta$} & \textbf{\# Selected Heads} & \textbf{$\Delta\UR$} & \textbf{$\HR$} \\
\midrule[0.9pt]
0.0 & 6 & 7.23 & 0.7477 \\
\rowcolor{bestrow}
\textbf{0.2} & \textbf{7} & \textbf{9.68} & 0.7444 \\
0.4 & 5 & 3.35 & 0.5581 \\
0.6 & 6 & 3.23 & 0.4156 \\
0.8 & 5 & 3.24 & 0.2455 \\
1.0 & 4 & 3.23 & -0.0001 \\
\bottomrule[1.2pt]
\end{tabular}
\end{table}

\noindent\textbf{Analysis.}
\Cref{tab:layer-scale-summary} shows that brake localization requires a non trivial attenuation signal.
For $\beta<1$, attenuating candidate layers yields positive $\HR$, indicating that these layers contribute to concrete answer commitment under degraded visual evidence.
As $\beta$ approaches 1.0, the perturbation vanishes and $\HR$ decreases toward zero, as expected from an unchanged forward pass.

The tracing score alone does not determine held out intervention performance.
Although $\beta=0.0$ gives the largest $\HR$, it produces a smaller $\Delta\UR$ than $\beta=0.2$ after head selection.
The setting $\beta=0.2$ achieves the largest unknown rate gain while retaining a strong hallucination relief signal.
Larger values from $\beta=0.4$ to $\beta=1.0$ weaken the localization contrast and yield smaller abstention gains.
Accordingly, $\beta=0.2$ is used for layer wise brake tracing in the main experiments.
\section{Robustness to Local Visual Noise}
\label{app:noise_robustness}

In addition to hard region masking, a local visual noise variant is evaluated on the Hulu-Med 4B SLAKE image conflict pool. Local noise preserves the annotated diagnostic region while degrading its signal quality, approximating acquisition noise, motion artefacts, or low-dose imaging conditions. Stochastic noise is applied only to the annotated diagnostic region, and the head set selected under the main image conflict setting is reused without re-localizing heads for each noise level.

\begin{table}[h]
\centering
\caption{
\textbf{Robustness to local visual noise.}
Brake head intervention remains effective under local corruption at varying noise strengths.
}
\label{tab:noise_robustness}
\scriptsize
\setlength{\tabcolsep}{7pt}
\resizebox{0.60\textwidth}{!}{
\begin{tabular}{ccccc}
\toprule[1.2pt]
\textbf{Noise std} & \textbf{\#Heads} & \textbf{$\UR$ Pre} & \textbf{$\UR$ Post} & \textbf{$\Delta\UR$} \\
\midrule[0.9pt]
24 & 6 & 1.82  & 35.45 & 33.64 \\
36 & 6 & 3.18  & 36.14 & 32.95 \\
48 & 6 & 7.05  & 40.45 & 33.41 \\
64 & 6 & 11.59 & 43.64 & 32.05 \\
\bottomrule[1.2pt]
\end{tabular}}
\end{table}

\noindent\textbf{Analysis.}
\Cref{tab:noise_robustness} shows that the brake head intervention remains effective when diagnostic evidence is locally corrupted instead of removed by masking.
As noise strength increases, the pre intervention $\UR$ rises from 1.82 to 11.59, indicating that stronger local corruption makes the visual evidence less reliable.
After excising the same six brake heads, $\UR$ increases substantially at all noise levels, with $\Delta\UR$ remaining \textit{stable around 32 to 34 points}.
Clean image accuracy is unchanged across the sweep, suggesting that the effect is concentrated on locally corrupted inputs.
These results indicate that the selected brake heads are not specific to mask based evidence removal; they also modulate abstention when diagnostic regions remain visible but unreliable, supporting the interpretation that the intervention \textit{strengthens abstention under degraded visual evidence}.

\section{More Experiment Results}
\label{app:text_results}

\subsection{Head-Set Disjointness between Arbitration and Brake localization}
\label{app:head_overlap}

We examine whether the two failure modes are associated with the same intervention sensitive heads.
For each model, we compare the selected arbitration heads $\calC_{\mathrm{arb}}^*$ and brake heads $\calC_{\mathrm{brk}}^*$ on SLAKE, and report their set overlap in \cref{tab:head_overlap}.
Across both Hulu-Med 4B and InternVL3.5 4B, the two selected head sets have \textit{zero intersection}.
This provides additional evidence for the decoupling analysis in \cref{sec:decoupling}: arbitration failure and brake failure are localized to \textbf{distinct selected head populations} under our intervention based criterion.

\begin{table}[t]
\centering
\caption{
\textbf{Head set overlap between arbitration and brake localization on SLAKE.}
Both Jaccard and overlap coefficient are zero, confirming disjoint head sets.
}
\label{tab:head_overlap}
\scriptsize
\setlength{\tabcolsep}{6pt}
\resizebox{1.0\textwidth}{!}{
\begin{tabular}{lcccccc}
\toprule[1.2pt]
\textbf{Model} & \textbf{Brake Heads} & \textbf{Arbitration Heads} & \textbf{Intersection} & \textbf{Union} & \textbf{Jaccard} & \textbf{Overlap Coef.} \\
\midrule[0.9pt]
Hulu-Med 4B & 6 & 10 & 0 & 16 & 0.000 & 0.000 \\
InternVL3.5 4B & 7 & 8 & 0 & 15 & 0.000 & 0.000 \\
\bottomrule[1.2pt]
\end{tabular}}
\end{table}

\subsection{Arbitration Head Distribution in InternVL3.5 and Qwen3-VL}
\label{app:arb_intern_qwen}

\begin{figure*}[t]
    \centering
    \includegraphics[width=0.95\textwidth]{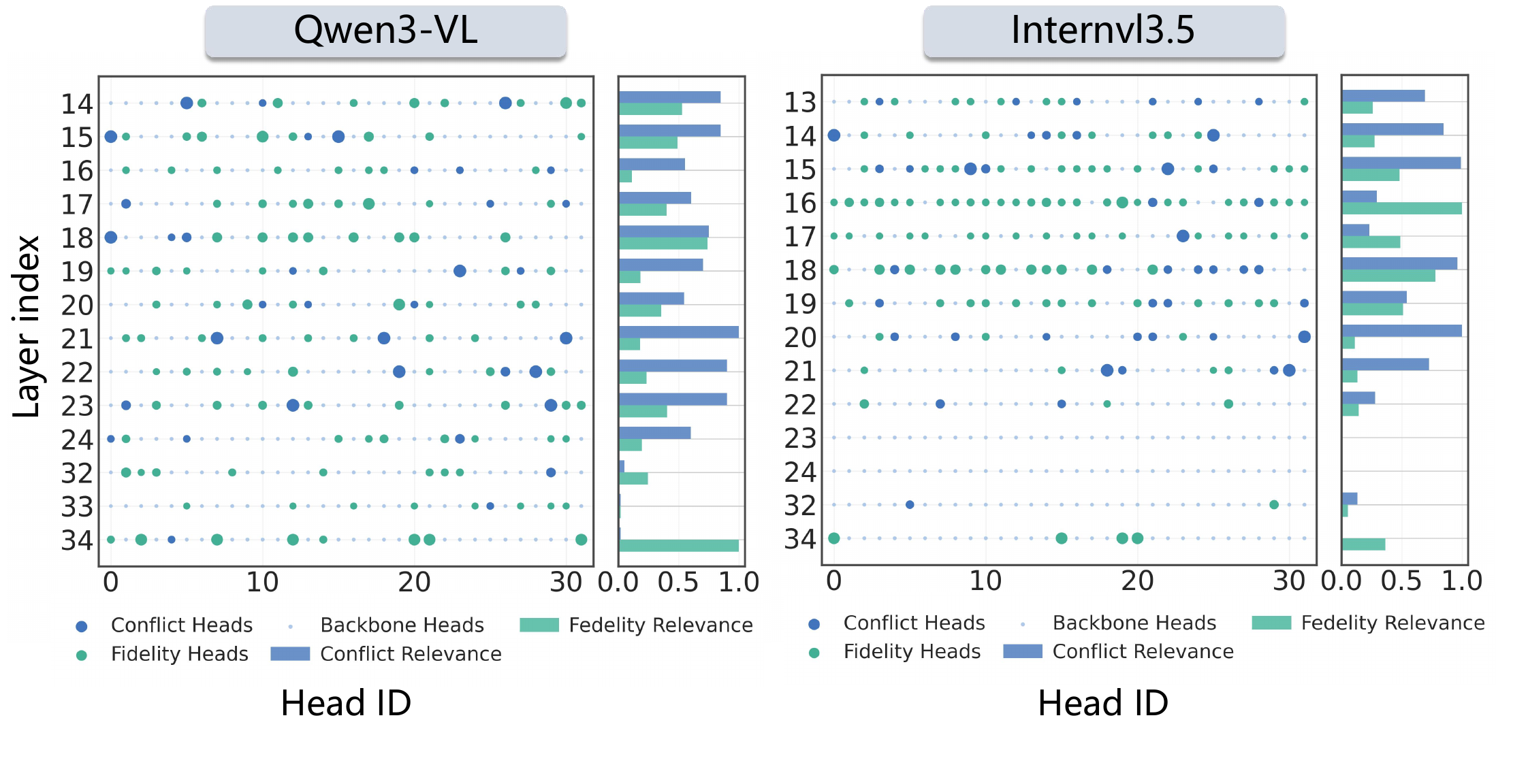}
    \caption{
    \textbf{Arbitration head distribution in InternVL3.5 and Qwen3-VL.}
Blue circles denote selected arbitration heads, green circles denote selected backbone heads, and pale dots denote unselected heads.
Circle size is proportional to the corresponding score.
Both models exhibit a broad arbitration band in mid to late layers, while backbone heads are more diffuse with a late-layer tail.
    }
    \label{fig:arb_head_grid_intern_qwen}
\end{figure*}

As shown in \cref{fig:arb_head_grid_intern_qwen}, the same dual threshold selection procedure yields a similar spatial pattern in InternVL3.5 4B and Qwen3-VL-8B.
For InternVL3.5 4B, the selected arbitration heads are concentrated in a broad mid layer band spanning roughly layers 13 to 22, with the strongest layer level $\CER$ around layer 21.
The selected backbone heads are more diffuse, but retain a visible late layer tail at layers 32 and 34.
Relative to Hulu-Med (\cref{fig:arb_head_grid}), the arbitration band in InternVL3.5 4B appears slightly deeper overall.

Qwen3-VL-8B shows the same qualitative organization.
Its arbitration heads again form a contiguous band, here extending roughly from layers 14 to 24, with the strongest $\CER$ around layer 22.
Backbone heads are distributed more broadly across the same range and also extend into the late layers, particularly layer 34.
Thus, although the exact peak layer varies by architecture, the overall separation remains consistent: heads with high conflict sensitivity occupy a wide override band, whereas clean preserving backbone heads are more dispersed and include a late layer component.

Taken together with Hulu-Med, these results suggest that the arbitration circuit is \textbf{not confined to a single architecture}.
Across all three backbones, localization recovers a broad mid layer override region rather than isolated single head hotspots, while backbone heads remain comparatively distributed.
This cross model regularity supports the view that the wideband override pattern reflects a shared functional organization of arbitration failure.

\subsection{Cross-Dataset and Cross-Model Transfer}
\label{app:transfer}

We evaluate whether the selected intervention targets transfer beyond the dataset or model used for localization.
For arbitration failure, transfer is measured by the change in resist rate, $\Delta\mathrm{Resist}$.
For brake failure, transfer is measured by the change in unknown rate, $\Delta\UR$.
Arbitration heads transfer consistently across datasets and model backbones, whereas brake head transfer is weaker and less directional.
This asymmetry suggests that arbitration heads capture a more \textit{reusable textual override mechanism}, while brake heads are more sensitive to the target model and visual degradation regime.

\begin{table*}[h]
\centering
\caption{
\textbf{Cross dataset transfer results.}
Arbitration transfer reports $\Delta\mathrm{Resist}$; brake transfer reports $\Delta\UR$.
}
\label{tab:cross_dataset_transfer}
\scriptsize
\setlength{\tabcolsep}{5pt}
\resizebox{0.95\textwidth}{!}{
\begin{tabular}{lllccc}
\toprule[1.2pt]
\textbf{Model} & \textbf{Failure Mode} & \textbf{Transfer} & \textbf{\#Heads} & \textbf{$\Delta\mathrm{Resist}$} & \textbf{$\Delta\UR$} \\
\midrule[0.9pt]

\multirow{4}{*}{Hulu-Med 4B}
& Arbitration & SLAKE heads $\rightarrow$ VQA-RAD & 10 & 30.77 & -- \\
& Arbitration & VQA-RAD heads $\rightarrow$ SLAKE & 6 & 12.44 & -- \\
& Brake & HealMed heads $\rightarrow$ SLAKE & 5 & -- & 4.55 \\
& Brake & SLAKE heads $\rightarrow$ HealMed & 6 & -- & 21.56 \\

\midrule[0.9pt]

\multirow{4}{*}{InternVL3.5 4B}
& Arbitration & SLAKE heads $\rightarrow$ VQA-RAD & 8 & 13.71 & -- \\
& Arbitration & VQA-RAD heads $\rightarrow$ SLAKE & 10 & 1.14 & -- \\
& Brake & HealMed heads $\rightarrow$ SLAKE & 7 & -- & 1.29 \\
& Brake & SLAKE heads $\rightarrow$ HealMed & 7 & -- & 3.68 \\

\bottomrule[1.2pt]
\end{tabular}}
\end{table*}

\begin{table}[h]
\centering
\caption{
\textbf{Cross model transfer of arbitration head sets.}
Transfer is reported as $\Delta\mathrm{Resist}$.
}
\label{tab:cross_model_transfer_arb}
\scriptsize
\setlength{\tabcolsep}{6pt}
\resizebox{0.65\textwidth}{!}{
\begin{tabular}{llcc}
\toprule[1.2pt]
\textbf{Benchmark} & \textbf{Transfer} & \textbf{\#Heads} & \textbf{$\Delta\mathrm{Resist}$} \\
\midrule[0.9pt]
VQA-RAD & InternVL $\rightarrow$ Hulu-Med & 10 & 17.95 \\
VQA-RAD & Hulu-Med $\rightarrow$ InternVL & 6 & 12.00 \\
SLAKE VQA & InternVL $\rightarrow$ Hulu-Med & 8 & 5.07 \\
SLAKE VQA & Hulu-Med $\rightarrow$ InternVL & 10 & 2.76 \\
\bottomrule[1.2pt]
\end{tabular}}
\end{table}

\begin{table}[h]
\centering
\caption{
\textbf{Cross model transfer of brake head sets.}
Transfer is reported as $\Delta\UR$.
}
\label{tab:cross_model_transfer_brake}
\scriptsize
\setlength{\tabcolsep}{7pt}
\resizebox{0.64\textwidth}{!}{
\begin{tabular}{llcc}
\toprule[1.2pt]
\textbf{Benchmark} & \textbf{Transfer} & \textbf{\#Heads} & \textbf{$\Delta\UR$} \\
\midrule[0.9pt]
SLAKE VQA & Hulu-Med $\rightarrow$ InternVL & 6 & 0.00 \\
SLAKE VQA & InternVL $\rightarrow$ Hulu-Med & 7 & -3.03 \\
HealMed-VQA & Hulu-Med $\rightarrow$ InternVL & 6 & 0.42 \\
HealMed-VQA & InternVL $\rightarrow$ Hulu-Med & 7 & -2.76 \\
\bottomrule[1.2pt]
\end{tabular}}
\end{table}

\noindent\textbf{Analysis.}
The transfer results reveal a consistent asymmetry between the two intervention targets.
Arbitration head transfer remains \textit{positive across dataset and model shifts}.
In cross dataset transfer, Hulu-Med 4B obtains $\Delta\mathrm{Resist}$ gains of 30.77 and 12.44 points, while InternVL3.5 4B also shows positive transfer with gains of 13.71 and 1.14 points.
Together with the results, where transferred arbitration head sets improve $\mathrm{Resist}$ in all evaluated directions, this suggests that arbitration heads capture a \textit{reusable textual override mechanism} rather than a dataset specific localization artefact.

Brake head transfer is weaker and less directional.
Although cross dataset transfer can increase $\UR$, the magnitude varies substantially across target datasets and model backbones.
Moreover, cross model brake transfer does not consistently improve abstention: on SLAKE VQA, Hulu-Med heads leave $\UR$ unchanged on InternVL, whereas InternVL heads reduce $\UR$ on Hulu-Med.
These results indicate that brake heads are more tightly coupled to the target model and the visual degradation regime.
We therefore interpret arbitration heads as comparatively transferable intervention targets, while treating brake heads as \textit{setting specific abstention control heads}.

\subsection{Clean-Competence-Stratified Head localization}
\label{app:competence_stratified_heads}

\begin{wrapfigure}{r}{0.52\textwidth}
\vspace{-15pt}
\centering
\includegraphics[width=1\linewidth]{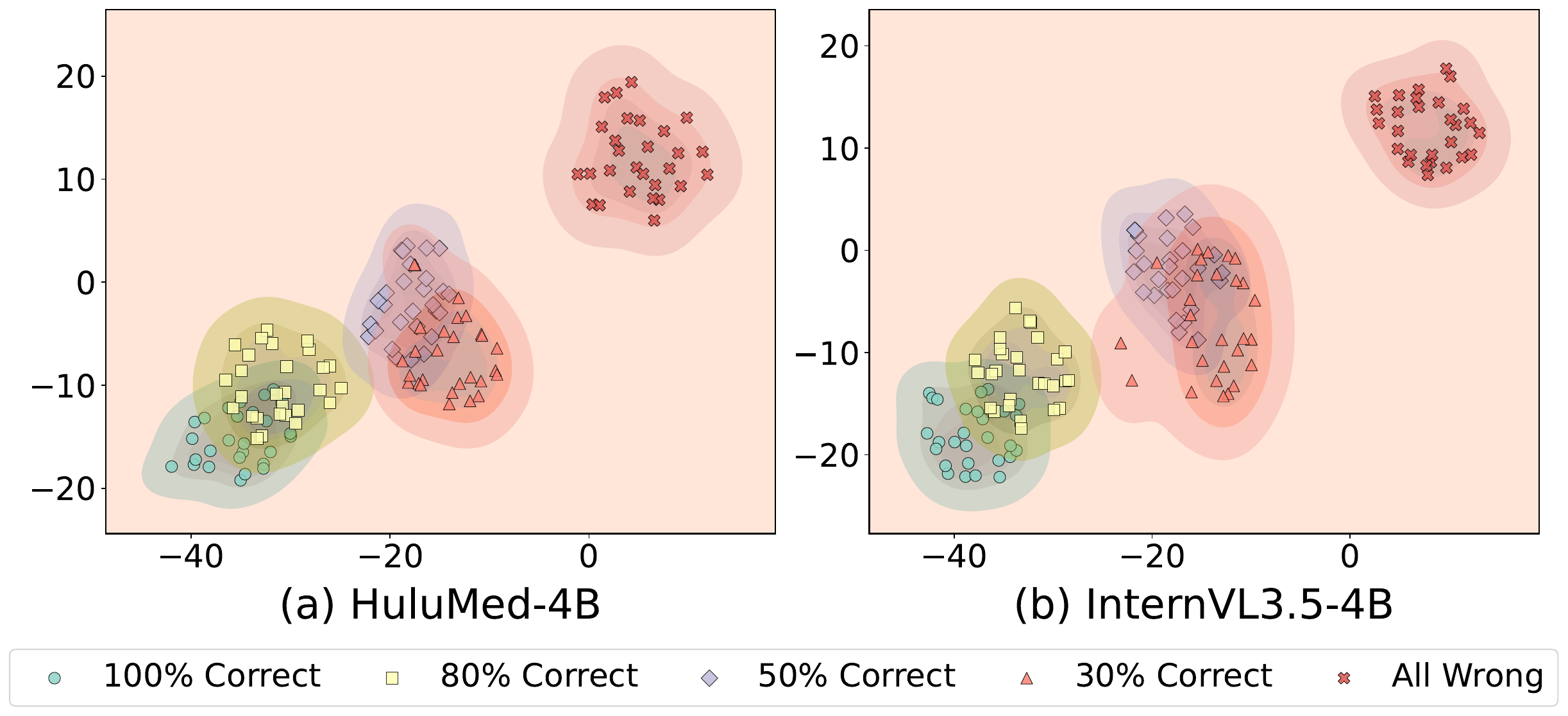}
\caption{
\textbf{Bootstrap visualisation of clean competence stratified head sets.}
Each point denotes a bootstrap sample of the head set under a given stratum.
Shaded regions show kernel density estimates.
The 100\% and 80\% strata remain nearby, whereas 50\% and 30\% strata shift toward an intermediate region and the all wrong stratum is clearly separated.
}
\label{fig:competence_stratified_heads}
\vspace{-15pt}
\end{wrapfigure}
Our main experiments localize arbitration heads on a clean correct subset, where the model answers correctly before conflict injection.
This filtering controls a key confounder: if the model does not know the answer under the clean condition, conflict following cannot be cleanly separated from ordinary knowledge failure.
The resulting intervention is therefore intended to isolate conflict arbitration rather than general answer instability.

We further examine whether the identified head set is specific to the fully clean correct subset.
To this end, we repeat head localization under different clean competence strata, constructed to contain approximately 100\%, 80\%, 50\%, 30\%, and 0\% clean correct samples.
The 100\% stratum corresponds to the main experimental protocol.
Lower competence strata progressively include samples on which the model is less reliable before conflict is introduced.
This analysis is used only as a sensitivity check: when clean competence is relaxed, selected heads may reflect a mixture of conflict arbitration, sample difficulty, uncertainty handling, and wrong answer commitment.

\begin{table*}[t]
\centering
\caption{
\textbf{Clean competence stratified head set summary.}
The 100\% stratum is the main setting. Lower strata are sensitivity checks only.
}
\label{tab:competence_stratified_summary}
\scriptsize
\setlength{\tabcolsep}{4pt}
\resizebox{\textwidth}{!}{
\begin{tabular}{llcccl}
\toprule[1.2pt]
\textbf{Model} & \textbf{Stratum} & \textbf{\#Heads} & \textbf{Dominant layers} & \textbf{Overlap with C100} & \textbf{Interpretation} \\
\midrule[0.9pt]
\multirow{5}{*}{Hulu-Med 4B}
& \cellcolor{bestrow}\textbf{100\% clean correct} & \cellcolor{bestrow}\textbf{14} & \cellcolor{bestrow}\textbf{8, 12--15, 18--21} & \cellcolor{bestrow}\textbf{100\%} & \cellcolor{bestrow}Main clean conflict arbitration set. \\
& 80\% clean correct & 15 & 12--22, mainly 18--21 & 65--75\% & Largely preserves the C100 core with neighbouring heads. \\
& 50\% clean correct & 21 & 14--27 plus 32--35 & 10--25\% & Less localized; recruits difficulty sensitive heads. \\
& 30\% clean correct & 24 & 15--27 plus 32--35 & 5--20\% & Shifts toward late answer commitment heads. \\
& 0\% / all wrong & 20 & diffuse across 0--35 & $<$10\% & No stable conflict core; resembles wrong prior heads. \\
\midrule[0.9pt]
\multirow{5}{*}{InternVL3.5 4B}
& \cellcolor{bestrow}\textbf{100\% clean correct} & \cellcolor{bestrow}\textbf{21} & \cellcolor{bestrow}\textbf{13--21, 24, 34} & \cellcolor{bestrow}\textbf{100\%} & \cellcolor{bestrow}Main clean conflict arbitration set. \\
& 80\% clean correct & 21 & 13--24 plus 34 & 65--80\% & Close to C100, with mid layer additions. \\
& 50\% clean correct & 21 & 15--28 plus 32--34 & 10--25\% & Shifts toward difficulty and commitment heads. \\
& 30\% clean correct & 25 & 15--28 plus 32--35 & 5--20\% & Similar to 50\% with stronger late layer participation. \\
& 0\% / all wrong & 19 & diffuse across 0--35 & $<$10\% & No stable clean conflict core. \\
\bottomrule[1.2pt]
\end{tabular}}
\end{table*}

\noindent\textbf{Analysis.}
\Cref{fig:competence_stratified_heads} and \cref{tab:competence_stratified_summary} show a consistent transition as the clean correctness requirement is relaxed.
At 80\% clean correctness, the selected head sets remain close to the C100 stratum in the bootstrap embedding and retain substantial overlap with the main head set.
This suggests that the main arbitration heads are \textit{not an artefact} of using an overly narrow fully clean correct subset.

The pattern changes at 50\% and 30\% clean correctness.
For both models, the selected sets move away from the C100 region and become broader, with more participation from middle to late layers.
This indicates that localization begins to capture factors beyond conflict arbitration, including sample difficulty, unstable clean predictions, uncertainty suppression, and answer commitment behavior.
The all wrong stratum is separated from the clean correct strata and has little overlap with C100.
In this regime, the model already fails before conflict injection, so the selected heads cannot be interpreted as clean conflict arbitration heads.

These results \textit{support the clean correct filtering strategy} used in the main experiments.
The C100 setting yields a compact and stable arbitration related head set, whereas relaxing the clean competence constraint changes the selected head distribution and weakens the causal interpretation.

\subsection{Full Results: ConflictMedQA}
\label{app:full_conflictmedqa}

\Cref{tab:full_conflictmedqa_all} provides the full train/validation results on ConflictMedQA across three conflict injection positions.
Across Qwen3-4B and Llama3.2-3B, \method consistently lowers $\CFR$ and increases the resist rate on the validation split, indicating that the selected arbitration heads generalize beyond the localization set.
Compared with random ablation, \method produces larger $\mathrm{W2C}$ transitions while keeping $\mathrm{C2W}$ low, suggesting that the intervention mainly corrects conflict induced errors with limited damage to clean predictions.
The Before A setting yields the largest improvements, consistent with the injected evidence being closest to the answer position and therefore exerting stronger influence on final answer selection.

\begin{table*}[!t]
\centering
\caption{
\textbf{Full train/val results on ConflictMedQA.}
Results across three injection positions. Shaded rows denote \method.
}
\label{tab:full_conflictmedqa_all}
\scriptsize
\setlength{\tabcolsep}{3pt}
\resizebox{0.9\textwidth}{!}{
\begin{tabular}{@{}l@{\hspace{8pt}}l@{\hspace{6pt}}lcccccccc@{}}
\toprule[1.2pt]
& & & \multicolumn{4}{c}{\textbf{Train}} & \multicolumn{4}{c}{\textbf{Val}} \\
\cmidrule(lr){4-7} \cmidrule(lr){8-11}
\textbf{Position} & \textbf{Model} & \textbf{Method}
& \textbf{CFR$\downarrow$} & \textbf{Resist$\uparrow$} & \textbf{C2W$\downarrow$} & \textbf{W2C$\uparrow$}
& \textbf{CFR$\downarrow$} & \textbf{Resist$\uparrow$} & \textbf{C2W$\downarrow$} & \textbf{W2C$\uparrow$} \\
\midrule[0.9pt]

\multirow{6}{*}{\textbf{Prefix}}
& \multirow{3}{*}{Qwen3-4B}
& Baseline & 81.60 & 18.40 & -- & -- & 83.33 & 16.67 & -- & -- \\
& & Random   & 91.04 & 8.96  & 11.79 & 2.36 & 94.25 & 5.75  & 12.07 & 1.15 \\
& & \cellcolor{bestrow}\method
& \cellcolor{bestrow}\textbf{33.49} & \cellcolor{bestrow}\textbf{66.51} & \cellcolor{bestrow}\textbf{0.47} & \cellcolor{bestrow}\textbf{48.58}
& \cellcolor{bestrow}\textbf{37.36} & \cellcolor{bestrow}\textbf{62.64} & \cellcolor{bestrow}\textbf{0.00} & \cellcolor{bestrow}\textbf{45.98} \\
\cmidrule(lr){2-11}
& \multirow{3}{*}{Llama3.2-3B}
& Baseline & 54.58 & 45.42 & -- & -- & 57.69 & 42.31 & -- & -- \\
& & Random   & 38.33 & 61.67 & 27.08 & \textbf{43.33} & 42.31 & 57.69 & 26.92 & \textbf{42.31} \\
& & \cellcolor{bestrow}\method
& \cellcolor{bestrow}\textbf{29.17} & \cellcolor{bestrow}\textbf{70.83} & \cellcolor{bestrow}\textbf{1.25} & \cellcolor{bestrow}26.67
& \cellcolor{bestrow}\textbf{27.88} & \cellcolor{bestrow}\textbf{72.12} & \cellcolor{bestrow}\textbf{0.00} & \cellcolor{bestrow}29.81 \\

\midrule[0.9pt]

\multirow{6}{*}{\textbf{Before Q}}
& \multirow{3}{*}{Qwen3-4B}
& Baseline & 89.15 & 10.85 & -- & -- & 94.83 & 5.17 & -- & -- \\
& & Random   & 84.43 & 15.57 & 1.89 & 6.60 & 89.66 & 10.34 & \textbf{1.15} & 6.32 \\
& & \cellcolor{bestrow}\method
& \cellcolor{bestrow}\textbf{76.42} & \cellcolor{bestrow}\textbf{23.58} & \cellcolor{bestrow}\textbf{0.94} & \cellcolor{bestrow}\textbf{13.68}
& \cellcolor{bestrow}\textbf{68.39} & \cellcolor{bestrow}\textbf{31.61} & \cellcolor{bestrow}\textbf{1.15} & \cellcolor{bestrow}\textbf{27.59} \\
\cmidrule(lr){2-11}
& \multirow{3}{*}{Llama3.2-3B}
& Baseline & 77.92 & 22.08 & -- & -- & 76.92 & 23.08 & -- & -- \\
& & Random   & 81.67 & 18.33 & 12.50 & 8.75 & 81.73 & 18.27 & 11.54 & 6.73 \\
& & \cellcolor{bestrow}\method
& \cellcolor{bestrow}\textbf{54.58} & \cellcolor{bestrow}\textbf{45.42} & \cellcolor{bestrow}\textbf{0.42} & \cellcolor{bestrow}\textbf{23.75}
& \cellcolor{bestrow}\textbf{57.69} & \cellcolor{bestrow}\textbf{42.31} & \cellcolor{bestrow}\textbf{0.00} & \cellcolor{bestrow}\textbf{19.23} \\

\midrule[0.9pt]

\multirow{6}{*}{\textbf{Before A}}
& \multirow{3}{*}{Qwen3-4B}
& Baseline & 75.47 & 24.53 & -- & -- & 78.16 & 21.84 & -- & -- \\
& & Random   & 73.11 & 26.89 & 5.66 & 8.02 & 76.44 & 23.56 & 5.75 & 7.47 \\
& & \cellcolor{bestrow}\method
& \cellcolor{bestrow}\textbf{25.00} & \cellcolor{bestrow}\textbf{75.00} & \cellcolor{bestrow}\textbf{0.47} & \cellcolor{bestrow}\textbf{50.94}
& \cellcolor{bestrow}\textbf{23.56} & \cellcolor{bestrow}\textbf{76.44} & \cellcolor{bestrow}\textbf{0.57} & \cellcolor{bestrow}\textbf{55.17} \\
\cmidrule(lr){2-11}
& \multirow{3}{*}{Llama3.2-3B}
& Baseline & 35.00 & 65.00 & -- & -- & 43.27 & 56.73 & -- & -- \\
& & Random   & 49.58 & 50.42 & 22.08 & 7.50 & 48.08 & 51.92 & 16.35 & 11.54 \\
& & \cellcolor{bestrow}\method
& \cellcolor{bestrow}\textbf{22.50} & \cellcolor{bestrow}\textbf{77.50} & \cellcolor{bestrow}\textbf{2.08} & \cellcolor{bestrow}\textbf{14.58}
& \cellcolor{bestrow}\textbf{31.73} & \cellcolor{bestrow}\textbf{68.27} & \cellcolor{bestrow}\textbf{2.88} & \cellcolor{bestrow}\textbf{14.42} \\

\bottomrule[1.2pt]
\end{tabular}
}
\vspace{-4pt}
\end{table*}

\subsection{Full Results: PubMedQA}
\label{app:full_pubmedqa}

\Cref{tab:full_pubmedqa_all} reports the full train/validation results on PubMedQA across three conflict injection positions.
Across both Qwen3-4B and Llama3.2-3B, \method consistently reduces $\CFR$ and improves the resist rate on the validation split, showing that the selected arbitration heads generalize beyond the localization set.
The gains are especially strong for Qwen3-4B under the Before Q setting, where validation $\CFR$ drops from 53.70 to 3.70 and the resist rate increases from 46.30 to 96.30.
In contrast, random ablation is unstable: although it can reduce $\CFR$ in some settings, it can also increase conflict following, as seen in the Before A position for Qwen3-4B.
\method also keeps $\mathrm{C2W}$ close to zero in most settings while producing substantially larger $\mathrm{W2C}$ transitions, indicating that the intervention primarily converts conflict induced wrong predictions back to correct ones with only limited collateral damage to otherwise clean predictions.

\begin{table*}[!t]
\centering
\caption{
\textbf{Full train/val results on PubMedQA.}
Results across three injection positions. Shaded rows denote \method.
}
\label{tab:full_pubmedqa_all}
\scriptsize
\setlength{\tabcolsep}{3pt}
\resizebox{0.9\textwidth}{!}{
\begin{tabular}{@{}l@{\hspace{8pt}}l@{\hspace{6pt}}lcccccccc@{}}
\toprule[1.2pt]
& & & \multicolumn{4}{c}{\textbf{Train}} & \multicolumn{4}{c}{\textbf{Val}} \\
\cmidrule(lr){4-7} \cmidrule(lr){8-11}
\textbf{Position} & \textbf{Model} & \textbf{Method}
& \textbf{CFR$\downarrow$} & \textbf{Resist$\uparrow$} & \textbf{C2W$\downarrow$} & \textbf{W2C$\uparrow$}
& \textbf{CFR$\downarrow$} & \textbf{Resist$\uparrow$} & \textbf{C2W$\downarrow$} & \textbf{W2C$\uparrow$} \\
\midrule[0.9pt]

\multirow{6}{*}{\textbf{Prefix}}
& \multirow{3}{*}{Qwen3-4B}
& Baseline & 48.40 & 51.60 & -- & -- & 49.67 & 50.33 & -- & -- \\
& & Random   & 22.40 & 77.60 & 0.80 & \textbf{26.80} & 57.17 & 42.83 & 12.17 & 4.67 \\
& & \cellcolor{bestrow}\method
& \cellcolor{bestrow}\textbf{22.00} & \cellcolor{bestrow}\textbf{78.00} & \cellcolor{bestrow}\textbf{0.40} & \cellcolor{bestrow}\textbf{26.80}
& \cellcolor{bestrow}\textbf{23.59} & \cellcolor{bestrow}\textbf{76.41} & \cellcolor{bestrow}\textbf{0.00} & \cellcolor{bestrow}\textbf{26.09} \\
\cmidrule(lr){2-11}
& \multirow{3}{*}{Llama3.2-3B}
& Baseline & 56.80 & 43.20 & -- & -- & 58.46 & 41.54 & -- & -- \\
& & Random   & 60.40 & 39.60 & 7.60 & 4.00 & 63.08 & 36.92 & 8.85 & 4.23 \\
& & \cellcolor{bestrow}\method
& \cellcolor{bestrow}\textbf{44.80} & \cellcolor{bestrow}\textbf{55.20} & \cellcolor{bestrow}\textbf{1.20} & \cellcolor{bestrow}\textbf{13.20}
& \cellcolor{bestrow}\textbf{46.15} & \cellcolor{bestrow}\textbf{53.85} & \cellcolor{bestrow}\textbf{1.15} & \cellcolor{bestrow}\textbf{13.46} \\

\midrule[0.9pt]

\multirow{6}{*}{\textbf{Before Q}}
& \multirow{3}{*}{Qwen3-4B}
& Baseline & 53.60 & 46.40 & -- & -- & 53.70 & 46.30 & -- & -- \\
& & Random   & 19.80 & 80.20 & 0.60 & 34.40 & 46.30 & 53.70 & 3.91 & 11.30 \\
& & \cellcolor{bestrow}\method
& \cellcolor{bestrow}\textbf{6.20} & \cellcolor{bestrow}\textbf{93.80} & \cellcolor{bestrow}\textbf{0.00} & \cellcolor{bestrow}\textbf{47.40}
& \cellcolor{bestrow}\textbf{3.70} & \cellcolor{bestrow}\textbf{96.30} & \cellcolor{bestrow}\textbf{0.00} & \cellcolor{bestrow}\textbf{50.00} \\
\cmidrule(lr){2-11}
& \multirow{3}{*}{Llama3.2-3B}
& Baseline & 69.20 & 30.80 & -- & -- & 71.54 & 28.46 & -- & -- \\
& & Random   & 65.60 & 34.40 & 4.80 & 8.40 & 68.08 & 31.92 & 4.62 & 8.08 \\
& & \cellcolor{bestrow}\method
& \cellcolor{bestrow}\textbf{57.20} & \cellcolor{bestrow}\textbf{42.80} & \cellcolor{bestrow}\textbf{1.60} & \cellcolor{bestrow}\textbf{13.60}
& \cellcolor{bestrow}\textbf{59.23} & \cellcolor{bestrow}\textbf{40.77} & \cellcolor{bestrow}\textbf{1.54} & \cellcolor{bestrow}\textbf{13.85} \\

\midrule[0.9pt]

\multirow{6}{*}{\textbf{Before A}}
& \multirow{3}{*}{Qwen3-4B}
& Baseline & 61.80 & 38.20 & -- & -- & 65.98 & 34.02 & -- & -- \\
& & Random   & 92.80 & 7.20  & 31.00 & 0.00 & 85.33 & 14.67 & 20.54 & 1.20 \\
& & \cellcolor{bestrow}\method
& \cellcolor{bestrow}\textbf{51.00} & \cellcolor{bestrow}\textbf{49.00} & \cellcolor{bestrow}\textbf{0.00} & \cellcolor{bestrow}\textbf{10.80}
& \cellcolor{bestrow}\textbf{51.41} & \cellcolor{bestrow}\textbf{48.59} & \cellcolor{bestrow}\textbf{0.00} & \cellcolor{bestrow}\textbf{14.57} \\
\cmidrule(lr){2-11}
& \multirow{3}{*}{Llama3.2-3B}
& Baseline & 64.80 & 35.20 & -- & -- & 66.92 & 33.08 & -- & -- \\
& & Random   & 70.00 & 30.00 & 8.40 & 3.20 & 72.31 & 27.69 & 9.23 & 3.85 \\
& & \cellcolor{bestrow}\method
& \cellcolor{bestrow}\textbf{52.80} & \cellcolor{bestrow}\textbf{47.20} & \cellcolor{bestrow}\textbf{2.40} & \cellcolor{bestrow}\textbf{14.40}
& \cellcolor{bestrow}\textbf{54.62} & \cellcolor{bestrow}\textbf{45.38} & \cellcolor{bestrow}\textbf{2.31} & \cellcolor{bestrow}\textbf{14.62} \\

\bottomrule[1.2pt]
\end{tabular}
}
\vspace{-4pt}
\end{table*}

\subsection{Full Results on Multimodal Conflict Benchmarks}
\label{app:full_vqa}

\Cref{tab:full_vqarad_text_conflict_all,tab:full_SLAKE_text_conflict_all} report the full train/validation results for arbitration failure on VQA-RAD and SLAKE.
We evaluate three positions for injecting misleading textual evidence: Prefix, Before Q, and Before A.
Across both benchmarks, \method consistently lowers validation $\CFR$ and increases the resist rate relative to the baseline, indicating that the selected arbitration heads transfer beyond the subset used for localization.
The gains are also larger than those obtained by random head ablation in nearly all settings.
This is reflected by higher $\mathrm{W2C}$ and generally low $\mathrm{C2W}$, suggesting that the intervention primarily converts conflict following predictions back to the correct answer rather than broadly disrupting clean predictions.

The effect is most pronounced in the Before A setting, where the misleading evidence is placed closest to the answer token.
On VQA-RAD, \method reduces validation $\CFR$ from 98.21 to 70.26 for Hulu-Med 4B, from 88.00 to 67.50 for Hulu-Med 7B, from 60.57 to 19.43 for InternVL3.5 4B, and from 57.98 to 34.00 for Qwen3-VL-8B.
On SLAKE, the same intervention reduces validation $\CFR$ from 87.73 to 46.32 for Hulu-Med 4B, from 75.00 to 45.00 for Hulu-Med 7B, from 16.40 to 5.84 for InternVL3.5 4B, and from 45.20 to 23.20 for Qwen3-VL-8B.
Random ablation is substantially less reliable in this regime and can increase conflict following, especially for InternVL3.5 4B under Before A.
These results support the interpretation that arbitration heads are not merely correlated with conflict behavior, but form \textit{effective intervention targets} for reducing textual override.

\Cref{tab:full_image_conflict} reports the corresponding full results for brake failure under visual degradation.
Here the target behavior is abstention, measured by the unknown rate $\UR$.
Ablating the selected brake heads consistently increases held out $\UR$ on both SLAKE and HealMed-VQA, with positive $\mathrm{O2U}$ and near zero $\mathrm{U2O}$ in most settings.
On SLAKE, \method raises validation $\UR$ from 37.88 to 65.15 for Hulu-Med 4B, from 44.12 to 61.18 for Hulu-Med 7B, from 22.58 to 32.26 for InternVL3.5 4B, and from 42.42 to 55.30 for Qwen3-VL-8B.
On HealMed-VQA, the same trend is stronger: validation $\UR$ increases from 78.44 to 98.20 for Hulu-Med 4B, from 85.07 to 95.83 for Hulu-Med 7B, from 81.56 to 93.64 for InternVL3.5 4B, and from 84.72 to 95.14 for Qwen3-VL-8B.
Together, the arbitration and brake results show that the two interventions act on \textit{distinct failure behaviors}: arbitration head ablation suppresses textual override, while brake head ablation recovers abstention when visual evidence is insufficient.

\begin{table*}[!t]
\centering
\caption{
\textbf{Full train/val results on VQA-RAD textual conflict.}
Results across three injection positions. Shaded rows denote \method.
}
\label{tab:full_vqarad_text_conflict_all}
\scriptsize
\setlength{\tabcolsep}{3pt}
\resizebox{0.9\textwidth}{!}{
\begin{tabular}{@{}l@{\hspace{8pt}}l@{\hspace{6pt}}lcccccccc@{}}
\toprule[1.2pt]
& & & \multicolumn{4}{c}{\textbf{Train}} & \multicolumn{4}{c}{\textbf{Val}} \\
\cmidrule(lr){4-7} \cmidrule(lr){8-11}
\textbf{Position} & \textbf{Model} & \textbf{Method}
& \textbf{CFR$\downarrow$} & \textbf{Resist$\uparrow$} & \textbf{C2W$\downarrow$} & \textbf{W2C$\uparrow$}
& \textbf{CFR$\downarrow$} & \textbf{Resist$\uparrow$} & \textbf{C2W$\downarrow$} & \textbf{W2C$\uparrow$} \\
\midrule[0.9pt]

\multirow{12}{*}{\textbf{Prefix}}

& \multirow{3}{*}{Hulu-Med 4B}
& Baseline & 33.76 & 66.24 & -- & -- & 33.33 & 66.67 & -- & -- \\
& & Random   & 34.53 & 65.47 & 2.81 & 2.05 & 33.59 & 66.41 & 1.79 & 1.54 \\
& & \cellcolor{bestrow}\method
& \cellcolor{bestrow}\textbf{16.11} & \cellcolor{bestrow}\textbf{83.89} & \cellcolor{bestrow}\textbf{0.77} & \cellcolor{bestrow}\textbf{18.41}
& \cellcolor{bestrow}\textbf{15.38} & \cellcolor{bestrow}\textbf{84.62} & \cellcolor{bestrow}\textbf{0.26} & \cellcolor{bestrow}\textbf{18.21} \\

\cmidrule(lr){2-11}
& \multirow{3}{*}{Hulu-Med 7B}
& Baseline & 26.40 & 73.60 & -- & -- & 25.80 & 74.20 & -- & -- \\
& & Random   & 26.10 & 73.90 & 0.90 & 1.20 & 25.40 & 74.60 & 0.90 & 1.30 \\
& & \cellcolor{bestrow}\method
& \cellcolor{bestrow}\textbf{14.80} & \cellcolor{bestrow}\textbf{85.20} & \cellcolor{bestrow}\textbf{0.10} & \cellcolor{bestrow}\textbf{11.70}
& \cellcolor{bestrow}\textbf{14.20} & \cellcolor{bestrow}\textbf{85.80} & \cellcolor{bestrow}\textbf{0.10} & \cellcolor{bestrow}\textbf{11.70} \\

\cmidrule(lr){2-11}
& \multirow{3}{*}{InternVL3.5 4B}
& Baseline & 39.83 & 60.17 & -- & -- & 30.29 & 69.71 & -- & -- \\
& & Random   & 47.71 & 52.29 & 9.60 & 1.72 & 34.29 & 65.71 & 4.57 & 0.57 \\
& & \cellcolor{bestrow}\method
& \cellcolor{bestrow}\textbf{16.33} & \cellcolor{bestrow}\textbf{83.67} & \cellcolor{bestrow}\textbf{0.43} & \cellcolor{bestrow}\textbf{23.93}
& \cellcolor{bestrow}\textbf{9.14} & \cellcolor{bestrow}\textbf{90.86} & \cellcolor{bestrow}\textbf{0.57} & \cellcolor{bestrow}\textbf{21.71} \\

\cmidrule(lr){2-11}
& \multirow{3}{*}{Qwen3-VL-8B}
& Baseline & 29.10 & 70.90 & -- & -- & 27.84 & 72.16 & -- & -- \\
& & Random   & 29.40 & 70.60 & 0.90 & 0.60 & 28.10 & 71.90 & 0.90 & 0.64 \\
& & \cellcolor{bestrow}\method
& \cellcolor{bestrow}\textbf{12.80} & \cellcolor{bestrow}\textbf{87.20} & \cellcolor{bestrow}\textbf{0.20} & \cellcolor{bestrow}\textbf{16.50}
& \cellcolor{bestrow}\textbf{11.90} & \cellcolor{bestrow}\textbf{88.10} & \cellcolor{bestrow}\textbf{0.20} & \cellcolor{bestrow}\textbf{16.14} \\

\midrule[0.9pt]
\multirow{12}{*}{\textbf{Before Q}}

& \multirow{3}{*}{Hulu-Med 4B}
& Baseline & 47.57 & 52.43 & -- & -- & 44.62 & 55.38 & -- & -- \\
& & Random   & 44.25 & 55.75 & 2.05 & 5.37 & 40.77 & 59.23 & 1.03 & 4.87 \\
& & \cellcolor{bestrow}\method
& \cellcolor{bestrow}\textbf{15.09} & \cellcolor{bestrow}\textbf{84.91} & \cellcolor{bestrow}\textbf{0.77} & \cellcolor{bestrow}\textbf{33.25}
& \cellcolor{bestrow}\textbf{14.36} & \cellcolor{bestrow}\textbf{85.64} & \cellcolor{bestrow}\textbf{0.51} & \cellcolor{bestrow}\textbf{30.77} \\

\cmidrule(lr){2-11}
& \multirow{3}{*}{Hulu-Med 7B}
& Baseline & 37.20 & 62.80 & -- & -- & 36.50 & 63.50 & -- & -- \\
& & Random   & 36.00 & 64.00 & 0.80 & 2.00 & 35.20 & 64.80 & 0.80 & 2.10 \\
& & \cellcolor{bestrow}\method
& \cellcolor{bestrow}\textbf{15.60} & \cellcolor{bestrow}\textbf{84.40} & \cellcolor{bestrow}\textbf{0.20} & \cellcolor{bestrow}\textbf{21.80}
& \cellcolor{bestrow}\textbf{14.70} & \cellcolor{bestrow}\textbf{85.30} & \cellcolor{bestrow}\textbf{0.20} & \cellcolor{bestrow}\textbf{22.00} \\

\cmidrule(lr){2-11}
& \multirow{3}{*}{InternVL3.5 4B}
& Baseline & 29.94 & 70.06 & -- & -- & 21.71 & 78.29 & -- & -- \\
& & Random   & 37.25 & 62.75 & 9.46 & 2.15 & 24.00 & 76.00 & 4.00 & 1.71 \\
& & \cellcolor{bestrow}\method
& \cellcolor{bestrow}\textbf{14.33} & \cellcolor{bestrow}\textbf{85.67} & \cellcolor{bestrow}\textbf{3.72} & \cellcolor{bestrow}\textbf{19.34}
& \cellcolor{bestrow}\textbf{7.43} & \cellcolor{bestrow}\textbf{92.57} & \cellcolor{bestrow}\textbf{2.29} & \cellcolor{bestrow}\textbf{16.57} \\

\cmidrule(lr){2-11}
& \multirow{3}{*}{Qwen3-VL-8B}
& Baseline & 17.05 & 82.95 & -- & -- & 16.16 & 83.84 & -- & -- \\
& & Random   & 17.90 & 82.10 & 1.00 & 0.15 & 17.00 & 83.00 & 1.00 & 0.16 \\
& & \cellcolor{bestrow}\method
& \cellcolor{bestrow}\textbf{8.40} & \cellcolor{bestrow}\textbf{91.60} & \cellcolor{bestrow}\textbf{0.15} & \cellcolor{bestrow}\textbf{8.80}
& \cellcolor{bestrow}\textbf{7.50} & \cellcolor{bestrow}\textbf{92.50} & \cellcolor{bestrow}\textbf{0.15} & \cellcolor{bestrow}\textbf{8.81} \\

\midrule[0.9pt]
\multirow{12}{*}{\textbf{Before A}}

& \multirow{3}{*}{Hulu-Med 4B}
& Baseline & 97.70 & 2.30 & -- & -- & 98.21 & 1.79 & -- & -- \\
& & Random   & 97.19 & 2.81 & 0.26 & 0.77 & 98.72 & 1.28 & 0.51 & 0.00 \\
& & \cellcolor{bestrow}\method
& \cellcolor{bestrow}\textbf{74.42} & \cellcolor{bestrow}\textbf{25.58} & \cellcolor{bestrow}\textbf{0.00} & \cellcolor{bestrow}\textbf{23.27}
& \cellcolor{bestrow}\textbf{70.26} & \cellcolor{bestrow}\textbf{29.74} & \cellcolor{bestrow}\textbf{0.00} & \cellcolor{bestrow}\textbf{27.95} \\

\cmidrule(lr){2-11}
& \multirow{3}{*}{Hulu-Med 7B}
& Baseline & 89.20 & 10.80 & -- & -- & 88.00 & 12.00 & -- & -- \\
& & Random   & 89.50 & 10.50 & 0.35 & 0.05 & 88.30 & 11.70 & 0.35 & 0.05 \\
& & \cellcolor{bestrow}\method
& \cellcolor{bestrow}\textbf{68.20} & \cellcolor{bestrow}\textbf{31.80} & \cellcolor{bestrow}\textbf{0.00} & \cellcolor{bestrow}\textbf{21.00}
& \cellcolor{bestrow}\textbf{67.50} & \cellcolor{bestrow}\textbf{32.50} & \cellcolor{bestrow}\textbf{0.00} & \cellcolor{bestrow}\textbf{20.50} \\

\cmidrule(lr){2-11}
& \multirow{3}{*}{InternVL3.5 4B}
& Baseline & 64.90 & 35.10 & -- & -- & 60.57 & 39.43 & -- & -- \\
& & Random   & 80.95 & 19.05 & 16.33 & 0.29 & 77.71 & 22.29 & 17.71 & 0.57 \\
& & \cellcolor{bestrow}\method
& \cellcolor{bestrow}\textbf{29.51} & \cellcolor{bestrow}\textbf{70.49} & \cellcolor{bestrow}\textbf{0.86} & \cellcolor{bestrow}\textbf{36.25}
& \cellcolor{bestrow}\textbf{19.43} & \cellcolor{bestrow}\textbf{80.57} & \cellcolor{bestrow}\textbf{0.00} & \cellcolor{bestrow}\textbf{41.14} \\

\cmidrule(lr){2-11}
& \multirow{3}{*}{Qwen3-VL-8B}
& Baseline & 58.90 & 41.10 & -- & -- & 57.98 & 42.02 & -- & -- \\
& & Random   & 59.30 & 40.70 & 0.60 & 0.20 & 58.20 & 41.80 & 0.35 & 0.13 \\
& & \cellcolor{bestrow}\method
& \cellcolor{bestrow}\textbf{36.20} & \cellcolor{bestrow}\textbf{63.80} & \cellcolor{bestrow}\textbf{0.10} & \cellcolor{bestrow}\textbf{22.80}
& \cellcolor{bestrow}\textbf{34.00} & \cellcolor{bestrow}\textbf{66.00} & \cellcolor{bestrow}\textbf{0.10} & \cellcolor{bestrow}\textbf{24.08} \\

\bottomrule[1.2pt]
\end{tabular}
}
\vspace{-4pt}
\end{table*}

\begin{table*}[!t]
\centering
\caption{
\textbf{Full train/val results on SLAKE textual conflict.}
Results across three injection positions. Shaded rows denote \method.
}
\label{tab:full_SLAKE_text_conflict_all}
\scriptsize
\setlength{\tabcolsep}{3pt}
\resizebox{\textwidth}{!}{
\begin{tabular}{@{}l@{\hspace{8pt}}l@{\hspace{6pt}}lcccccccc@{}}
\toprule[1.2pt]
& & & \multicolumn{4}{c}{\textbf{Train}} & \multicolumn{4}{c}{\textbf{Val}} \\
\cmidrule(lr){4-7} \cmidrule(lr){8-11}
\textbf{Position} & \textbf{Model} & \textbf{Method}
& \textbf{CFR$\downarrow$} & \textbf{Resist$\uparrow$} & \textbf{C2W$\downarrow$} & \textbf{W2C$\uparrow$}
& \textbf{CFR$\downarrow$} & \textbf{Resist$\uparrow$} & \textbf{C2W$\downarrow$} & \textbf{W2C$\uparrow$} \\
\midrule[0.9pt]

\multirow{12}{*}{\textbf{Prefix}}

& \multirow{3}{*}{Hulu-Med 4B}
& Baseline & 18.17 & 81.83 & -- & -- & 19.15 & 80.85 & -- & -- \\
& & Random   & 21.32 & 78.68 & 3.93 & 0.70 & 21.28 & 78.72 & \textbf{2.95} & 0.82 \\
& & \cellcolor{bestrow}\method
& \cellcolor{bestrow}\textbf{11.16} & \cellcolor{bestrow}\textbf{88.84} & \cellcolor{bestrow}\textbf{3.86} & \cellcolor{bestrow}\textbf{10.87}
& \cellcolor{bestrow}\textbf{12.11} & \cellcolor{bestrow}\textbf{87.89} & \cellcolor{bestrow}5.07 & \cellcolor{bestrow}\textbf{12.11} \\

\cmidrule(lr){2-11}
& \multirow{3}{*}{Hulu-Med 7B}
& Baseline & 16.00 & 84.00 & -- & -- & 15.40 & 84.60 & -- & -- \\
& & Random   & 16.50 & 83.50 & \textbf{1.00} & 0.50 & 15.90 & 84.10 & \textbf{1.00} & 0.50 \\
& & \cellcolor{bestrow}\method
& \cellcolor{bestrow}\textbf{11.60} & \cellcolor{bestrow}\textbf{88.40} & \cellcolor{bestrow}1.50 & \cellcolor{bestrow}\textbf{5.90}
& \cellcolor{bestrow}\textbf{11.00} & \cellcolor{bestrow}\textbf{89.00} & \cellcolor{bestrow}1.50 & \cellcolor{bestrow}\textbf{5.90} \\

\cmidrule(lr){2-11}
& \multirow{3}{*}{InternVL3.5 4B}
& Baseline & 15.71 & 84.29 & -- & -- & 12.66 & 87.34 & -- & -- \\
& & Random   & 20.50 & 79.50 & 8.06 & 3.27 & 18.18 & 81.82 & 8.28 & 2.76 \\
& & \cellcolor{bestrow}\method
& \cellcolor{bestrow}\textbf{8.90} & \cellcolor{bestrow}\textbf{91.10} & \cellcolor{bestrow}\textbf{0.49} & \cellcolor{bestrow}\textbf{7.30}
& \cellcolor{bestrow}\textbf{6.66} & \cellcolor{bestrow}\textbf{93.34} & \cellcolor{bestrow}\textbf{0.65} & \cellcolor{bestrow}\textbf{6.66} \\

\cmidrule(lr){2-11}
& \multirow{3}{*}{Qwen3-VL-8B}
& Baseline & 16.20 & 83.80 & -- & -- & 15.60 & 84.40 & -- & -- \\
& & Random   & 16.90 & 83.10 & 0.95 & 0.25 & 16.20 & 83.80 & 0.85 & 0.25 \\
& & \cellcolor{bestrow}\method
& \cellcolor{bestrow}\textbf{11.60} & \cellcolor{bestrow}\textbf{88.40} & \cellcolor{bestrow}\textbf{0.30} & \cellcolor{bestrow}\textbf{4.90}
& \cellcolor{bestrow}\textbf{10.80} & \cellcolor{bestrow}\textbf{89.20} & \cellcolor{bestrow}\textbf{0.30} & \cellcolor{bestrow}\textbf{5.10} \\

\midrule[0.9pt]
\multirow{12}{*}{\textbf{Before Q}}

& \multirow{3}{*}{Hulu-Med 4B}
& Baseline & 18.51 & 81.49 & -- & -- & 19.15 & 80.85 & -- & -- \\
& & Random   & 20.62 & 79.38 & 3.86 & 1.75 & 20.79 & 79.21 & 4.09 & 2.45 \\
& & \cellcolor{bestrow}\method
& \cellcolor{bestrow}\textbf{8.20} & \cellcolor{bestrow}\textbf{91.80} & \cellcolor{bestrow}\textbf{0.98} & \cellcolor{bestrow}\textbf{11.29}
& \cellcolor{bestrow}\textbf{7.04} & \cellcolor{bestrow}\textbf{92.96} & \cellcolor{bestrow}\textbf{0.65} & \cellcolor{bestrow}\textbf{12.77} \\

\cmidrule(lr){2-11}
& \multirow{3}{*}{Hulu-Med 7B}
& Baseline & 14.80 & 85.20 & -- & -- & 14.20 & 85.80 & -- & -- \\
& & Random   & 15.30 & 84.70 & 0.80 & 0.30 & 14.80 & 85.20 & 0.90 & 0.30 \\
& & \cellcolor{bestrow}\method
& \cellcolor{bestrow}\textbf{7.20} & \cellcolor{bestrow}\textbf{92.80} & \cellcolor{bestrow}\textbf{0.40} & \cellcolor{bestrow}\textbf{8.00}
& \cellcolor{bestrow}\textbf{6.70} & \cellcolor{bestrow}\textbf{93.30} & \cellcolor{bestrow}\textbf{0.40} & \cellcolor{bestrow}\textbf{7.90} \\

\cmidrule(lr){2-11}
& \multirow{3}{*}{InternVL3.5 4B}
& Baseline & 10.91 & 89.09 & -- & -- & 8.93 & 91.07 & -- & -- \\
& & Random   & 19.81 & 80.19 & 10.35 & 1.46 & 18.51 & 81.49 & 10.71 & 1.14 \\
& & \cellcolor{bestrow}\method
& \cellcolor{bestrow}\textbf{7.51} & \cellcolor{bestrow}\textbf{92.49} & \cellcolor{bestrow}\textbf{0.63} & \cellcolor{bestrow}\textbf{4.03}
& \cellcolor{bestrow}\textbf{5.84} & \cellcolor{bestrow}\textbf{94.16} & \cellcolor{bestrow}\textbf{0.65} & \cellcolor{bestrow}\textbf{3.73} \\

\cmidrule(lr){2-11}
& \multirow{3}{*}{Qwen3-VL-8B}
& Baseline & 12.40 & 87.60 & -- & -- & 11.80 & 88.20 & -- & -- \\
& & Random   & 13.10 & 86.90 & 0.85 & 0.15 & 12.40 & 87.60 & 0.75 & 0.15 \\
& & \cellcolor{bestrow}\method
& \cellcolor{bestrow}\textbf{7.90} & \cellcolor{bestrow}\textbf{92.10} & \cellcolor{bestrow}\textbf{0.25} & \cellcolor{bestrow}\textbf{4.75}
& \cellcolor{bestrow}\textbf{7.20} & \cellcolor{bestrow}\textbf{92.80} & \cellcolor{bestrow}\textbf{0.25} & \cellcolor{bestrow}\textbf{4.85} \\

\midrule[0.9pt]
\multirow{12}{*}{\textbf{Before A}}

& \multirow{3}{*}{Hulu-Med 4B}
& Baseline & 87.17 & 12.83 & -- & -- & 87.73 & 12.27 & -- & -- \\
& & Random   & 92.08 & 7.92 & 5.75 & 0.84 & 91.00 & 9.00 & 4.26 & 0.98 \\
& & \cellcolor{bestrow}\method
& \cellcolor{bestrow}\textbf{47.62} & \cellcolor{bestrow}\textbf{52.38} & \cellcolor{bestrow}\textbf{5.33} & \cellcolor{bestrow}\textbf{44.88}
& \cellcolor{bestrow}\textbf{46.32} & \cellcolor{bestrow}\textbf{53.68} & \cellcolor{bestrow}4.91 & \cellcolor{bestrow}\textbf{46.32} \\

\cmidrule(lr){2-11}
& \multirow{3}{*}{Hulu-Med 7B}
& Baseline & 76.40 & 23.60 & -- & -- & 75.00 & 25.00 & -- & -- \\
& & Random   & 77.50 & 22.50 & 1.40 & 0.30 & 76.20 & 23.80 & 1.50 & 0.30 \\
& & \cellcolor{bestrow}\method
& \cellcolor{bestrow}\textbf{46.20} & \cellcolor{bestrow}\textbf{53.80} & \cellcolor{bestrow}\textbf{0.80} & \cellcolor{bestrow}\textbf{31.00}
& \cellcolor{bestrow}\textbf{45.00} & \cellcolor{bestrow}\textbf{55.00} & \cellcolor{bestrow}\textbf{1.00} & \cellcolor{bestrow}\textbf{31.00} \\

\cmidrule(lr){2-11}
& \multirow{3}{*}{InternVL3.5 4B}
& Baseline & 21.06 & 78.94 & -- & -- & 16.40 & 83.60 & -- & -- \\
& & Random   & 30.23 & 69.77 & 10.56 & 1.39 & 23.70 & 76.30 & 8.77 & 1.46 \\
& & \cellcolor{bestrow}\method
& \cellcolor{bestrow}\textbf{9.38} & \cellcolor{bestrow}\textbf{90.62} & \cellcolor{bestrow}\textbf{1.60} & \cellcolor{bestrow}\textbf{13.27}
& \cellcolor{bestrow}\textbf{5.84} & \cellcolor{bestrow}\textbf{94.16} & \cellcolor{bestrow}\textbf{2.27} & \cellcolor{bestrow}\textbf{12.82} \\

\cmidrule(lr){2-11}
& \multirow{3}{*}{Qwen3-VL-8B}
& Baseline & 46.50 & 53.50 & -- & -- & 45.20 & 54.80 & -- & -- \\
& & Random   & 48.20 & 51.80 & 2.05 & 0.35 & 47.00 & 53.00 & 2.10 & 0.30 \\
& & \cellcolor{bestrow}\method
& \cellcolor{bestrow}\textbf{24.60} & \cellcolor{bestrow}\textbf{75.40} & \cellcolor{bestrow}\textbf{0.30} & \cellcolor{bestrow}\textbf{22.20}
& \cellcolor{bestrow}\textbf{23.20} & \cellcolor{bestrow}\textbf{76.80} & \cellcolor{bestrow}\textbf{0.30} & \cellcolor{bestrow}\textbf{22.30} \\

\bottomrule[1.2pt]
\end{tabular}
}
\vspace{-4pt}
\end{table*}

\begin{table*}[!t]
\centering
\caption{
\textbf{Full train/val results for brake failure under visual degradation.}
\method ablates brake heads to recover abstention.
$\Delta\mathrm{UR}=\mathrm{O2U}-\mathrm{U2O}$. Shaded rows denote \method.
}
\label{tab:full_image_conflict}
\scriptsize
\setlength{\tabcolsep}{3pt}
\resizebox{0.8\textwidth}{!}{
\begin{tabular}{@{}ll lccc ccc@{}}
\toprule[1.2pt]
& & & \multicolumn{3}{c}{\textbf{Train}} & \multicolumn{3}{c}{\textbf{Val}} \\
\cmidrule(lr){4-6} \cmidrule(lr){7-9}
\textbf{Benchmark} & \textbf{Model} & \textbf{Method}
& UR$\uparrow$ & O2U$\uparrow$ & U2O$\downarrow$
& UR$\uparrow$ & O2U$\uparrow$ & U2O$\downarrow$ \\
\midrule[0.9pt]

\multirow{12}{*}{\textbf{SLAKE}}
& \multirow{3}{*}{Hulu-Med 4B}
& Baseline
& 36.90 & -- & --
& 37.88 & -- & -- \\
& & Random
& 44.23 & 8.10 & 0.77
& 45.45 & 8.33 & 0.76 \\
& & \cellcolor{bestrow}\method
& \cellcolor{bestrow}\textbf{64.20} & \cellcolor{bestrow}\textbf{27.30} & \cellcolor{bestrow}\textbf{0.00}
& \cellcolor{bestrow}\textbf{65.15} & \cellcolor{bestrow}\textbf{27.27} & \cellcolor{bestrow}\textbf{0.00} \\
\cmidrule(lr){2-9}

& \multirow{3}{*}{Hulu-Med 7B}
& Baseline
& 43.40 & -- & --
& 44.12 & -- & -- \\
& & Random
& 47.85 & 5.05 & 0.60
& 48.82 & 5.29 & 0.59 \\
& & \cellcolor{bestrow}\method
& \cellcolor{bestrow}\textbf{60.10} & \cellcolor{bestrow}\textbf{16.70} & \cellcolor{bestrow}\textbf{0.00}
& \cellcolor{bestrow}\textbf{61.18} & \cellcolor{bestrow}\textbf{17.06} & \cellcolor{bestrow}\textbf{0.00} \\
\cmidrule(lr){2-9}

& \multirow{3}{*}{InternVL3.5 4B}
& Baseline
& 23.29 & -- & --
& 22.58 & -- & -- \\
& & Random
& 20.55 & 1.37 & 4.11
& 16.13 & 0.00 & 6.45 \\
& & \cellcolor{bestrow}\method
& \cellcolor{bestrow}\textbf{26.03} & \cellcolor{bestrow}\textbf{2.74} & \cellcolor{bestrow}\textbf{0.00}
& \cellcolor{bestrow}\textbf{32.26} & \cellcolor{bestrow}\textbf{9.68} & \cellcolor{bestrow}\textbf{0.00} \\
\cmidrule(lr){2-9}

& \multirow{3}{*}{Qwen3-VL-8B}
& Baseline
& 41.80 & -- & --
& 42.42 & -- & -- \\
& & Random
& 47.10 & 6.05 & 0.75
& 47.73 & 6.06 & 0.75 \\
& & \cellcolor{bestrow}\method
& \cellcolor{bestrow}\textbf{54.60} & \cellcolor{bestrow}\textbf{12.80} & \cellcolor{bestrow}\textbf{0.00}
& \cellcolor{bestrow}\textbf{55.30} & \cellcolor{bestrow}\textbf{12.88} & \cellcolor{bestrow}\textbf{0.00} \\

\midrule[0.9pt]

\multirow{12}{*}{\textbf{HealMed-VQA}}
& \multirow{3}{*}{Hulu-Med 4B}
& Baseline
& 77.90 & -- & --
& 78.44 & -- & -- \\
& & Random
& 86.10 & 10.00 & 1.80
& 86.83 & 10.18 & 1.80 \\
& & \cellcolor{bestrow}\method
& \cellcolor{bestrow}\textbf{97.60} & \cellcolor{bestrow}\textbf{19.70} & \cellcolor{bestrow}\textbf{0.00}
& \cellcolor{bestrow}\textbf{98.20} & \cellcolor{bestrow}\textbf{19.76} & \cellcolor{bestrow}\textbf{0.00} \\
\cmidrule(lr){2-9}

& \multirow{3}{*}{Hulu-Med 7B}
& Baseline
& 84.60 & -- & --
& 85.07 & -- & -- \\
& & Random
& 87.70 & 4.30 & 1.20
& 88.19 & 4.31 & 1.19 \\
& & \cellcolor{bestrow}\method
& \cellcolor{bestrow}\textbf{95.30} & \cellcolor{bestrow}\textbf{11.04} & \cellcolor{bestrow}\textbf{0.34}
& \cellcolor{bestrow}\textbf{95.83} & \cellcolor{bestrow}\textbf{11.10} & \cellcolor{bestrow}\textbf{0.34} \\
\cmidrule(lr){2-9}

& \multirow{3}{*}{InternVL3.5 4B}
& Baseline
& 81.10 & -- & --
& 81.56 & -- & -- \\
& & Random
& 82.00 & 3.05 & 2.15
& 82.47 & 3.06 & 2.15 \\
& & \cellcolor{bestrow}\method
& \cellcolor{bestrow}\textbf{93.20} & \cellcolor{bestrow}\textbf{12.48} & \cellcolor{bestrow}\textbf{0.38}
& \cellcolor{bestrow}\textbf{93.64} & \cellcolor{bestrow}\textbf{12.46} & \cellcolor{bestrow}\textbf{0.38} \\
\cmidrule(lr){2-9}

& \multirow{3}{*}{Qwen3-VL-8B}
& Baseline
& 84.30 & -- & --
& 84.72 & -- & -- \\
& & Random
& 87.10 & 4.20 & 1.40
& 87.50 & 4.17 & 1.39 \\
& & \cellcolor{bestrow}\method
& \cellcolor{bestrow}\textbf{94.70} & \cellcolor{bestrow}\textbf{11.09} & \cellcolor{bestrow}\textbf{0.69}
& \cellcolor{bestrow}\textbf{95.14} & \cellcolor{bestrow}\textbf{11.11} & \cellcolor{bestrow}\textbf{0.69} \\

\bottomrule[1.2pt]
\end{tabular}
}
\vspace{-3pt}
\end{table*}

\subsection{Validation Confusion Matrix Diagnostics}
\label{app:confusion_matrix}

\begin{wrapfigure}{r}{0.4\columnwidth}
\vspace{-22pt}
\centering
\includegraphics[width=1\linewidth]{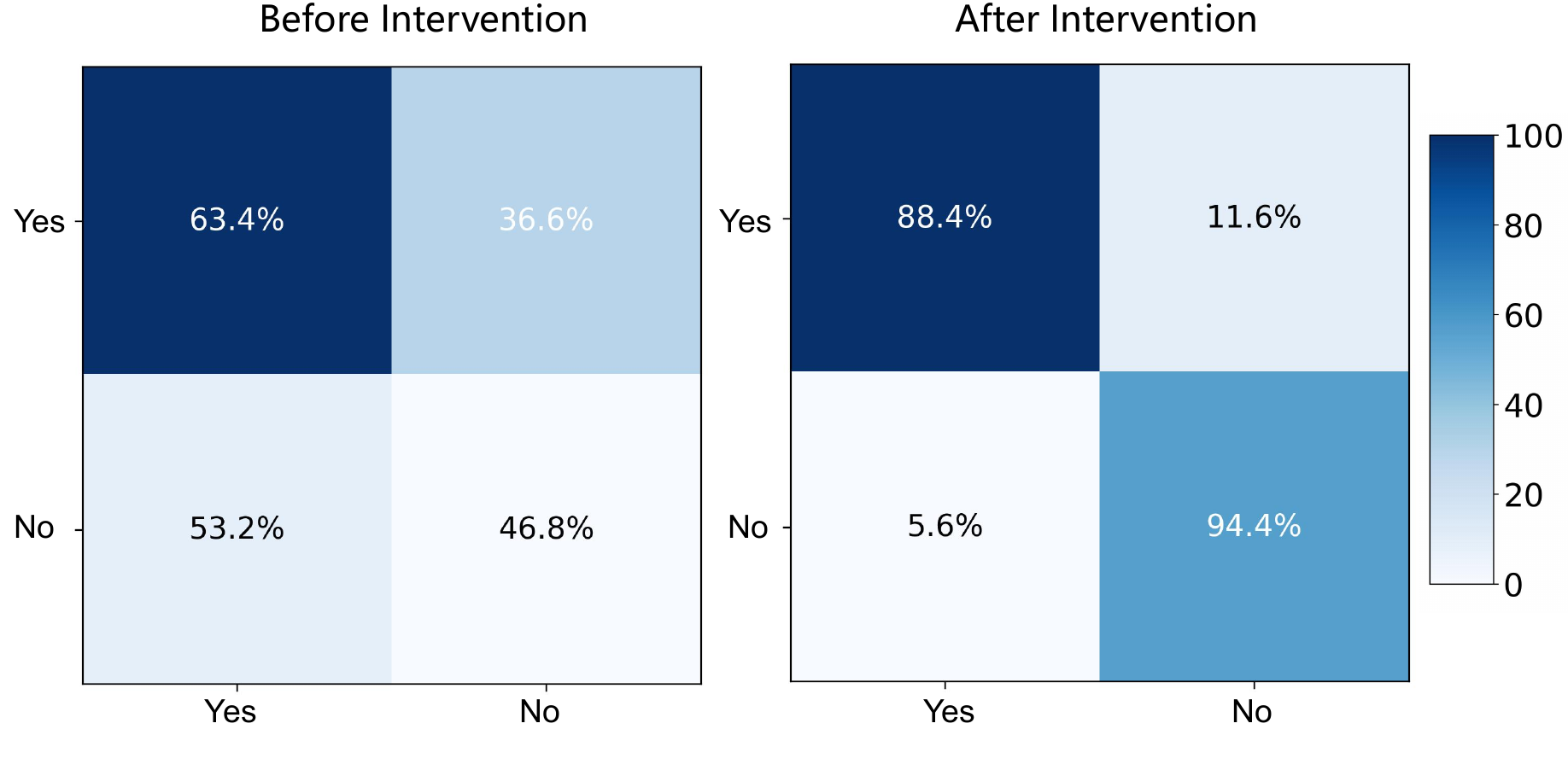}
\caption{
Binary validation confusion matrices for arbitration intervention on Hulu-Med 4B and VQA-RAD.
}
\label{fig:binary_confusion_matrix}
\vspace{-24pt}
\end{wrapfigure}

To better understand how arbitration head ablation changes prediction behavior, we visualise validation set confusion matrices for Hulu-Med 4B on VQA-RAD.
Since most validation answers are binary yes or no labels, we focus on this major subset and report row normalised yes or no confusion matrices before and after ablating the selected arbitration heads.

\noindent\textbf{Analysis.}
Before ablation (Fig.~\ref{fig:binary_confusion_matrix}), conflicting context induces errors in both directions:
36.6\% of gold yes samples are predicted as no, and 53.2\% of gold no samples are predicted as yes.
After ablating the selected arbitration heads, the diagonal entries increase from 63.4\% to 88.4\% for gold yes samples and from 46.8\% to 94.4\% for gold no samples.
This shows that the correction is \textit{bidirectional} rather than driven by a generic preference for a single answer label.
Instead, the intervention restores the yes or no decision boundary under conflicting context, consistent with the increased resist rate reported in the main results.
\section{Failure Case Studies}
\label{app:failure_cases}

Although \method improves arbitration and brake behavior in aggregate, some validation samples remain unresolved after intervention.
We visualise two representative persistent failures in \Cref{fig:persistent_failure_cases}.
The goal of this analysis is not to introduce new evaluation metrics or additional methodological components, but rather to provide a clearer understanding of the residual errors that persist after sparse head excision and their underlying characteristics. Specifically, the two examples correspond to the two distinct failure modes investigated in the main paper: \textit{arbitration failure}, where misleading textual evidence overrides correct visual information, and \textit{brake failure}, where the model produces confident predictions despite insufficient visual evidence caused by degradation.

\subsection{Arbitration Failure}
\label{app:failure_vqarad_arbitration}

\begin{wrapfigure}{r}{0.55\textwidth}
\vspace{-10pt}
\centering
\includegraphics[width=1\linewidth]{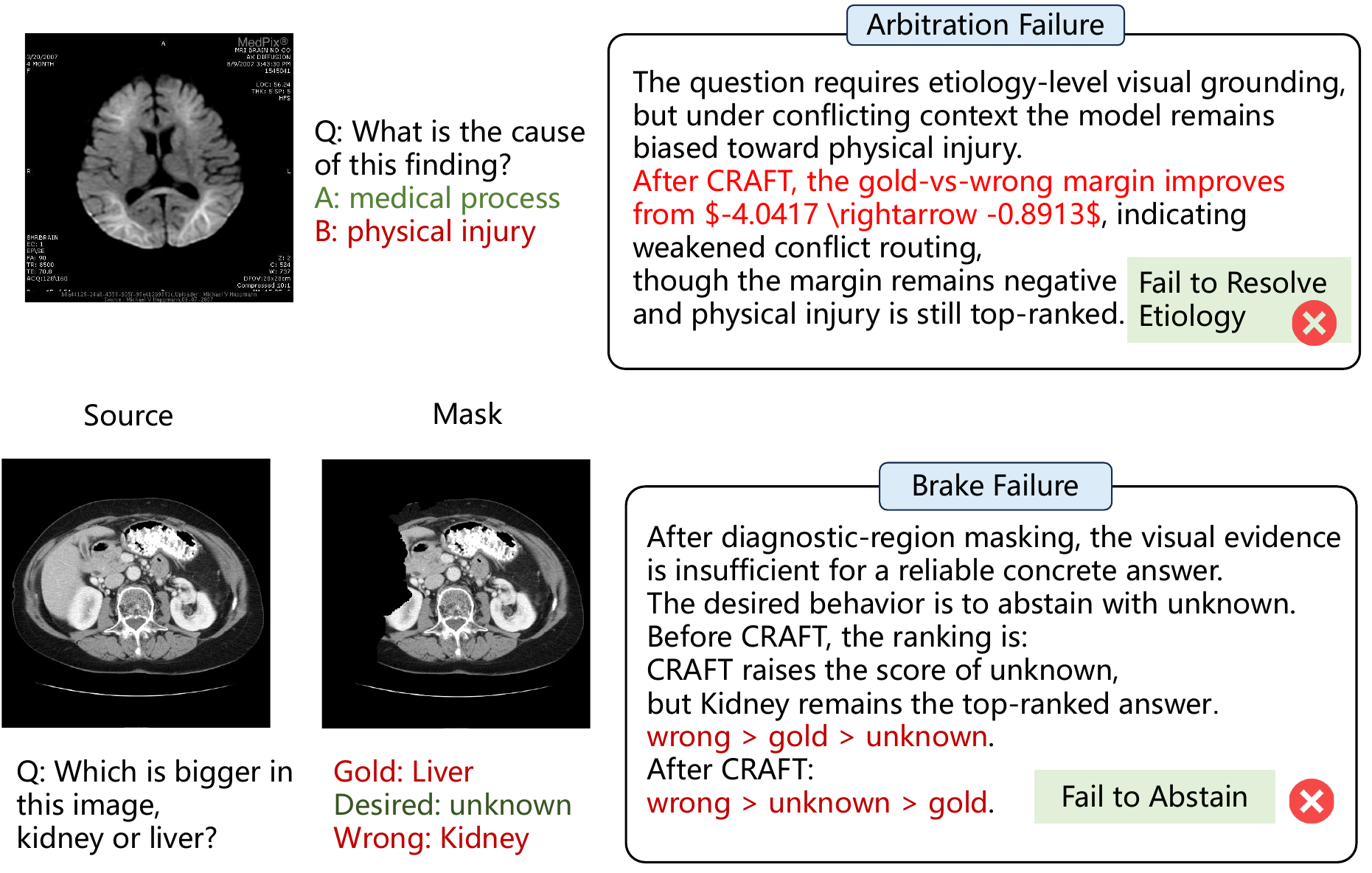}
\caption{
\textbf{Representative persistent failures after \method intervention.}
Upper: arbitration failure case. Lower: brake failure case.
}
\label{fig:persistent_failure_cases}
\vspace{-20pt}
\end{wrapfigure}
The upper case in \Cref{fig:persistent_failure_cases} is a VQA-RAD arbitration example from the Before A setting (\cref{sec:loc_text}).
The question asks whether the visual finding is caused by a \texttt{medical process} or \texttt{physical injury}.
The gold answer is \texttt{medical process}, while the conflicting textual cue favours \texttt{physical injury}.

After \method intervention, the model does move toward the gold answer.
The gold against wrong margin improves from $\Delta_{\mathrm{gold,wrong}} = -4.0417$ to $\Delta_{\mathrm{gold,wrong}} = -0.8913$.
This indicates that the selected arbitration heads contribute to the harmful conflict pathway.
However, the margin remains negative, so \texttt{physical injury} is still the top ranked answer.

This case therefore remains an \textit{arbitration failure}.
The residual error is not a complete failure of intervention, since the conflict preference is weakened.
Instead, the remaining conflict signal is still strong enough to determine the final answer.
The case also requires etiology level visual grounding: the model must infer whether the finding is better explained by a medical process or by physical injury.
This explains why the figure labels the residual pattern as \textit{Fail to Resolve Etiology}.

\subsection{Brake Failure}
\label{app:failure_SLAKE_brake}

The lower case in \Cref{fig:persistent_failure_cases} is a SLAKE brake example from the image conflict setting (\cref{sec:loc_image}).
The question asks which organ is bigger in the image, \texttt{kidney} or \texttt{liver}.
In the original image, the gold answer is \texttt{Liver}.
After diagnostic region masking, the visual evidence is insufficient for a reliable concrete answer, so the desired behavior is \texttt{unknown}.

Before intervention, the ranking is $\texttt{wrong} > \texttt{gold} > \texttt{unknown}$.
After \method intervention, the ranking becomes $\texttt{wrong} > \texttt{unknown} > \texttt{gold}$.
Thus, \method raises the score of \texttt{unknown}, showing that the selected brake heads are related to abstention control.
However, \texttt{Kidney} remains the top ranked answer, so the model still commits to a concrete wrong prediction.
The intervention partially activates abstention behavior but does not make \texttt{unknown} overtake the wrong concrete answer.
This explains why the figure labels the residual pattern as \textit{Fail to Abstain}.

Together, the cases show that sparse head excision can correct the target failure mode in aggregate, while hard samples may still require additional mechanisms.
In particular, difficult etiology attribution may require stronger image grounded reasoning, and difficult brake cases may require adaptive abstention thresholds or localization aware decoding.

\section{Computational Cost}
\label{app:computational_cost}

We report the inference overhead of \method style head intervention under the Before Q setting on VQA-RAD.
The benchmark is conducted on one validation sample for each model, using the corresponding selected arbitration heads.
We compare two implementations.
The first applies head masking through forward hooks, which is the implementation used in our analysis experiments.
This implementation leaves the original attention computation intact and therefore does not reduce the profiled matrix multiplication FLOPs.
The second implements a query head skipping variant that removes the selected query heads from the attention computation before the output projection.
This variant reduces profiled FLOPs, but it is implemented with Python level module patching rather than fused kernels.

\begin{table}[h]
\centering
\caption{
\textbf{Computational cost of head intervention on VQA-RAD.}
Hooked masking is the analysis implementation; true head skipping offers marginal FLOP savings.
}
\label{tab:computational_cost}
\scriptsize
\setlength{\tabcolsep}{3.5pt}
\resizebox{\textwidth}{!}{
\begin{tabular}{llcccccc}
\toprule[1.2pt]
\textbf{Model} & \textbf{Implementation} & \textbf{\#Heads} & \textbf{Normal Lat.} & \textbf{Interv. Lat.} & \textbf{$\Delta$Lat.} & \textbf{Normal FLOPs} & \textbf{$\Delta$FLOPs} \\
\midrule[0.9pt]
Hulu-Med 4B
& Hooked masking
& 6
& 206.39 ms
& 210.31 ms
& +1.90\%
& 19539.96G
& 0.00\% \\
Hulu-Med 4B
& True head skipping
& 6
& 207.11 ms
& 208.10 ms
& +0.48\%
& 19539.96G
& 0.11\% lower \\
InternVL3.5 4B
& Hooked masking
& 10
& 43.27 ms
& 49.43 ms
& +14.24\%
& 3629.54G
& 0.00\% \\
InternVL3.5 4B
& True head skipping
& 10
& 43.91 ms
& 47.03 ms
& +7.11\%
& 3629.54G
& 0.10\% lower \\
\bottomrule[1.2pt]
\end{tabular}}
\end{table}

\noindent\textbf{Analysis.}
Hooked masking is designed for faithful analysis, not acceleration.
It leaves attention computation intact and therefore does not reduce profiler FLOPs.
True head skipping gives a small theoretical FLOP reduction, but the sparse intervention level is too small for latency gains without fused kernels.
These results clarify that \method is a \textit{safety oriented post hoc intervention} rather than an inference acceleration method.

\section{Limitations and Future Directions}
\label{app:limitations_revised}

\paragraph{Limitations.}

\textbf{\textit{(i)} Model scope.}
Our mechanistic analysis is restricted to open-weight models up to 14B parameters.
Larger or closed-source systems may distribute arbitration and brake circuits across a wider set of layers, and our fixed scan window might miss relevant heads in such architectures.
\textbf{\textit{(ii)} Language coverage.}
All evaluated prompts and conflict injections are in English.
Medical VLMs deployed in multilingual clinical environments may exhibit language-dependent attention routing, and the current head discovery pipeline has not been validated on non-English inputs.
\textbf{\textit{(iii)} Imaging modality.}
The benchmarks predominantly feature chest X-rays and fundoscopy.
Three-dimensional modalities (CT, MRI) that produce multi-slice inputs could engage different visual token aggregation patterns not captured by our single-image protocol.
\textbf{\textit{(iv)} Static intervention.}
Head suppression is applied uniformly across all tokens and decoding steps.
A more refined approach would modulate intervention strength dynamically based on per-token conflict signals, which we leave unexplored.

\paragraph{Future directions.}

\textbf{\textit{(i)} Adaptive and input-conditioned intervention.}
A promising extension is to replace static head zeroing with a lightweight gating network that activates suppression only when conflict is detected in the current input.
This would reduce the risk of over-abstention on unambiguous cases and enable deployment as a real-time clinical safety layer that imposes negligible cost on routine (non-conflicting) queries.
Preliminary experiments with a single-layer binary classifier on the residual stream at the conflict-sensitive layer suggest that such gating is feasible with $<$0.5\% additional latency.
\textbf{\textit{(ii)} Scaling to multimodal agents and report generation.}
Current evaluation focuses on single-turn VQA; however, clinical AI systems increasingly operate as multi-step agents that retrieve records, compare serial images, and generate structured reports.
Extending \method{} to trace arbitration circuits across multi-turn reasoning chains, where conflicts may arise between the model's own prior outputs and new visual evidence, would address a critical safety gap.
We envision a hierarchical tracing framework that first localises conflict-sensitive layers within each reasoning step, then tracks how unresolved conflicts propagate across steps, enabling targeted intervention at the earliest point of failure.ervention at the earliest point of failure.

\end{document}